\documentclass[dvipsnames]{article}
\usepackage[nonatbib,final]{neurips_2021}

\usepackage[utf8]{inputenc} %
\usepackage[T1]{fontenc}    %
\usepackage{xcolor}
\usepackage{textcomp}
\usepackage{float}
\usepackage{gensymb}
\usepackage[colorlinks=true, linkcolor=red, urlcolor=blue, citecolor=blue, anchorcolor=blue]{hyperref}
\usepackage{booktabs}       %
\usepackage{amsfonts}       %
\usepackage{nicefrac}       %
\usepackage{microtype}      %
\usepackage{mathtools}
\usepackage{lipsum}
\usepackage{graphicx}
\usepackage{cleveref}
\usepackage[shortlabels]{enumitem}
\usepackage{tikz}
\usetikzlibrary{shapes.geometric, arrows, shadows, bayesnet}
\usepackage{wrapfig}
\usepackage{subcaption}
\usepackage[numbers, compress, sort]{natbib}
\usepackage{tabularx}

\usepackage{pdflscape}
\usepackage{bm}
\usepackage{titlesec}
\usepackage{amsmath, amsthm}
\usepackage{thmtools, thm-restate}
\usepackage{notation}
\usepackage{wrapfig}
\usepackage{colortbl}

\titlespacing*{\section}{0pt}{0.2ex plus .1ex minus .1ex}{0.1ex plus .1ex minus .1ex}
\titlespacing*{\subsection}{0pt}{0.1ex plus .1ex minus .1ex}{0.05ex plus .05ex minus .05ex}
\setlist[itemize]{leftmargin=*,itemsep=0em}
\setlist[enumerate]{leftmargin=*,itemsep=0em}
\crefname{figure}{Fig.}{Figs.}
\crefname{definition}{Defn.}{Defns.}
\crefname{corollary}{Corollary}{Corollaries}
\crefname{proposition}{Prop.}{Props.}
\crefname{theorem}{Thm.}{Thms.}
\crefname{remark}{Remark}{Remarks}
\crefname{principle}{Principle}{Principles}
\crefname{lemma}{Lemma}{Lemmas}
\crefname{claim}{Claim}{Claims}
\crefname{table}{Tab.}{Tabs.}
\crefname{section}{\S}{\S\S}
\crefname{subsection}{\S}{\S\S}
\crefname{subsubsection}{\S}{\S\S}
\crefname{assumption}{Assumption}{Assumptions}
\newcommand{\xshift}{0em}
\newcommand{\yshift}{3em}

\newcommand{\J}[1]{{\color{ForestGreen}~\textbf{JvK}: #1}}

\newcommand{\michel}[1]{{\color{red}~\textbf{Michel}: #1}}

\usepackage{pifont}
\newcommand{\cmark}{\ding{51}}%
\newcommand{\xmark}{\ding{55}}%
\usepackage{multirow}

\title{Data augmentations in self-supervised learning isolate content}
\title{Self-supervised learning with data augmentations\\ can isolate content from style}
\title{Self-Supervised Learning with Data Augmentations\\ Provably Isolates Content from Style}

\author{
Julius von K\"ugelgen\thanks{Joint first author. $^\dagger$Joint senior author. Correspondence to: \texttt{jvk@tue.mpg.de}}\, $^{1,2}$
\And
 Yash Sharma\footnotemark[1]\, $^{3,4}$
 \And
 Luigi Gresele\footnotemark[1]\, $^{1}$
\AND
\textbf{Wieland Brendel} $^3$
\quad 
\textbf{Bernhard Sch\"olkopf}\footnotemark[2]\, $^1$
\quad
\textbf{Michel Besserve}\footnotemark[2]\, $^1$
\quad 
\textbf{Francesco Locatello}\footnotemark[2]\, $^{5}$
\\[1em]
$^1$ Max Planck Institute for Intelligent Systems T\"ubingen
\quad \quad  
$^2$ University of Cambridge\\[.5em]
$^{3}$ Tübingen AI Center, University of T\"ubingen
\quad    $^{4}$ IMPRS for Intelligent Systems
\quad    $^{5}$ Amazon %
}

\begin{document}

\maketitle

\begin{abstract}
\let\thefootnote\relax\footnotetext{Code available at: \href{https://www.github.com/ysharma1126/ssl_identifiability}{https://www.github.com/ysharma1126/ssl\_identifiability}}
Self-supervised representation learning has shown remarkable success in a number of domains.
A common practice is to perform data augmentation via hand-crafted transformations intended to leave the semantics of the data invariant. 
We seek to understand the empirical success of this approach from a theoretical perspective.
We formulate the augmentation process as a latent variable model by postulating a partition of the latent representation into a \textit{content} component, which is assumed invariant to augmentation, and a \textit{style} component, which is allowed to change.
Unlike prior work on disentanglement and independent component analysis, we allow for both nontrivial statistical and causal dependencies in the latent space.
We study the identifiability of the latent representation based on pairs of views of the observations
and prove sufficient conditions that allow us to identify the invariant content partition up to an invertible mapping in both generative and discriminative settings.
We find numerical simulations with dependent latent variables are consistent with our theory.
Lastly, we
introduce \textit{Causal3DIdent}, a dataset of high-dimensional, visually complex images with rich causal dependencies, 
which we use to study the effect of data augmentations performed in practice.

\end{abstract}

\setcounter{footnote}{0}

\section{Introduction}
\label{sec:introduction}
Learning good representations of high-dimensional observations
from large amounts of unlabelled data
is widely recognised as an important step for
more capable and data-efficient learning systems~\cite{bengio2013representation,lake2017building}.
Over the last decade, \textit{self-supervised learning} (SSL)
has emerged as the dominant paradigm for such unsupervised representation learning~\cite{wu2018unsupervised,oord2018representation,henaff2020data,tian2019contrastive,he2019momentum,chen2020simple, grill2020byol, chen2020exploring,agrawal2015learning,doersch2015unsupervised,wang2015unsupervised,vincent2008extracting,noroozi2016unsupervised}.
The main idea behind SSL is to extract a supervisory signal from unlabelled observations by leveraging known structure of the data, which allows for the application of supervised learning techniques. 
A common approach
is to directly predict some
part of the observation from 
another part (e.g., future from past, or original from corruption), thus forcing the model to learn a meaningful representation in the process.
While this technique has shown remarkable success in natural language processing~\cite{collobert2011natural,mikolov2013distributed,pennington2014glove,logeswaran2018efficient,devlin2019bert,radford2018improving,liu2019roberta,GPT3}
and speech recognition~\cite{baevski2020wav2vec,Baevski2020vqwav2vec,ravanelli2020multi,schneider2019wav2vec}, where
a finite dictionary allows one to output a distribution over the missing part, such \textit{predictive} SSL methods are not easily applied to continuous or high-dimensional domains such as vision.
Here, 
a common approach is to learn a \textit{joint embedding} of similar observations or \textit{views} such that their representation is close~\cite{becker1992self,hadsell2006dimensionality,chopra2005learning,bromley1993signature}.
Different views can come, for example, from different modalities (text \& speech; video \& audio) or time points.
As still images lack such multi-modality or temporal structure, recent advances in representation learning have relied on generating similar views by means of \textit{data augmentation}. 
In order to be useful, data augmentation is thought to require the transformations applied to generate additional views to be generally chosen to \textit{preserve the semantic characteristics} of an observation, while changing other ``nuisance'' aspects.
While this intuitively makes sense and has shown remarkable empirical results, the success of data augmentation techniques in practice is still not very well understood from a theoretical perspective---despite some efforts~\cite{ChaSch02,dao2019kernel,chen2020group}.
In the present work, we seek to better understand the empirical success of SSL with data augmentation by formulating the generative process
as a latent variable model (LVM) and studying \textit{identifiability} of the representation, i.e., under which conditions the ground truth latent factors can provably be inferred from the data~\cite{lehmann2006theory}. 

\textbf{Related work and its relation to the current.}
Prior work on unsupervised representation learning from an LVM perspective often postulates \textit{mutually independent latent factors}: this independence assumption is, for example, at the heart of  independent component analysis (ICA)~\cite{comon1994independent,hyvarinen2000independent} and disentanglement~\cite{bengio2013representation,higgins2017beta,kim2018disentangling,burgess2018understanding,kumar2018variational,chen2018isolating}.
Since it is 
impossible to
identify the true latent factors without any supervisory signal in the general nonlinear case~\cite{hyvarinen1999nonlinear,locatello2019challenging}, recent work has turned to weakly- or self-supervised approaches which leverage additional information in the form of multiple views~\cite{gresele2019incomplete,locatello2020weakly,shu2019weakly,zimmermann2021contrastive}, auxiliary variables~\cite{hyvarinen2019nonlinear,khemakhem2020variational}, or temporal structure~\cite{halva2020hidden,hyvarinen2016unsupervised,hyvarinen2017nonlinear,klindt2020slowvae}.
To identify or disentangle the individual independent latent factors, it is typically assumed
that there is a chance
that \textit{each factor changes} across views,
environments, or
time points.
\newcommand{\dependencecolor}{Plum}
\newcommand{\decodingcolor}{Orange}
\newcommand{\stylechangecolor}{ForestGreen}
\newcommand{\originalcolor}{Blue}
\newcommand{\augmentationcolor}{Maroon}
\begin{wrapfigure}[]{r}{0.375\textwidth}
\centering
    \vspace{-1.5em}
    \begin{tikzpicture}
        \centering
        \node (c) [latent] {$\cb$};
        \node (s) [latent, left=of c] {$\sb$};
        \node (s') [latent, right=of c] {$\sbt$};
        \node (x) [obs, below=of c, xshift=-2.75em] {$\xb$};
        \node (x') [obs, below=of c, xshift=2.75em] {$\xbt$};
        \edge[color=\dependencecolor,thick]{c}{s};
        \path[->, color=\stylechangecolor,thick] (s) edge[bend right=-40] node[yshift=.5em] {\textbf{style change}} (s');
        \edge[color=\decodingcolor,thick]{c}{x,x'};
        \edge[color=\decodingcolor,thick]{s}{x};
        \edge[color=\decodingcolor,thick]{s'}{x'};
        \plate[inner sep=0.3em,yshift=0.2em,dashed,color=\augmentationcolor,thick] {augmented}{(c) (s') (x')}{\textcolor{\augmentationcolor}{\textbf{augmented}}}; 
        \tikzset{plate caption/.style={caption, node distance=0, inner sep=0pt, below left=5pt and 0pt of #1.south,text height=1.2em,text depth=0.3em}}
        \plate[inner sep=0.2em,yshift=0.1em,dashed,color=\originalcolor,thick] {original}{(c) (s) (x)}{\textcolor{\originalcolor}{\textbf{original}}};
    \end{tikzpicture}
    \caption{
    \small 
    \textbf{Overview of our problem formulation.} We partition the latent variable $\zb$ into content~$\cb$ and style~$\sb$, and allow for \textcolor{\dependencecolor}{statistical and causal dependence of style on content}. We assume that \textcolor{\stylechangecolor}{only style changes between} \textcolor{\originalcolor}{the original view} $\xb$ and \textcolor{\augmentationcolor}{the augmented view} $\xbt$, i.e., they are obtained by \textcolor{\decodingcolor}{applying the same deterministic function} $\fb$ to $\zb=(\cb,\sb)$ and $\zbt=(\cb,\sbt)$.}
    \label{fig:our_assumption}
    \vspace{-.75em}
\end{wrapfigure}
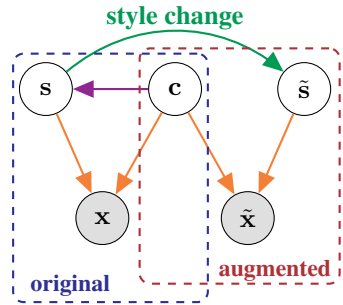
Our work---being directly motivated by common practices in SSL with data augmentation---differs from these works in the following two key aspects (see~\cref{fig:our_assumption} for an overview).
First, we do not assume independence and instead \textit{allow for both nontrivial statistical and causal relations between latent variables}.
This is in line with a recently proposed~\cite{scholkopf2019causality} shift towards causal representation learning~\cite{scholkopf2021toward,yang2020causalvae,suter2019robustly,shen2020disentangled,mitrovic2020representation,von2020towards,leeb2020structural,lu2021nonlinear,gresele2021independent}, motivated by the fact that many underlying variables of interest may not be independent but causally related to each other.\footnote{E.g.,~\cite{klindt2020slowvae}, Fig.~11 where dependence between latents was demonstrated for multiple natural video data sets.}
Second, instead of a scenario wherein all latent factors may change as a result of augmentation, we assume a \textit{partition of the latent space} into two
blocks: a \textit{content} block which is shared or \textit{invariant} across different augmented views, and a \textit{style} block that \textit{may change}. %
\looseness-1 This is aligned with the notion that  augmentations leave certain semantic aspects (i.e., content) intact and only affect style, and is thus a more appropriate assumption for studying SSL.
In line with earlier work~\cite{locatello2020weakly,gresele2019incomplete,zimmermann2021contrastive,klindt2020slowvae,locatello2019challenging,khemakhem2020variational,hyvarinen1999nonlinear,hyvarinen2019nonlinear,hyvarinen2016unsupervised}, we focus on the setting of continuous ground-truth latents, though we believe our results to hold more broadly.

\textbf{Structure and contributions.}
Following a review of SSL with data augmentation
and identifiability theory~(\cref{sec:background_ICA}), we formalise the process of data generation and augmentation as an LVM with content and style variables~(\cref{sec:problem_formulation}).
We then establish identifiability results of the invariant content partition~(\cref{sec:theory}),
validate our theoretical insights experimentally~(\cref{sec:experiments}), and discuss our findings and their limitations in the broader context of SSL with data augmentation~(\cref{sec:discussion}).
We highlight the following contributions:%
\begin{itemize}[itemsep=0pt, topsep=0pt]
    \item we prove that SSL with data augmentations identifies the invariant content partition of the representation in generative~(\Cref{thm:main}) and discriminative learning with invertible~(\Cref{thm:CL}) and non-invertible encoders with entropy regularisation~(\Cref{thm:CL_MaxEnt});
    in particular, \Cref{thm:CL_MaxEnt} provides a theoretical justification for the empirically observed effectiveness of contrastive SSL methods  that use data augmentation and InfoNCE~\cite{oord2018representation} as an objective, such as \texttt{SimCLR}~\cite{chen2020simple};
    \item we show that our theory is consistent with results in simulating statistical dependencies within blocks of content and style variables, as well as with style causally dependent on content~(\cref{sec:experiment_1_numerical_simulation}); 
    \item we introduce \textit{Causal3DIdent}, a 
    dataset of 3D objects 
    which allows for the study of identifiability in a causal representation learning setting, and use it to perform a systematic study of data augmentations used in practice, %
    yielding novel insights on what particular data augmentations are truly isolating as invariant content and discarding as varying style when applied~(\cref{sec:experiment_2_causal3dident}). 
\end{itemize}

\section{Preliminaries and background}
\label{sec:background}

\textbf{Self-supervised representation learning with data augmentation.}
\label{sec:background_augmentation}
Given an unlabelled dataset of observations (e.g., images) $\xb$, data augmentation techniques proceed as follows.
First, a set of observation-level transformations $\tb\in\Tcal$ are specified together with a distribution $p_\tb$ over $\Tcal$. 
Both $\Tcal$ and $p_\tb$ are typically designed using human intelligence and domain knowledge with the intention of \textit{not changing the semantic characteristics} of the data (which arguably constitutes a form of weak supervision).\footnote{Note that recent work has investigated automatically discovering good augmentations~\cite{cubuk2019autoaugment,cubuk2020randaugment}.}
For images, for example, a common choice for $\Tcal$ are combinations of random crops~\citep{szegedy2015going}, horizontal or vertical flips, blurring, colour distortion~\citep{howard2013improvements,szegedy2015going}, or cutouts~\citep{devries2017improved}; and $p_\tb$ is a distribution over the parameterisation of these transformations, e.g., the centre and size of a crop~\cite{chen2020simple,devries2017improved}.
For each observation $\xb$, a pair of transformations $\tb,\tb'
\sim
p_\tb$ is sampled and applied separately to $\xb$ to generate a pair of augmented views $(\xbt,\xbt')=(\tb(\xb),\tb'(\xb))$.

The joint-embedding approach to SSL then uses a pair of encoder functions $(\gb,\gb')$, i.e. deep nets, to map the pair $(\xbt,\xbt')$ to a typically lower-dimensional representation $(\zbt,\zbt')=(\gb(\xbt),\gb'(\xbt'))$.
Often, the two encoders are either identical, $\gb=\gb'$, or directly related (e.g., via shared parameters or asynchronous updates). 
Then, the encoder(s) $(\gb,\gb')$ are trained  such that the representations $(\zbt,\zbt')$ are ``close'', i.e., such that $\text{sim}(\zbt,\zbt')$ is large for some  similarity metric $\text{sim}(\cdot)$, e.g., the cosine similarity~\citep{chen2020simple, zimmermann2021contrastive}, or negative L2 norm~\citep{zimmermann2021contrastive}.
The advantage of 
directly optimising for similarity in representation space over generative alternatives is that reconstruction can be very challenging for high-dimensional data.
The disadvantage
is the problem of \textit{collapsed representations}.\footnote{If the only goal is to make representations of augmented views similar, a degenerate solution which simply maps any observation to the origin trivially achieves this goal.}
To avoid 
collapsed representations and force the encoder(s) to learn a meaningful representation, two main
families of 
approaches have been used: (i) \textit{contrastive learning} (CL)~\citep{oord2018representation,he2019momentum,chen2020simple,wu2018unsupervised,henaff2020data,tian2019contrastive}; and (ii) \textit{regularisation-based} SSL~\cite{grill2020byol,chen2020exploring,zbontar2021barlow}.

The idea behind CL is to not only learn similar representations for augmented views $(\xbt_i, \xbt'_i)$ of the same $\xb_i$, or \textit{positive pairs}, but to also use other observations $\xb_j$ $(j\neq i)$ to contrast with, i.e., to enforce a dissimilar representation across \textit{negative pairs} $(\xbt_i, \xbt'_j)$.
In other words, CL pulls representations of positive pairs together, and pushes those of negative pairs apart.
Since both aims cannot be achieved simultaneously with a constant representation, collapse is avoided.
A popular CL objective function (used, e.g., in \texttt{SimCLR}~\cite{chen2020simple}%
)
is InfoNCE~\cite{oord2018representation}
(based on noise-contrastive estimation
~\cite{gutmann2010noise,Gutmann12JMLR}):%
\begin{equation}
\label{eq:InfoNCE_objective}
\textstyle
    \Lcal_{\text{InfoNCE}}
    (\gb;\tau,K)
    =
    \EE_{\{\xb_i\}_{i=1}^K \sim p_\xb}
    \Big[
    -
    \sum_{i=1}^K
    \log 
    \frac{
    \exp\{\text{sim}(\zbt_i,\zbt'_i)/\tau\}
    }
    {
    \sum_{j=1}^K
    \exp\{\text{sim}(\zbt_i,\zbt'_j)/\tau\}
    }
    \Big]
\end{equation}
where $\zbt=\EE_{\tb\sim p_\tb}[\gb(\tb(\xb))]$, $\tau$ is a temperature, and $K-1$ is the number of negative pairs.
InfoNCE~\eqref{eq:InfoNCE_objective} has an interpretation as multi-class logistic regression, and
lower bounds the mutual information across similar views $(\zbt,\zbt')$---a common representation learning objective~\cite{tschannen2019mutual,poole2019variational,hjelm2018learning,bachman2019learning,linsker1988self,linsker1989application,cardoso1997infomax,lee1999independent,bell1995information}.
Moreover,~\eqref{eq:InfoNCE_objective} can be interpreted as \textit{alignment} (numerator) and \textit{uniformity} (denominator) terms, the latter constituting a nonparametric entropy estimator of the representation as $K\rightarrow\infty$~\cite{wang2020understanding}.
{CL with InfoNCE can thus be seen as alignment of positive pairs with (approximate) entropy regularisation.}

Instead of using negative pairs, as in CL, a set of recent SSL methods only optimise for alignment and avoid collapsed representations through different forms of regularisation. 
For example, \texttt{BYOL}~\cite{grill2020byol} and \texttt{SimSiam}~\cite{chen2020exploring} rely on ``architectural regularisation'' in the form of moving-average updates for a separate ``target'' net $\gb'$ (\texttt{BYOL} only) or
a stop-gradient operation (both).
\texttt{BarlowTwins}~\cite{zbontar2021barlow}, on the other hand, optimises the cross correlation  between $(\zbt,\zbt')$ to be close to the identity matrix, thus enforcing redundancy reduction (zero off-diagonals) in addition to  alignment (ones on the diagonal).

\textbf{Identifiability of learned representations.}
\label{sec:background_ICA}
In this work, we address the question of whether SSL with data augmentation 
can reveal or uncover properties of the underlying data generating process.
Whether a representation learned from observations
can be expected to match the true underlying latent factors---up to acceptable ambiguities and
subject to suitable assumptions on the generative process and inference model---is captured by the notion of identifiability~\cite{lehmann2006theory}.

Within 
representation learning, identifiability has mainly been studied in the framework of (nonlinear) ICA which assumes a model of the form $\xb=\fb(\zb)$ and aims to recover the independent latents, or \textit{sources}, $\zb$, typically up to permutation or element-wise transformation.
A crucial negative
result states that, with i.i.d.\ data and without further assumptions,
this is fundamentally impossible~\cite{hyvarinen1999nonlinear}.
However, recent breakthroughs have shown that
identifiability can be achieved if
an auxiliary variable (e.g., a time stamp or environment index) renders the sources \textit{conditionally} independent~\cite{hyvarinen2016unsupervised, hyvarinen2017nonlinear, hyvarinen2019nonlinear, halva2020hidden}.
These methods rely on constructing positive and negative pairs using the auxiliary variable and learning a representation with CL.
\looseness-1 This development has sparked a renewed interest in identifiability in the context of deep representation learning~\cite{khemakhem2020variational,khemakhem2020ice,klindt2020slowvae, Sorrenson2020Disentanglement, zimmermann2021contrastive,locatello2020weakly,roeder2020linear,shu2019weakly}.

Most closely related to SSL with data augmentation are works which study identifiability when 
given 
a second view $\xbt$ of an observation $\xb$, resulting from a modified version $\zbt$ of the underlying latents or sources $\zb$~\cite{gresele2019incomplete,richard2020modeling,locatello2020weakly,shu2019weakly,zimmermann2021contrastive,klindt2020slowvae}.
Here, $\zbt$ is either an element-wise  corruption of $\zb$~\cite{gresele2019incomplete,richard2020modeling,zimmermann2021contrastive,klindt2020slowvae} or may share a random subset of its components~\cite{locatello2020weakly,shu2019weakly}.
Crucially, all previously mentioned works assume that \textit{any} of the independent latents (are allowed to) change, and aim to identify the individual factors.
However, in the context of SSL with data augmentation, where the semantic (content) part of the representation is intended to be shared between views, this assumption does not hold.

\section{Problem formulation}
\label{sec:problem_formulation}

We specify our problem setting by formalising the processes of data generation and augmentation.
We take a latent-variable model perspective and
assume that observations $\xb$ (e.g., images) are generated by a \emph{mixing} function $\fb$ which takes a latent code $\zb$ as input.
Importantly, we describe the augmentation process through changes in this latent space as captured by a conditional distribution $p_{\zbt|\zb}$, as opposed to traditionally describing the transformations $\tb$ as acting directly at the observation level.

Formally, let $\zb$ be a continuous r.v.\ taking values in an open, simply-connected $n$-dim.\ \textit{representation space} $\Zcal\subseteq\RR^n$ with associated probability 
density $p_\zb$.
Moreover, let $\fb:\Zcal\rightarrow\Xcal$ be a \textit{smooth and invertible} mapping to an \textit{observation space} $\Xcal\subseteq \RR^d$
and let $\xb$ be the continuous r.v.\ defined as $\xb=\fb(\zb)$.\footnote{While $\xb$ may be high-dimensional $n\ll d$, invertibility of $\fb$ implies that $\Xcal$ is an $n$-dim.\ sub-manifold of~$\RR^d$.}
The generative process for the dataset of original observations of $\xb$ is thus given by:%
\begin{align}
\textstyle
\label{eq:generative_process_original}
    \zb \sim p_\zb,
    \quad \quad \quad \quad 
    \xb =\fb(\zb).
\end{align}%
Next, we formalise the data augmentation process.
As stated above, we take a representation-centric view, i.e., we assume that an augmentation $\xbt$ of the original $\xb$ is obtained by applying the same mixing or rendering function $\fb$ to a modified representation $\zbt$ which is (stochastically) related to the original representation $\zb$ of 
$\xb$.
Specifying the effect of data augmentation thus corresponds to specifying a conditional distribution $p_{\zbt|\zb}$ which captures the relation between $\zb$ and $\zbt$.

In terms of the transformation-centric view presented in~\cref{sec:background_augmentation}, we can view the modified representation $\zbt\in\Zcal$ as obtained by applying $\fbinv$ to a transformed observation $\xbt=\tb(\xb)\in\Xcal$ where $\tb\sim p_\tb$, i.e., $\zbt=\fbinv(\xbt)$.
The conditional distribution $p_{\zbt|\zb}$ in the representation space can thus be viewed as being induced by the distribution $p_\tb$ over transformations applied at the observation level.\footnote{We investigate this correspondence between changes in observation and latent space empirically in~\cref{sec:experiments}.}

We now encode the notion that the set of transformations $\Tcal$ used for augmentation is typically chosen such that any transformation $\tb\in\Tcal$ leaves certain aspects of the data invariant.
To this end, we assume that \textit{the representation $\zb$ can be uniquely partitioned into two disjoint parts}:%
\begin{enumerate}[label=(\roman*), topsep=-3pt,itemsep=0pt]
    \item an \textit{invariant} part $\cb$ which will \textit{always be shared} across $(\zb,\zbt)$, and which we refer to as \textit{content};
    \item a \textit{varying} part $\sb$ which \textit{may
    change} across $(\zb,\zbt)$, and which we refer to as \textit{style}.
\end{enumerate}%
We assume that $\cb$ and $\sb$ take values in content and style subspaces $\Ccal\subseteq \RR^{n_c}$ and 
$\Scal\subseteq\RR^{n_s}$, respectively, i.e., $n=n_c+n_s$ and $\Zcal=\Ccal\times\Scal$. 
W.l.o.g., we let $\cb$ corresponds to the first $n_c$ dimensions of~$\zb$:
\begin{equation*}
\textstyle
\label{eq:def_content_style}
    \zb = (\cb, \sb),
    \quad \quad \quad \quad 
    \cb := \zb_{1:n_c},
    \quad \quad \quad \quad 
    \sb := \zb_{(n_c+1):n},
\end{equation*}%
We formalise the process of data augmentation with content-preserving transformations by defining the conditional $p_{\zbt|\zb}$ such that only a (random) subset of the style variables change at a time.

\begin{assumption}[Content-invariance]
\label{ass:content_invariance}
The conditional density $p_{\zbt|\zb}$ over $\Zcal\times\Zcal$ takes the form
\begin{equation*}
\textstyle
    p_{\zbt|\zb}(\zbt|\zb)=\delta(\cbt-\cb)p_{\sbt|\sb}(\sbt|\sb)
\end{equation*}
for some continuous density $p_{\sbt|\sb}$ on $\Scal\times\Scal$, where $\delta(\cdot)$ is the Dirac delta function, i.e., $\cbt=\cb$ a.e.
\end{assumption}

\begin{assumption}[Style changes]
\label{ass:style_changes}
Let $\Acal$ be the set of subsets of style variables $A\subseteq\{1, ..., n_s\}$ and let $p_A$ be a distribution on $\Acal$.
Then, the style conditional $p_{\sbt|\sb}$ is obtained via%
\begin{equation*}
\textstyle
    A\sim p_A,
    \quad \quad \quad \quad 
    p_{\sbt|\sb,A}(\sbt|\sb,A) = \delta(\sbt_{\Ac} - \sb_{\Ac}) p_{\sbt_\A|\sb_\A}(\sbt_\A|\sb_\A)\, ,
\end{equation*}%
where $p_{\sbt_\A|\sb_\A}$ is a continuous density on $\Scal_A\times\Scal_A$, $\Scal_A\subseteq\Scal$ denotes the subspace of changing style variables specified by $A$, and $\Ac
=\{1,...,n_s\}\setminus A
$ denotes the complement of $A$.
\end{assumption}%
Note that Assumption~\ref{ass:style_changes} is less restrictive than assuming that all style variables need to change, since it also allows for only a (possibly different) subset of style variables to change for any given observation.
This is in line with the intuition that not all transformations affect all changeable (i.e., style) properties of the data: e.g., a colour distortion should not affect positional information, and, in the same vein, a (horizontal or vertical) flip should not affect the colour spectrum.

The generative process of an augmentation or transformed observation $\xbt$ is thus given by
\begin{equation}
\label{eq:generative_process_augmentation}
    \textstyle
    A\sim p_A,
    \quad \quad \quad 
    \zbt|\zb, A \sim p_{\zbt|\zb,A},
    \quad \quad \quad
    \xbt = \fb(\zbt).
\end{equation}
Our setting for modelling data augmentation differs from that commonly assumed in (multi-view) disentanglement and ICA in that \textit{we do not assume that the latent factors $\zb=(\cb,\sb)$ are mutually (or conditionally) independent}, i.e., we allow for \textit{arbitrary} (non-factorised) marginals $p_\zb$
in~\eqref{eq:generative_process_original}.\footnote{The recently proposed Independently \textit{Modulated} Component Analysis (IMCA)~\cite{khemakhem2020ice} extension of ICA is a notable exception, but only allows for trivial dependencies across $\zb$ in the form of a shared base measure.
}

\textbf{Causal interpretation: data augmentation as counterfactuals under soft style intervention.}
We now provide a causal account of the above data generating process by describing the (allowed) causal dependencies among latent variables using a 
structural causal model (SCM)~\cite{pearl2009causality}. 
As we will see, this leads to an interpretation of data augmentations as counterfactuals in the underlying latent SCM.
The assumption that $\cb$ stays invariant as $\sb$ changes is consistent with the view that content may causally influence style, $\cb\rightarrow\sb$, but not vice versa, see~\cref{fig:our_assumption}. \looseness-1 We therefore formalise their relation as:%
\begin{align*}
\textstyle
    \cb := \fb_\cb(\ub_\cb), \quad \quad \quad \quad
    \sb := \fb_\sb(\cb, \ub_\sb),
    \quad \quad\quad \quad
    (\ub_\cb, \ub_\sb)\sim p_{\ub_\cb}\times p_{\ub_\sb}
\end{align*}%
where $\ub_\cb, \ub_\sb$ are independent exogenous variables, and $\fb_\cb,\fb_\sb$ are deterministic functions.
The latent causal variables $(\cb, \sb)$ are subsequently decoded into observations $\xb=\fb(\cb, \sb)$.
Given a factual observation $\xbf=
\fb(\cbf, \sbf)$ which resulted from $(\ub_\cb^\texttt{F}, \ub_\sb^\texttt{F})$, 
we may ask the counterfactual question: 
    ``\textit{what would have happened if the style variables had been
    (randomly)
    perturbed, all else being equal%
    ?}''.
\looseness-1 Consider, e.g.,
a \textit{soft intervention}~\cite{eberhardt2007interventions} on~$\sb$, i.e., an intervention that changes the mechanism~$\fb_\sb$ to%
\begin{equation*}
\textstyle
   do(\sb:=\Tilde{\fb}_\sb(\cb, \ub_\sb, \ub_\A)), 
\end{equation*}%
where $\ub_\A$ is an additional source of stochasticity accounting for the randomness of the augmentation process ($p_A\times  p_{\sbt|\sb,A}$).
The resulting distribution over counterfactual observations $\xbcf
=\fb(\cbf, \sbcf)$ can be computed from the modified SCM by fixing the exogenous variables to their factual values and performing the soft intervention:
\begin{align*}
\textstyle
    \cb^\texttt{CF} := \cbf, \quad \quad \quad \quad
    \sb^\texttt{CF} := \Tilde{\fb}_\sb(\cbf, \ub_\sb^\texttt{F}, \ub_\A),
    \quad    \quad     \quad    \quad 
    \ub_\A\sim p_{\ub_\A}.
\end{align*}%
This aligns with our intuition and assumed problem setting of data augmentations as style corruptions.
We note that the notion of augmentation as (hard) style interventions is also at the heart of \texttt{ReLIC}~\cite{mitrovic2020representation}, a recently proposed, causally-inspired SSL regularisation term for instance-discrimination~\cite{hadsell2006dimensionality,wu2018unsupervised}.
However, \texttt{ReLIC} assumes independence between content and style and does not address identifiability.
{For another causal perspective on data augmentation in the context of domain generalisation, c.f.~\cite{ilse2021selecting}.}

\section{Theory: block-identifiability of the invariant content partition}
\label{sec:theory}
Our goal is to prove that we can identify the invariant content partition $\cb$ under a distinct, weaker set of assumptions, compared to existing results in
disentanglement and nonlinear ICA~\cite{zimmermann2021contrastive,gresele2019incomplete,locatello2020weakly,shu2019weakly,klindt2020slowvae}.
We stress again that our primary interest is not to identify or disentangle individual (and independent) latent factors $z_j$, but instead to separate content from style,
such that the content variables can be subsequently used for 
downstream tasks.
We first define this distinct notion of \textit{block-identifiability}.%
\begin{definition}[Block-identifiability]
\label{def:block-identifiability}
We say that the true content partition $\cb=\fbinv(\xb)_{1:n_c}$ is \emph{block-identified} by a function $\gb:\Xcal\rightarrow\Zcal$ if the inferred content partition $\cbh=\gb(\xb)_{1:n_c}$ contains \emph{all} and \emph{only} information about $\cb$, i.e., if there exists an \textit{invertible} function $\hb:\RR^{n_c}\rightarrow \RR^{n_c}$ s.t.\ $\cbh=\hb(\cb)$.%
\end{definition}%
\Cref{def:block-identifiability} is related to independent subspace analysis~\cite{hyvarinen2000emergence, le2011learning,theis2006towards,casey2000separation}, which also aims to identify blocks of random variables
as opposed to individual factors, though under an \textit{independence assumption across blocks}, and typically not within a multi-view setting as studied in the present work.

\subsection{Generative self-supervised representation learning}
First, we consider \textit{generative} SSL, i.e., fitting a generative model to pairs $(\xb,\xbt)$ of original and augmented views.\footnote{For notational simplicity, we present our theory for pairs $(\xb,
\xbt)$ rather than for two augmented views $(\xbt,\xbt')$, as typically used in practice but it also holds for the latter, see~\cref{sec:discussion} for further discussion.}
We show that under our specified data generation and augmentation process~(\cref{sec:problem_formulation}), as well as suitable additional assumptions (stated and discussed in more detail below), it is possible to isolate (i.e., block-identify) the invariant content partition.
Full proofs are included in~\Cref{app:proofs}.

\begin{restatable}[Identifying content with a generative model]
{theorem}{generative}
\label{thm:main}
Consider the data generating process described in~\cref{sec:problem_formulation}, i.e., the pairs $(\xb,\xbt)$ of original and augmented views are generated according to~\eqref{eq:generative_process_original} and~\eqref{eq:generative_process_augmentation} with $p_{\zbt|\zb}$ as defined in Assumptions~\ref{ass:content_invariance} and~\ref{ass:style_changes}.
Assume further that%
\begin{enumerate}[label=(\roman*), itemsep=0pt, topsep=0pt]
    \item
    $\fb:\Zcal\rightarrow \Xcal$ is smooth and invertible with smooth inverse (i.e., a diffeomorphism);
    \item $p_\zb$ is a smooth, continuous density on $\Zcal$ with $p_\zb(\zb)>0$ almost everywhere;
    \item for any $l\in\{1, ..., n_s\}$, $\exists A\subseteq\{1, ..., n_s\}$ s.t.\ $l\in A$; $p_A(A)>0$; $p_{\sbt_\A|\sb_\A}$ is smooth w.r.t.\ both $\sb_A$ and $\sbt_A$; and for any $\sb_A$,  $p_{\sbt_\A|\sb_\A}(\cdot|\sb_A)>0$ in some open, non-empty subset containing $\sb_A$.
\end{enumerate}%
If, for a given $n_s$ ($1\leq n_s<n$), a generative model $(\ph_\zb, \ph_A, \ph_{\sbt|\sb, A}, \fbh )$ assumes the same generative process~(\cref{sec:problem_formulation}), satisfies the above assumptions \textit{(i)-(iii)}, and matches the data likelihood,
\begin{equation*}
    \label{eq:match_likelihood}
    p_{\xb,\xbt}(\xb,\xbt)=\ph_{\xb,\xbt}(\xb,\xbt)
    \quad \quad \quad
    \forall (\xb,\xbt)\in \Xcal\times\Xcal,
\end{equation*}
then it block-identifies the true content variables via $\gb=\fbh^{-1}$ in the sense of~\cref{def:block-identifiability}.
\end{restatable}%
\textbf{Proof sketch.}
First, show (using \textit{(i)} and the matching likelihoods) that the representation $\zbh=\gb(\xb)$ extracted by $\gb$ is related to the true $\zb$ by a smooth invertible mapping $\hb=\gb\circ\fb$ such that $\cbh=\hb(\zb)_{1:n_c}$ is invariant across $(\zb,\zbt)$ almost surely w.r.t.~$p_{\zb,\zbt}$.\footnote{This step is partially inspired by~\cite{locatello2020weakly}; the technique used to prove the second \textit{main} step is entirely novel.}
Second, show by contradiction (using \textit{(ii)}, \textit{(iii)}) that $\hb(\cdot)_{1:n_c}$ can, in fact, only depend on the true content~$\cb$ and not on style~$\sb$, \looseness-1 for otherwise the invariance from step 1 would be violated in a region of the style (sub)space of measure greater than zero.

\textbf{Intuition. 
}
\cref{thm:main} assumes that the number of content ($n_c$) and style ($n_s$) variables is known,
and that there is a positive probability that each \textit{style} variable may change, though not necessarily on its own, according to \textit{(iii)}.
In this case, training a generative model of the form specified in~\cref{sec:problem_formulation} 
(i.e., with an invariant content partition and subsets of changing style variables)
by maximum likelihood on pairs $(\xb,\xbt)$ will asymptotically (in the limit of infinite data) recover the true invariant content partition up to an invertible function, i.e., it isolates, or unmixes, content from style.

\textbf{Discussion.} The identifiability result of~\cref{thm:main} for generative SSL is of potential relevance for existing variational autoencoder (VAE)~\cite{kingma2013auto} variants such as the \texttt{GroupVAE}~\cite{hosoya2019group},\footnote{which also uses a fixed content-style partition for multi-view data, but assumes that all latent factors are mutually independent, and that all style variables change between views, independent of the original style;} or its adaptive version \texttt{AdaGVAE}~\cite{locatello2020weakly}.
Since, contrary to existing results, \cref{thm:main} does not assume independent latents, it may also provide a principled basis for generative causal representation learning algorithms~\cite{leeb2020structural,yang2020causalvae,shen2020disentangled}.
However, an important limitation to its practical applicability is that generative modelling does not tend to scale very well to complex high-dimensional observations,
such as images. 

\subsection{Discriminative self-supervised representation learning}
\label{sec:theory_discriminative}
We therefore next turn to a discriminative approach, i.e., directly learning an encoder function $\gb$ which leads to a similar embedding across $(\xb,\xbt)$. 
As discussed in~\cref{sec:background_augmentation}, this is much more common for SSL with data augmentations.
First, we show that if an invertible encoder $\gb$ is used, then learning a representation which is aligned in the first $n_c$ dimensions is sufficient to block-identify content.

\begin{restatable}[Identifying content with
an invertible encoder]{theorem}{discriminative}
\label{thm:CL}
Assume the same data generating process~(\cref{sec:problem_formulation}) and conditions \textit{(i)}-\textit{(iv)} as in~\Cref{thm:main}. 
Let $\gb:\Xcal \rightarrow \Zcal$ be any smooth and \emph{invertible} function 
which minimises the following functional:%
\begin{equation}
\label{eq:CL_MSE_objective}
\textstyle
\Lcal_\mathrm{Align}(\gb) := \EE_{(\xb,\xbt)\sim p_{\xb, \xbt}}
\left[
\bignorm{
\gb(\xb)_{1:n_c}-\gb(\xbt)_{1:n_c}
}^2_2
\right]
\end{equation}%
Then $\gb$ block-identifies the true content variables in the sense of Definition~\ref{def:block-identifiability}.%
\end{restatable}%
\textbf{Proof sketch.}
First, we show that the global minimum of~\eqref{eq:CL_MSE_objective} is reached by the smooth invertible function~$\fbinv$. Thus, any other minimiser $\gb$ must satisfy the same invariance across $(\xb,\xbt)$ used in step 1 of the proof of~\cref{thm:main}. The second step uses the same argument by contradiction as in~\cref{thm:main}.

\textbf{Intuition.}
\cref{thm:CL} states that if---under the same assumptions on the generative process as in~\cref{thm:main}---we directly learn a representation with an \textit{invertible} encoder, then enforcing alignment between the first $n_c$ latents is sufficient to isolate the invariant content partition. Intuitively, invertibility guarantees that all information is preserved, thus avoiding a collapsed representation.

\textbf{Discussion.}
According to~\cref{thm:CL}, content can be isolated if, e.g.,
a flow-based architecture%
~\cite{dinh2016density,papamakarios2017masked,dinh2014nice,kingma2018glow, papamakarios2021normalizing} is used, or  invertibility is enforced otherwise during training~\cite{jacobsen2018revnet,behrmann2019invertible}.
However, the applicability of this approach is limited as it \textit{places strong constraints on the encoder architecture which makes it hard to scale these methods up to high-dimensional settings}.
As discussed in~\cref{sec:background_augmentation}, state-of-the-art SSL methods such as \texttt{SimCLR}~\cite{chen2020simple},
\texttt{BYOL}~\cite{grill2020byol},
\texttt{SimSiam}~\cite{chen2020exploring}, or
\texttt{BarlowTwins}~\cite{zbontar2021barlow} do not use invertible encoders, but instead avoid collapsed representations---which would result from naively optimising~\eqref{eq:CL_MSE_objective} for arbitrary, non-invertible $\gb$---using different forms of regularisation.

To close this gap between theory and practice, finally, we investigate how to block-identify content without assuming an invertible encoder.
We show that, if we add a regularisation term to~\eqref{eq:CL_MSE_objective} that encourages maximum entropy of the learnt representation, the invertibility assumption can be dropped.%
\begin{restatable}[Identifying content with discriminative learning and a non-invertible encoder]{theorem}{discriminativeMaxEnt}
\label{thm:CL_MaxEnt}
Assume the same data generating process~(\cref{sec:problem_formulation}) and conditions \textit{(i)}-\textit{(iv)} as in~\Cref{thm:main}. 
Let $\gb:\Xcal \rightarrow (0,1)^{n_c}$ be any smooth function which minimises the following functional:%
\begin{equation}
\label{eq:CL_MSE_MaxEnt_objective}
\textstyle
\Lcal_\mathrm{AlignMaxEnt}(\gb)
:= 
\EE_{(\xb,\xbt)\sim p_{\xb, \xbt}}
\left[
\bignorm{
\gb(\xb)-\gb(\xbt)
}^2_2
\right] - H\left(\gb(\xb)\right)
\end{equation}%
where $H(\cdot)$ denotes the differential entropy of the random variable $\gb(\xb)$ taking values in $(0,1)^{n_c}$.
Then $\gb$ block-identifies the true content variables in the sense of~\cref{def:block-identifiability}.%
\end{restatable}%
\textbf{Proof sketch.}
First, use the Darmois construction~\cite{darmois1951construction,hyvarinen1999nonlinear} to build a function $\db:\Ccal\rightarrow (0,1)^{n_c}$ mapping $\cb=\fbinv(\xb)_{1:n_c}$ to a uniform random variable. Then $\gb^\star=\db\circ \fb^{-1}_{1:n_c}$ attains the global minimum of~\eqref{eq:CL_MSE_MaxEnt_objective} because $\cb$ is invariant across $(\xb,\xbt)$ and  the uniform distribution is the maximum entropy distribution on $(0,1)^{n_c}$.
Thus, any other minimiser $\gb$ of~\eqref{eq:CL_MSE_MaxEnt_objective} must satisfy invariance across $(\xb,\xbt)$ and map to a uniform r.v. 
Then, use the same step 2 as in~\cref{thm:main,thm:CL} to show that $\hb=\gb\circ \fb:\Zcal\rightarrow (0,1)^{n_c}$ cannot depend on style, i.e., it is a function from $\Ccal$ to $(0,1)^{n_c}$.
Finally, we show that $\hb$ must be invertible 
since it maps 
$p_\cb$ to a uniform distribution,
using a result from~\cite{zimmermann2021contrastive}.

\textbf{Intuition.}
\cref{thm:CL_MaxEnt} states that if we do not explicitly enforce invertibility of~$\gb$ as in~\cref{thm:CL}, additionally maximising the entropy of the learnt representation (i.e., optimising alignment \textit{and} uniformity~\cite{wang2020understanding})  avoids a collapsed representation and recovers the invariant content block. Intuitively, this is because any function that only depends on $\cb$ will be invariant across $(\xb,\xbt)$, so it is beneficial to preserve all content information to maximise entropy. 

\textbf{Discussion.
}
Of our theoretical results, \cref{thm:CL_MaxEnt} requires the weakest set of assumptions, and is most closely aligned with common SSL practice.
As discussed in~\cref{sec:background_augmentation}, contrastive SSL with negative samples using InfoNCE~\eqref{eq:InfoNCE_objective} as an objective 
can asymptotically be understood as alignment with entropy regularisation~\cite{wang2020understanding}, i.e., objective~\eqref{eq:CL_MSE_MaxEnt_objective}.
\textit{\Cref{thm:CL_MaxEnt} thus provides a theoretical justification for the empirically observed effectiveness of CL with InfoNCE}:
subject to our assumptions, CL with InfoNCE asymptotically isolates content, i.e., the part of the representation that is always left invariant by augmentation.
For example, the 
strong image classification performance based on representations learned by \texttt{SimCLR}~\cite{chen2020simple},
which uses color distortion and random crops as augmentations, can be explained in that object class is a content variable in this case.
We extensively evaluate the effect of various augmentation techniques on different ground-truth latent factors in our experiments in~\cref{sec:experiments}.
There is also an interesting connection between~\cref{thm:CL_MaxEnt} and \texttt{BarlowTwins}~\cite{zbontar2021barlow}, which only uses positive pairs and combines alignment with a redundancy reduction regulariser that enforces decorrelation between the inferred latents.
Intuitively, redundancy reduction is related to increased entropy: $\gb^\star$ constructed in the proof of~\cref{thm:CL_MaxEnt}---and thus also any other minimiser of~\eqref{eq:CL_MSE_MaxEnt_objective}---attains the global optimum of the \texttt{BarlowTwins} objective, though the reverse implication may not hold.

\section{Experiments}
\label{sec:experiments}

We perform two main experiments.
First,
we numerically test our main result,~\cref{thm:CL_MaxEnt},
in a \textit{fully-controlled}, finite sample setting~(\cref{sec:experiment_1_numerical_simulation}), using CL to estimate the entropy term in~\eqref{eq:CL_MSE_MaxEnt_objective}.
Second, we seek to better understand the effect of data augmentations used \textit{in practice}~(\cref{sec:experiment_2_causal3dident}).
To this end, we introduce a new
dataset of 3D objects with 
dependencies between a number of known ground-truth factors, and use it to evaluate the effect of different augmentation techniques
on what is identified as content.
{Additional experiments are summarised in~\cref{sec:additional_experiments} and described in more detail in~\Cref{app:additional_results}.}

\subsection{Numerical data}
\label{sec:experiment_1_numerical_simulation}

\textbf{Experimental setup.}
We generate synthetic data as described in~\cref{sec:problem_formulation}. We consider $n_c=n_s=5$, with content and style latents distributed as $\cb\sim\Ncal(0,\Sigma_c)$ and $\sb|\cb\sim\Ncal(\ab+B\cb, \Sigma_s)$, thus allowing for \emph{statistical dependence} within the two blocks (via $\Sigma_c$ and $\Sigma_s$) and \emph{causal dependence} between content and style (via $B$). For $\fb$, we use a
3-layer MLP with LeakyReLU activation functions.\footnote{chosen to lead to invertibility almost surely by following the settings used by previous work~\citep{hyvarinen2016unsupervised,hyvarinen2017nonlinear}}
The distribution $p_A$ over subsets of changing style variables is obtained by independently flipping the same biased coin for each $s_i$.
The conditional style distribution is taken as $p_{\sbt_A|\sb_A}=\Ncal(\sb_A,\Sigma_A)$.
We train an encoder $\gb$ on pairs $(\xb,\xbt)$ with InfoNCE using the negative L2 loss as the similarity measure, i.e., we approximate~\eqref{eq:CL_MSE_MaxEnt_objective} using empirical averages and negative samples. 
For evaluation, we 
use kernel ridge regression~\citep{murphybook} to predict the ground truth $\cb$ and $\sb$ from the learnt representation $\cbh=\gb(\xb)$ and report the $R^2$ coefficient of determination. For a more detailed account, we refer to~\Cref{app:experimental_details}.

\begin{wraptable}{r}{0.45\textwidth}
    \vspace{-1.0em}
    \centering
    \resizebox{0.45\textwidth}{!}{
    \small
    \begin{tabular}{ccccc}
    \toprule
    \multicolumn{3}{c}{\textbf{Generative process}} & \multicolumn{2}{c}{$\bm{R^2}$ \textbf{(nonlinear)}}  \\
    \cmidrule(r){1-3}\cmidrule(r){4-5}
    \textbf{p(chg.)} & \textbf{Stat.} & \textbf{Cau.} & \textbf{Content $\cb$} & \textbf{Style $\sb$} \\
    \midrule
    1.0 & \xmark & \xmark & $\textbf{1.00} \pm 0.00$ & $\textcolor{red}{0.07} \pm 0.00$ \\
    0.75 & \xmark & \xmark & $\textbf{1.00} \pm 0.00$ & $\textcolor{red}{0.06} \pm 0.05$ \\
    0.75 & \cmark & \xmark & $\textbf{0.98} \pm 0.03$ & ${0.37} \pm 0.05$ \\
    0.75 & \cmark & \cmark & $\textbf{0.99} \pm 0.01$ & $\textbf{0.80} \pm 0.08$ \\
    \bottomrule
    \end{tabular}
    }
    \vspace{-0.5em}
\end{wraptable}
\textbf{Results.}
In the inset table, we report mean $\pm$ std.\ dev.\ over $3$ random seeds across four generative processes of increasing complexity covered by~\cref{thm:CL_MaxEnt}: ``p(chg.)'', ``Stat.'', and ``Cau.'' denote respectively the change probability for each~$s_i$,
statistical dependence within blocks ($\Sigma_c\neq I\neq \Sigma_s$), and  causal dependence of style on content ($B\neq 0$).
An $R^2$ close to one indicates that almost all variation is explained by $\cbh$, i.e., that there is a 1-1 mapping, 
as required by~\cref{def:block-identifiability}.
As can be seen, \textit{across all settings, content is block-identified}. %
Regarding style, we observe an increased score with the introduction of dependencies, which we explain in an extended discussion in~\Cref{app:numerical}.
Finally, we show {in~\Cref{app:numerical}} that a high $R^2$ score can be obtained even if we use linear regression to predict $\cb$ from $\cbh$ ($R^2=0.98 \pm 0.01$, for the last row). 

\subsection{High-dimensional images: \textit{Causal3DIdent}}
\label{sec:experiment_2_causal3dident}

\begin{figure}[t]
    \renewcommand{\xshift}{2em}
    \renewcommand{\yshift}{1.em}
    \newcommand{\nodesize}{2.5em}
    \vspace{-0.0em}
    \begin{subfigure}{0.35\textwidth}
    \centering
        \begin{tikzpicture}
            \centering
            \node (spotlight angle)
            [latent, minimum size=\nodesize]
            {\small $\text{pos}_\text{spl}$};
            \node (class)
            [latent, right=of spotlight angle, xshift=-\xshift, minimum size=\nodesize]
            {\small class};
            \node (object position)
            [latent, below=of spotlight angle, yshift=\yshift, minimum size=\nodesize]
            {\small $\text{pos}_\text{obj}$};
            \node (object rotation)
            [latent, below=of class, yshift=\yshift, minimum size=\nodesize]
            {\small $\text{rot}_\text{obj}$};
            \node (background hue) 
            [latent, right=of class, xshift=-\xshift, minimum size=\nodesize] 
            {\small $\text{hue}_\text{bg}$};
            \node (object hue)
            [latent, below=of background hue, yshift=\yshift, minimum size=\nodesize]
            {\small $\text{hue}_\text{obj}$};
            \node (spotlight hue)
            [latent, right=of background hue, xshift=-\xshift, minimum size=\nodesize]
            {\small $\text{hue}_\text{spl}$};
            \edge{class}{object hue, object position, object rotation};
            \edge[]{spotlight angle}{object position};
            \edge[]{spotlight hue, background hue}{object hue};
        \end{tikzpicture}
    \end{subfigure}%
    \begin{subfigure}{0.65\textwidth}
        \centering
        \includegraphics[width=\textwidth]{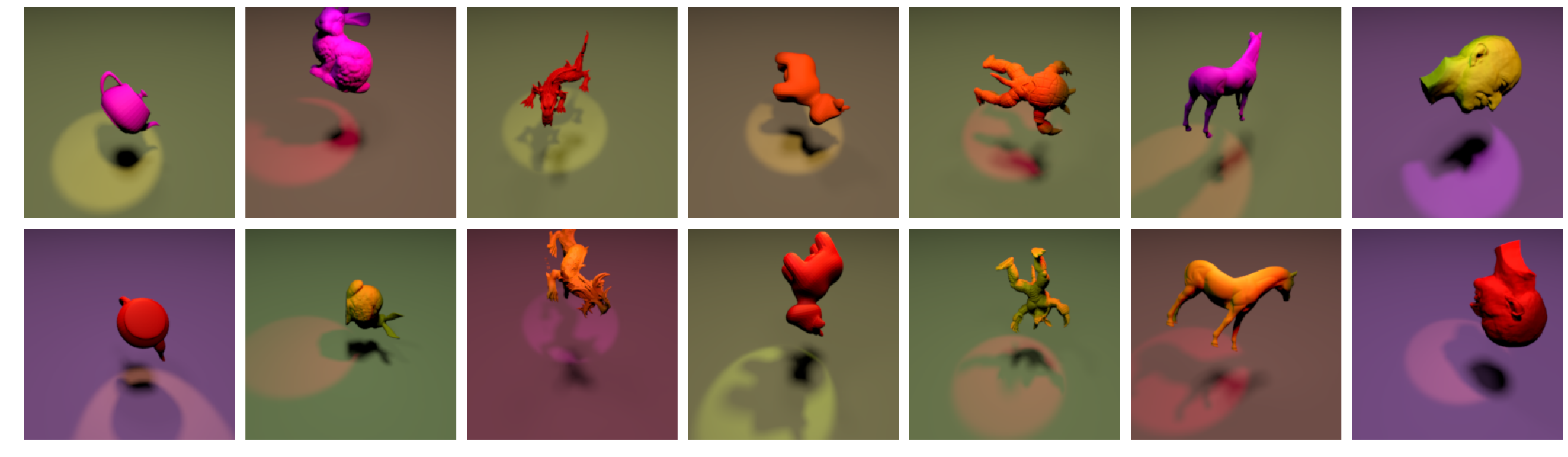}
    \end{subfigure}
    \caption{
    \small
    \textit{(Left)} Causal graph for the \textit{Causal3DIdent} dataset. 
    \textit{(Right)} Two samples from each object class.
    }
    \label{fig:3DIdent_causal_graph}
\end{figure}
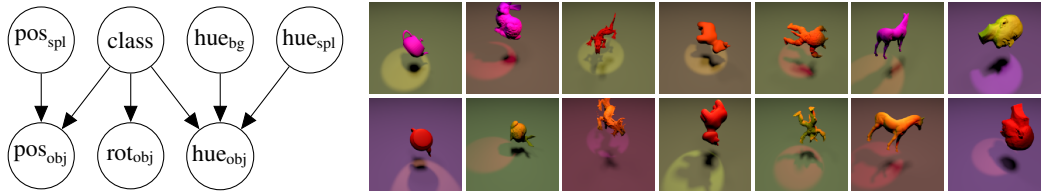

\textbf{\textit{Causal3DIdent} dataset.}
\textit{3DIdent}~\citep{zimmermann2021contrastive} is a benchmark for evaluating identifiability with rendered $224\times224$ images which contains hallmarks of natural environments (e.g. shadows, different lighting conditions, a 3D object). For influence of the latent factors on the renderings, see Fig.~2 of~\citep{zimmermann2021contrastive}.
In \textit{3DIdent}, there is a single object class (Teapot~\citep{newell1975utah}), and all $10$ latents are sampled independently. For \textit{Causal3DIdent}, we introduce \textbf{six} additional classes: Hare~\citep{turk1994bunny}, Dragon~\citep{StanfordScanRep}, Cow~\citep{KeenanScanRep}, Armadillo~\citep{Krishnamurthy1996armadillo}, Horse~\citep{Praun2000horse}, and Head~\citep{SuggestContourGallery}; and impose a causal graph over the latent variables, see~\cref{fig:3DIdent_causal_graph}. 
While object class and all environment variables (spotlight position \& hue, background hue) are sampled independently, all object latents are dependent,\footnote{e.g., our causal graph entails hares blend into the environment (object hue centered about background \& spotlight hue), a form of active camouflage observed in Alaskan~\citep{Lepusothus}, Arctic~\citep{arctichare}, \& Snowshoe hares.}
see~\Cref{app:data_set_details} for  details.\footnote{We made the Causal3DIdent dataset   \href{https://zenodo.org/record/4784282}{publicly available at this URL}.}  

\begin{wrapfigure}[]{!h}{0.4\textwidth}
\centering
    \vspace{-1.5em}
    \includegraphics[width=0.4\textwidth]{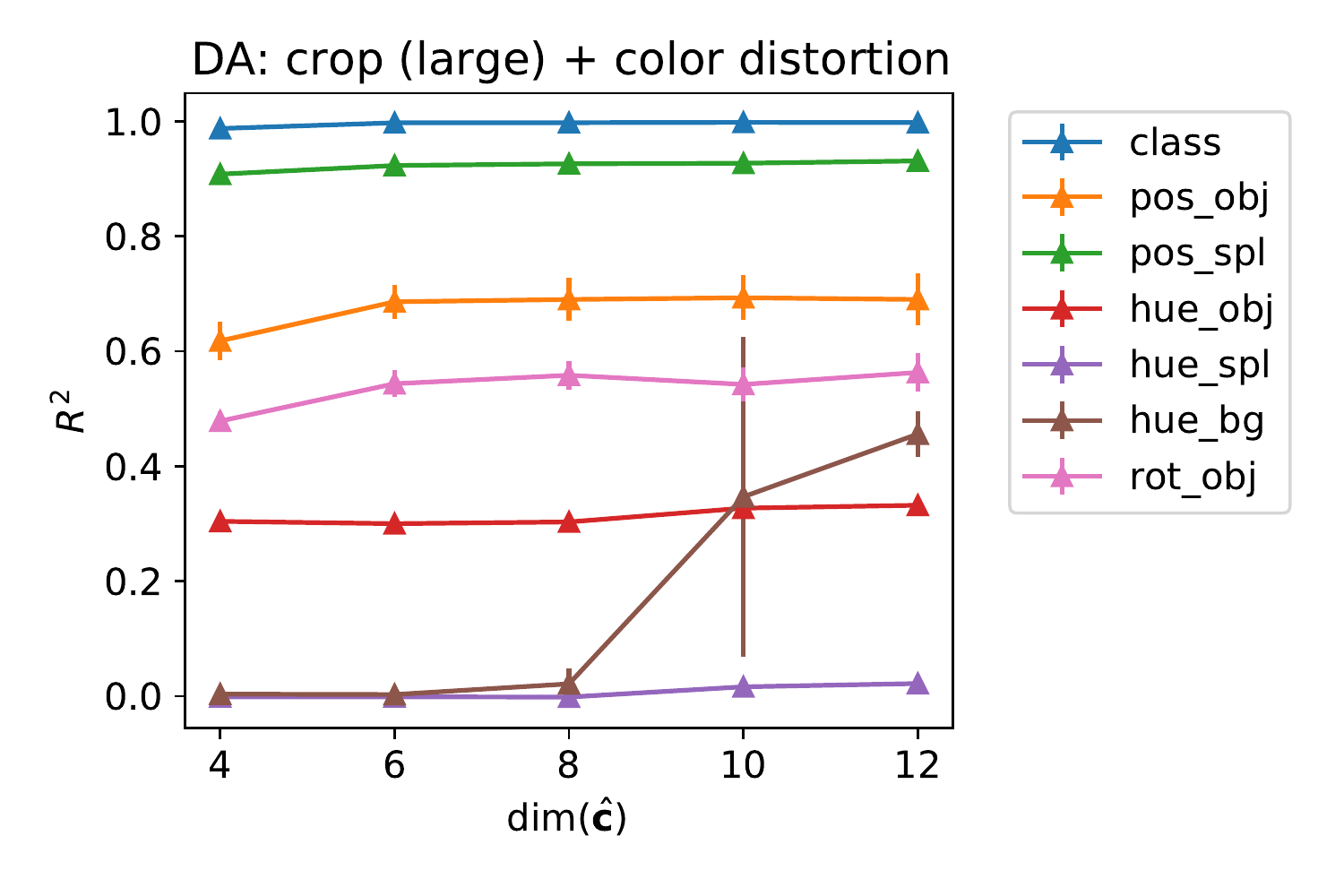}
    \vspace{-2.5em}
\end{wrapfigure}
\textbf{Experimental setup.}
For $\gb$, we train a convolutional
encoder composed of a ResNet18%
~\citep{he2015deep} and an additional fully-connected layer, with LeakyReLU activation.
As in \texttt{SimCLR}~\citep{chen2020simple}, we use
InfoNCE with cosine 
similarity, 
and train on pairs of augmented examples~$(\xbt,\xbt')$.
As $n_c$ is unknown and variable depending on the augmentation, we fix $\text{dim}(\cbh)=8$ throughout.
Note that we
find the results to be, for the most part, robust to the choice of $\text{dim}(\cbh)$,
see inset figure.
We consider the following data augmentations (DA): crop, resize \& flip; colour distortion (jitter \& drop); and rotation {\small$\in\{90\degree,180\degree,270\degree\}$}. For comparison, we also consider directly imposing a content-style partition
by
performing a latent transformation (LT) to generate views. 
\looseness-1 For evaluation, we use linear logistic regression to predict object class, and 
kernel ridge 
to predict the other latents from $\cbh$.\footnote{
See~\Cref{app:causal} for results with linear regression, as well as evaluation using a higher-dimensional intermediate layer by considering a projection head~\citep{chen2020simple}.}

\definecolor{LTcolor}{rgb}{0.95,1,1}
\definecolor{DAcolor}{rgb}{1,1,0.95}
\begin{table}[t]
\centering
\caption{
\small 
\textit{Causal3DIdent} results: $R^2$ mean $\pm$ std.\ dev.\  over $3$ random seeds. DA: data augmentation, LT: latent transformation, bold: $R^2\geq 0.5$, red: $R^2<0.25$.
Results for individual axes of object position \& rotation are aggregated, see~\Cref{app:additional_results} for the full table.
}
\vspace{0.5em}
\label{tbl:final_nonlinear_abbrv}
\resizebox{\textwidth}{!}{
\small
\begin{tabular}{lc|cc|ccc|c}
\toprule
\multirow{2}{*}{\textbf{Views generated by}} & \multirow{2}{*}{\textbf{Class}} & \multicolumn{2}{c}{\textbf{Positions}} & \multicolumn{3}{c}{\textbf{Hues}} & \multirow{2}{*}{\textbf{Rotations}} \\
\cmidrule(r){3-4}\cmidrule(r){5-7}
& & $\text{object}$ & $\text{spotlight}$ & $\text{object}$ & $\text{spotlight}$ & $\text{background}$ & 
\\
\midrule
\rowcolor{DAcolor}
DA: colour distortion  & 
$0.42 \pm 0.01$ & $\textbf{0.61} \pm 0.10$ & $\textcolor{red}{0.17} \pm 0.00$ & $\textcolor{red}{0.10} \pm 0.01$ & $\textcolor{red}{0.01} \pm 0.00$ & $\textcolor{red}{0.01} \pm 0.00$ & $0.33 \pm 0.02$ \\
\rowcolor{LTcolor}
LT: change hues & 
$\textbf{1.00} \pm 0.00$ & $\textbf{0.59} \pm 0.33$ & $\textbf{0.91} \pm 0.00$ & $0.30 \pm 0.00$ & $\textcolor{red}{0.00} \pm 0.00$ & $\textcolor{red}{0.00} \pm 0.00$ & $0.30 \pm 0.01$ \\
\midrule
\rowcolor{DAcolor}
DA: crop (large) & 
$0.28 \pm 0.04$ & $\textcolor{red}{0.09} \pm 0.08$ & $\textcolor{red}{0.21} \pm 0.13$ & $\textbf{0.87} \pm 0.00$ & $\textcolor{red}{0.09} \pm 0.02$ & $\textbf{1.00} \pm 0.00$ & $\textcolor{red}{0.02} \pm 0.02$ \\
\rowcolor{DAcolor}
DA: crop (small) & 
$\textcolor{red}{0.14} \pm 0.00$ & $\textcolor{red}{0.00} \pm 0.01$ & $\textcolor{red}{0.00} \pm 0.01$ & $\textcolor{red}{0.00} \pm 0.00$ & $\textcolor{red}{0.00} \pm 0.00$ & $\textbf{1.00} \pm 0.00$ & $\textcolor{red}{0.00} \pm 0.00$ \\
\rowcolor{LTcolor}
LT: change positions & 
$\textbf{1.00} \pm 0.00$ & $\textcolor{red}{0.16} \pm 0.23$ & $\textcolor{red}{0.00} \pm 0.01$ & $0.46 \pm 0.02$ & $\textcolor{red}{0.00} \pm 0.00$ & $\textbf{0.97} \pm 0.00$ & $0.29 \pm 0.01$ \\
\midrule
\rowcolor{DAcolor}
DA: crop (large) + colour distortion & %
$\textbf{0.97} \pm 0.00$ & $\textbf{0.59} \pm 0.07$ & $\textbf{0.59} \pm 0.05$ & $0.28 \pm 0.00$ & $\textcolor{red}{0.01} \pm 0.01$ & $\textcolor{red}{0.01} \pm 0.00$ & $\textbf{0.74} \pm 0.03$ \\
\rowcolor{DAcolor}
DA: crop (small) + colour distortion & %
$\textbf{1.00} \pm 0.00$ & $\textbf{0.69} \pm 0.04$ & $\textbf{0.93} \pm 0.00$ & $0.30 \pm 0.01$ & $\textcolor{red}{0.00} \pm 0.00$ & $\textcolor{red}{0.02} \pm 0.03$ & $\textbf{0.56} \pm 0.03$ \\
\rowcolor{LTcolor}
LT: change positions + hues & 
$\textbf{1.00} \pm 0.00$ & $\textcolor{red}{0.22} \pm 0.22$ & $\textcolor{red}{0.07} \pm 0.08$ & $0.32 \pm 0.02$ & $\textcolor{red}{0.00} \pm 0.01$ & $\textcolor{red}{0.02} \pm 0.03$ & $0.34 \pm 0.06$ \\
\midrule
\rowcolor{DAcolor}
DA: rotation &
$0.33 \pm 0.06$ & $\textcolor{red}{0.17} \pm 0.09$ & $\textcolor{red}{0.23} \pm 0.12$ & $\textbf{0.83} \pm 0.01$ & $0.30 \pm 0.12$ & $\textbf{0.99} \pm 0.00$ & $\textcolor{red}{0.05} \pm 0.03$ \\
\rowcolor{LTcolor}
LT: change rotations & 
$\textbf{1.00} \pm 0.00$ & $\textbf{0.53} \pm 0.33$ & $\textbf{0.90} \pm 0.00$ & $0.41 \pm 0.00$ & $\textcolor{red}{0.00} \pm 0.00$ & $\textbf{0.97} \pm 0.00$ & $0.28 \pm 0.00$ \\
\midrule
\rowcolor{DAcolor}
DA: rotation + colour distortion & 
$\textbf{0.59} \pm 0.01$ & $\textbf{0.58} \pm 0.06$ & $\textcolor{red}{0.21} \pm 0.01$ & $\textcolor{red}{0.12} \pm 0.02$ & $\textcolor{red}{0.01} \pm 0.00$ & $\textcolor{red}{0.01} \pm 0.00$ & $0.33 \pm 0.04$ \\
\rowcolor{LTcolor}
LT: change rotations + hues &
$\textbf{1.00} \pm 0.00$ & $\textbf{0.57} \pm 0.34$ & $\textbf{0.91} \pm 0.00$ & $0.30 \pm 0.00$ & $\textcolor{red}{0.00} \pm 0.00$ & $\textcolor{red}{0.00} \pm 0.00$ & $0.28 \pm 0.00$ \\
\bottomrule
\end{tabular}
}
\vspace{-0.5em}
\end{table}

\textbf{Results.}
The results are presented in~\cref{tbl:final_nonlinear_abbrv}.
Overall, our main findings can be summarised as:
\begin{enumerate}[label=(\roman*), topsep=0pt,itemsep=0pt,leftmargin=*]
\item it can be difficult to design image-level augmentations that leave \textit{specific} latent factors invariant;
\item augmentations \& latent transformations generally have a similar effect on groups of latents;
\item augmentations that yield good classification performance induce variation in all other latents.
\end{enumerate}
We 
observe that, similar to directly varying the hue latents, colour distortion leads to
a discarding of hue information as style, and a preservation of (object) position
as content.
Crops, similar to varying the position latents, lead to a discarding of position
as style, and a preservation of background and object hue 
as content, the latter assuming crops are sufficiently large.
\looseness-1 In contrast, image-level rotation 
affects both the object rotation and position, and thus
deviates
from only varying the rotation latents.

Whereas class is always preserved as content when generating views with latent transformations,
when using data augmentations,
we can only reliably decode class when crops \& colour distortion are used in conjunction---a result which mirrors evaluation on ImageNet~\citep{chen2020simple}. As can be seen by our evaluation of crops \& colour distortion in isolation, while colour distortion leads to a discarding of hues as style, crops lead to a discarding of position \& rotation as style. Thus, when used in conjunction, class is isolated as the sole content variable. 
{See~\Cref{app:causal} for additional analysis.}

\subsection{{Additional experiments and ablations}}
\label{sec:additional_experiments}
We also perform an ablation 
on $\text{dim}(\cbh)$ 
for the synthetic setting from~\cref{sec:experiment_1_numerical_simulation}, { see~\Cref{app:numerical} for details}.
Generally, we find that if $\text{dim}(\cbh)<n_c$,
there is insufficient capacity to encode all content, so a lower-dimensional mixture of content is learnt. 
Conversely, if $\text{dim}(\cbh)>n_c$, the excess capacity is used to encode some style information (as that increases entropy).
Further, we repeat our analysis from~\cref{sec:experiment_2_causal3dident} using \texttt{BarlowTwins}~\cite{zbontar2021barlow} (instead of \texttt{SimCLR}) which, as discussed at the end of~\cref{sec:theory_discriminative}, is also loosely related to~\cref{thm:CL_MaxEnt}. The results mostly mirror those obtained for~\texttt{SimCLR} and  presented in~\cref{tbl:final_nonlinear_abbrv}, {see~\Cref{app:causal} for details}.
Finally, we ran the same experimental setup as in~\cref{sec:experiment_2_causal3dident} also on the \textit{MPI3D-real} dataset~\cite{gondal2019transfer} containing $>1$ million \textit{real} images with ground-truth annotations of 3D objects being moved by a robotic arm.
Subject to some caveats, the results show a similar trend as those on \textit{Causal3DIdent}, {see~\Cref{app:mpi} for details}.

\section{Discussion}
\label{sec:discussion}

\textbf{Theory vs practice.}
We have made an effort to tailor our problem formulation~(\cref{sec:problem_formulation}) to the setting of data augmentation with content-preserving transformations.
However, some of our more technical assumptions, which are necessary to 
prove block-identifiability of the invariant content partition, may not hold exactly in practice.
This is apparent, e.g., from our second experiment~(\cref{sec:experiment_2_causal3dident}), where we observe that---while class should, in principle, always be invariant across views (i.e., content)---when using
\textit{only} crops, colour distortion, or rotation,  $\gb$ appears to encode \emph{shortcuts}~\citep{geirhos2020shortcuts,pezeshki2020gradient}.%
\footnote{class is distinguished by shape, a feature commonly 
unused in 
downstream tasks on natural images%
~\citep{geirhos2019imagenettrained}}
Data augmentation, unlike latent transformations, generates views~$\xbt$ which 
are not restricted to the 11-dim.\ image
manifold~$\Xcal$ 
corresponding to the generative process of \textit{Causal3DIdent}, but may introduce additional variation: e.g., colour distortion leads to a rich combination of colours, typically a 3-dim.\ feature,
whereas \textit{Causal3DIdent} only contains one degree of freedom (hue).
With additional factors, any introduced invariances may be encoded as content in place of class.
Image-level augmentations also tend to change multiple latent factors in a correlated way, which may violate assumption \textit{(iii)} of our theorems, i.e., that $p_{\sbt_\A|\sb_\A}$ is fully-supported locally.
{We also assume that $\zb$ is continuous, even though \textit{Causal3DIdent} and most disentanglement datasets also contain discrete latents. This is a very common assumption in the related literature~\cite{locatello2020weakly,gresele2019incomplete,zimmermann2021contrastive,klindt2020slowvae,locatello2019challenging,khemakhem2020variational,hyvarinen1999nonlinear,hyvarinen2019nonlinear,hyvarinen2016unsupervised} that may be relaxed in future work.}
Moreover, our theory holds asymptotically and at the global optimum, whereas in practice we solve a non-convex optimisation problem with a finite sample and need to approximate the entropy term in~\eqref{eq:CL_MSE_MaxEnt_objective}, e.g., using a finite number of negative pairs. The resulting  challenges for optimisation may be further accentuated by the higher dimensionality of $\Xcal$ induced by image-level augmentations.
Finally, we remark that while, for simplicity, we have presented our theory for pairs $(\xb,\xbt)$ of original and augmented examples, in practice, using pairs $(\xbt,\xbt')$ of two augmented views typically yields better performance. All of our assumptions (content invariance, changing style, etc) and theoretical results still apply to the latter case. We believe that using two augmented views helps because it leads to \textit{increased variability} across the pair: for if $\xbt$ and $\xbt'$ differ from $\xb$ in style subsets $A$ and $A'$, respectively, then $(\xbt,\xbt')$ differ from each other (a.s.) in the union~$A\cup A'$.

\textbf{Beyond entropy regularisation.}
We have shown a clear link between an identifiable maximum entropy approach to SSL~(\cref{thm:CL_MaxEnt}) and \texttt{SimCLR}~\cite{chen2020simple} based on the analysis of~\cite{wang2020understanding}, and have discussed an intuitive connection to the notion of redundancy reduction used in \texttt{BarlowTwins}~\cite{zbontar2021barlow}. Whether other types of regularisation such as the architectural approach pursued in \texttt{BYOL}~\cite{grill2020byol} and \texttt{SimSiam}~\cite{chen2020exploring} can also be linked to entropy maximisation, remains an open question.
Deriving similar results to~\cref{thm:CL_MaxEnt} with other regularisers is a promising direction for future research, c.f.~\cite{tian2021understanding}.

\textbf{The choice of augmentation technique implicitly defines content and style.}
As we have defined content as the part of the representation which is always left invariant across views, the choice of augmentation implicitly determines the content-style partition.
This is particularly important
to keep in mind
when applying SSL with data augmentation to safety-critical domains, such as medical imaging.
We also advise caution when using data augmentation to identify specific latent properties, since, as observed in~\cref{sec:experiment_2_causal3dident}, image-level transformations may affect the underlying ground-truth factors in unanticipated ways.
Also note that, \textit{for a given downstream task}, we may not want to discard all style information since style variables may still be correlated with the task of interest and may thus help improve predictive performance. \textit{For arbitrary downstream tasks}, however, where style may change in an adversarial way, it can be shown that only using content is optimal~\cite{rojas2018invariant}.

{\textbf{\textit{What} vs \textit{how} information is encoded.}
We focus on \textit{what} information is learnt by SSL with data augmentations by specifying a generative process and studying identifiability of the latent representation.
Orthogonal to this, a different line of work instead studies \textit{how} information is encoded by analysing the sample complexity needed to solve a \textit{given downstream task} using a \textit{linear} predictor~\cite{arora2019theoretical,tosh2020contrastive,lee2020predicting,tosh2021contrastive,tsai2020self,tian2021understanding}.
Provided that downstream tasks only involve content, we can draw some comparisons. 
Whereas our results recover content only up to arbitrary invertible nonlinear functions~(see~\cref{def:block-identifiability}), our problem setting is more general: \cite{arora2019theoretical,lee2020predicting} assume (approximate) independence of views $(\xb,\xbt)$ given the task (content), while \cite{tosh2021contrastive,tsai2020self} assume (approximate) independence between one view and the task (content) given the other view, neither of which hold in our setting. 
}

\textbf{Conclusion.}
Existing representation learning approaches typically assume mutually independent latents, though 
dependencies clearly exist in nature~\cite{scholkopf2021toward}. 
We demonstrate that
in a \textit{non-i.i.d.}~scenario, e.g., by constructing multiple views of the same example
with data augmentation,
we can learn useful representations in the presence of this neglected phenomenon. 
More specifically, the present work 
contributes, to the best of our knowledge,
the first: (i) identifiability result under
\textit{arbitrary dependence} between
latents;
and (ii) empirical study that
evaluates the effect of data augmentations 
not only on classification, but also on other \textit{continuous}  ground-truth latents.
Unlike existing identifiability results which 
rely on \textit{change} as a learning signal, our approach 
aims to identify 
what is always shared across views, i.e., also using \textit{invariance} as a learning signal.
We hope that this change in perspective
will be helpful for
applications such as optimal style transfer
or disentangling shape from pose in vision, 
{and inspire other types of \textit{counterfactual training} to recover a more fine-grained causal representation.}

\clearpage

\section*{Acknowledgements}
We thank: the anonymous reviewers for several helpful suggestions that triggered improvements in theory and additional experiments; Cian Eastwood, Ilyes Khemakem, Michael Lohaus, Osama Makansi, Ricardo Pio Monti, Roland Zimmermann, Weiyang Liu, and the MPI T\"ubingen causality group for helpful discussions and comments; Hugo Yèche for pointing out a mistake in~\S 2 of an earlier version of the manuscript; 
github user \texttt{TangTangFei} for catching a bug in the implementation of the experiments from~\cref{sec:experiment_1_numerical_simulation} (that has been corrected in this version);
and the International Max Planck Research School for Intelligent Systems (IMPRS-IS) for supporting YS.

\section*{Funding Transparency Statement}
WB acknowledges support via his Emmy Noether Research Group funded by the German Science Foundation (DFG) under grant no. BR 6382/1-1 as well as support by Open Philantropy and the Good Ventures Foundation.
This work was supported by the German Federal Ministry of Education and Research (BMBF): Tübingen AI Center, FKZ: 01IS18039A, 01IS18039B; and by the Machine Learning Cluster of Excellence, EXC number 2064/1 – Project number 390727645.

{\small
\bibliographystyle{plainnat}
\bibliography{references}

\begin{thebibliography}{129}
\providecommand{\natexlab}[1]{#1}
\providecommand{\url}[1]{\texttt{#1}}
\expandafter\ifx\csname urlstyle\endcsname\relax
  \providecommand{\doi}[1]{doi: #1}\else
  \providecommand{\doi}{doi: \begingroup \urlstyle{rm}\Url}\fi

\bibitem[Agrawal et~al.(2015)Agrawal, Carreira, and Malik]{agrawal2015learning}
Pulkit Agrawal, Joao Carreira, and Jitendra Malik.
\newblock Learning to see by moving.
\newblock In \emph{Proceedings of the IEEE International Conference on Computer
  Vision}, pages 37--45, 2015.

\bibitem[Arctic Wildlife()]{arctichare}
Arctic Wildlife.
\newblock {Churchill Polar Bears }.
\newblock \text{https://churchillpolarbears.org/churchill/}, 2021.

\bibitem[Arora et~al.(2019)Arora, Khandeparkar, Khodak, Plevrakis, and
  Saunshi]{arora2019theoretical}
Sanjeev Arora, Hrishikesh Khandeparkar, Mikhail Khodak, Orestis Plevrakis, and
  Nikunj Saunshi.
\newblock A theoretical analysis of contrastive unsupervised representation
  learning.
\newblock In \emph{36th International Conference on Machine Learning}, pages
  9904--9923. International Machine Learning Society (IMLS), 2019.

\bibitem[Bachman et~al.(2019)Bachman, Hjelm, and
  Buchwalter]{bachman2019learning}
Philip Bachman, R.~Devon Hjelm, and William Buchwalter.
\newblock Learning representations by maximizing mutual information across
  views.
\newblock In \emph{Advances in Neural Information Processing Systems 32}, pages
  15509--15519, 2019.

\bibitem[Baevski et~al.(2020{\natexlab{a}})Baevski, Schneider, and
  Auli]{Baevski2020vqwav2vec}
Alexei Baevski, Steffen Schneider, and Michael Auli.
\newblock vq-wav2vec: Self-supervised learning of discrete speech
  representations.
\newblock In \emph{8th International Conference on Learning Representations,
  {ICLR} 2020, Addis Ababa, Ethiopia, April 26-30, 2020}. OpenReview.net,
  2020{\natexlab{a}}.

\bibitem[Baevski et~al.(2020{\natexlab{b}})Baevski, Zhou, Mohamed, and
  Auli]{baevski2020wav2vec}
Alexei Baevski, Yuhao Zhou, Abdelrahman Mohamed, and Michael Auli.
\newblock wav2vec 2.0: {A} framework for self-supervised learning of speech
  representations.
\newblock In \emph{Advances in Neural Information Processing Systems 33},
  2020{\natexlab{b}}.

\bibitem[Becker and Hinton(1992)]{becker1992self}
Suzanna Becker and Geoffrey~E Hinton.
\newblock Self-organizing neural network that discovers surfaces in random-dot
  stereograms.
\newblock \emph{Nature}, 355\penalty0 (6356):\penalty0 161--163, 1992.

\bibitem[Behrmann et~al.(2019)Behrmann, Grathwohl, Chen, Duvenaud, and
  Jacobsen]{behrmann2019invertible}
Jens Behrmann, Will Grathwohl, Ricky~TQ Chen, David Duvenaud, and
  J{\"o}rn-Henrik Jacobsen.
\newblock Invertible residual networks.
\newblock In \emph{International Conference on Machine Learning}, pages
  573--582. PMLR, 2019.

\bibitem[Bell and Sejnowski(1995)]{bell1995information}
Anthony~J Bell and Terrence~J Sejnowski.
\newblock An information-maximization approach to blind separation and blind
  deconvolution.
\newblock \emph{Neural computation}, 7\penalty0 (6):\penalty0 1129--1159, 1995.

\bibitem[Bengio et~al.(2013)Bengio, Courville, and
  Vincent]{bengio2013representation}
Yoshua Bengio, Aaron Courville, and Pascal Vincent.
\newblock Representation learning: A review and new perspectives.
\newblock \emph{IEEE Transactions on Pattern Analysis and Machine
  Intelligence}, 35\penalty0 (8):\penalty0 1798--1828, 2013.

\bibitem[{Blender Online Community}(2021)]{blender}
{Blender Online Community}.
\newblock \emph{Blender - a 3D modelling and rendering package}.
\newblock Blender Foundation, Blender Institute, Amsterdam, 2021.

\bibitem[Bromley et~al.(1993)Bromley, Guyon, LeCun, S{\"a}ckinger, and
  Shah]{bromley1993signature}
Jane Bromley, Isabelle Guyon, Yann LeCun, Eduard S{\"a}ckinger, and Roopak
  Shah.
\newblock Signature verification using a ``siamese'' time delay neural network.
\newblock \emph{Advances in Neural Information Processing Systems}, 6:\penalty0
  737--744, 1993.

\bibitem[Brown et~al.(2020)Brown, Mann, Ryder, Subbiah, Kaplan, Dhariwal,
  Neelakantan, Shyam, Sastry, Askell, Agarwal, Herbert-Voss, Krueger, Henighan,
  Child, Ramesh, Ziegler, Wu, Winter, Hesse, Chen, Sigler, Litwin, Gray, Chess,
  Clark, Berner, McCandlish, Radford, Sutskever, and Amodei]{GPT3}
Tom Brown, Benjamin Mann, Nick Ryder, Melanie Subbiah, Jared~D Kaplan, Prafulla
  Dhariwal, Arvind Neelakantan, Pranav Shyam, Girish Sastry, Amanda Askell,
  Sandhini Agarwal, Ariel Herbert-Voss, Gretchen Krueger, Tom Henighan, Rewon
  Child, Aditya Ramesh, Daniel Ziegler, Jeffrey Wu, Clemens Winter, Chris
  Hesse, Mark Chen, Eric Sigler, Mateusz Litwin, Scott Gray, Benjamin Chess,
  Jack Clark, Christopher Berner, Sam McCandlish, Alec Radford, Ilya Sutskever,
  and Dario Amodei.
\newblock Language models are few-shot learners.
\newblock In \emph{Advances in Neural Information Processing Systems},
  volume~33, pages 1877--1901, 2020.

\bibitem[Burgess et~al.(2018)Burgess, Higgins, Pal, Matthey, Watters,
  Desjardins, and Lerchner]{burgess2018understanding}
Christopher~P Burgess, Irina Higgins, Arka Pal, Loic Matthey, Nick Watters,
  Guillaume Desjardins, and Alexander Lerchner.
\newblock Understanding disentangling in $\beta $-{VAE}.
\newblock \emph{arXiv preprint arXiv:1804.03599}, 2018.

\bibitem[Cardoso(1997)]{cardoso1997infomax}
J-F Cardoso.
\newblock Infomax and maximum likelihood for blind source separation.
\newblock \emph{IEEE Signal processing letters}, 4\penalty0 (4):\penalty0
  112--114, 1997.

\bibitem[Casey and Westner(2000)]{casey2000separation}
Michael~A Casey and Alex Westner.
\newblock Separation of mixed audio sources by independent subspace analysis.
\newblock In \emph{ICMC}, pages 154--161, 2000.

\bibitem[Chapelle and Sch{\"o}lkopf(2002)]{ChaSch02}
Olivier Chapelle and Bernhard Sch{\"o}lkopf.
\newblock Incorporating invariances in nonlinear {SVM}s.
\newblock In \emph{Advances in Neural Information Processing Systems 14}, pages
  609--616, Cambridge, MA, USA, 2002. MIT Press.

\bibitem[Chen et~al.(2018)Chen, Li, Grosse, and Duvenaud]{chen2018isolating}
Ricky~TQ Chen, Xuechen Li, Roger Grosse, and David Duvenaud.
\newblock Isolating sources of disentanglement in vaes.
\newblock In \emph{Advances in Neural Information Processing Systems}, pages
  2615--2625, 2018.

\bibitem[Chen et~al.(2020{\natexlab{a}})Chen, Dobriban, and Lee]{chen2020group}
Shuxiao Chen, Edgar Dobriban, and Jane~H Lee.
\newblock A group-theoretic framework for data augmentation.
\newblock \emph{Journal of Machine Learning Research}, 21\penalty0
  (245):\penalty0 1--71, 2020{\natexlab{a}}.

\bibitem[Chen et~al.(2020{\natexlab{b}})Chen, Kornblith, Norouzi, and
  Hinton]{chen2020simple}
Ting Chen, Simon Kornblith, Mohammad Norouzi, and Geoffrey Hinton.
\newblock A simple framework for contrastive learning of visual
  representations.
\newblock In \emph{International Conference on Machine Learning}, pages
  1597--1607. PMLR, 2020{\natexlab{b}}.

\bibitem[Chen and He(2021)]{chen2020exploring}
Xinlei Chen and Kaiming He.
\newblock Exploring simple siamese representation learning.
\newblock In \emph{{IEEE} Conference on Computer Vision and Pattern
  Recognition}, pages 15750--15758, 2021.

\bibitem[Chopra et~al.(2005)Chopra, Hadsell, and LeCun]{chopra2005learning}
S.~Chopra, R.~Hadsell, and Y.~LeCun.
\newblock Learning a similarity metric discriminatively, with application to
  face verification.
\newblock In \emph{IEEE Conference on Computer Vision and Pattern Recognition},
  volume~1, pages 539--546, 2005.

\bibitem[Collobert et~al.(2011)Collobert, Weston, Bottou, Karlen, Kavukcuoglu,
  and Kuksa]{collobert2011natural}
Ronan Collobert, Jason Weston, L{\'e}on Bottou, Michael Karlen, Koray
  Kavukcuoglu, and Pavel Kuksa.
\newblock Natural language processing (almost) from scratch.
\newblock \emph{Journal of Machine Learning Research}, 12:\penalty0 2461--2505,
  2011.

\bibitem[Comon(1994)]{comon1994independent}
Pierre Comon.
\newblock Independent component analysis, a new concept?
\newblock \emph{Signal processing}, 36\penalty0 (3):\penalty0 287--314, 1994.

\bibitem[Cover and Thomas(2012)]{cover2012elements}
Thomas~M Cover and Joy~A Thomas.
\newblock \emph{Elements of Information Theory}.
\newblock John Wiley \& Sons, 2012.

\bibitem[Cubuk et~al.(2019)Cubuk, Zoph, Mane, Vasudevan, and
  Le]{cubuk2019autoaugment}
Ekin~D Cubuk, Barret Zoph, Dandelion Mane, Vijay Vasudevan, and Quoc~V Le.
\newblock Autoaugment: Learning augmentation strategies from data.
\newblock In \emph{IEEE Conference on Computer Vision and Pattern Recognition},
  pages 113--123, 2019.

\bibitem[Cubuk et~al.(2020)Cubuk, Zoph, Shlens, and Le]{cubuk2020randaugment}
Ekin~D Cubuk, Barret Zoph, Jonathon Shlens, and Quoc~V Le.
\newblock Randaugment: Practical automated data augmentation with a reduced
  search space.
\newblock In \emph{IEEE Conference on Computer Vision and Pattern Recognition
  Workshops}, pages 702--703, 2020.

\bibitem[Dao et~al.(2019)Dao, Gu, Ratner, Smith, De~Sa, and
  R{\'e}]{dao2019kernel}
Tri Dao, Albert Gu, Alexander Ratner, Virginia Smith, Chris De~Sa, and
  Christopher R{\'e}.
\newblock A kernel theory of modern data augmentation.
\newblock In \emph{International Conference on Machine Learning}, pages
  1528--1537. PMLR, 2019.

\bibitem[Darmois(1951)]{darmois1951construction}
G~Darmois.
\newblock Analyse des liaisons de probabilit{\'e}.
\newblock In \emph{Proc. Int. Stat. Conferences 1947}, page 231, 1951.

\bibitem[Devlin et~al.(2019)Devlin, Chang, Lee, and Toutanova]{devlin2019bert}
Jacob Devlin, Ming-Wei Chang, Kenton Lee, and Kristina Toutanova.
\newblock {BERT}: Pre-training of deep bidirectional transformers for language
  understanding.
\newblock In \emph{Proceedings of the 2019 Conference of the North {A}merican
  Chapter of the Association for Computational Linguistics: Human Language
  Technologies}, pages 4171--4186, 2019.

\bibitem[DeVries and Taylor(2017)]{devries2017improved}
Terrance DeVries and Graham~W Taylor.
\newblock Improved regularization of convolutional neural networks with cutout.
\newblock \emph{arXiv preprint arXiv:1708.04552}, 2017.

\bibitem[Dinh et~al.(2014)Dinh, Krueger, and Bengio]{dinh2014nice}
Laurent Dinh, David Krueger, and Yoshua Bengio.
\newblock {NICE}: Non-linear independent components estimation.
\newblock \emph{arXiv preprint arXiv:1410.8516}, 2014.

\bibitem[Dinh et~al.(2017)Dinh, Sohl{-}Dickstein, and Bengio]{dinh2016density}
Laurent Dinh, Jascha Sohl{-}Dickstein, and Samy Bengio.
\newblock Density estimation using real {NVP}.
\newblock In \emph{5th International Conference on Learning Representations},
  2017.

\bibitem[Doersch et~al.(2015)Doersch, Gupta, and
  Efros]{doersch2015unsupervised}
Carl Doersch, Abhinav Gupta, and Alexei~A Efros.
\newblock Unsupervised visual representation learning by context prediction.
\newblock In \emph{IEEE International Conference on Computer Vision}, pages
  1422--1430, 2015.

\bibitem[Eberhardt and Scheines(2007)]{eberhardt2007interventions}
Frederick Eberhardt and Richard Scheines.
\newblock Interventions and causal inference.
\newblock \emph{Philosophy of science}, 74\penalty0 (5):\penalty0 981--995,
  2007.

\bibitem[Geirhos et~al.(2019)Geirhos, Rubisch, Michaelis, Bethge, Wichmann, and
  Brendel]{geirhos2019imagenettrained}
Robert Geirhos, Patricia Rubisch, Claudio Michaelis, Matthias Bethge, Felix~A.
  Wichmann, and Wieland Brendel.
\newblock Imagenet-trained cnns are biased towards texture; increasing shape
  bias improves accuracy and robustness.
\newblock \emph{International Conference on Learning Representations (ICLR)},
  2019.

\bibitem[Geirhos et~al.(2020)Geirhos, Jacobsen, Michaelis, Zemel, Brendel,
  Bethge, and Wichmann]{geirhos2020shortcuts}
Robert Geirhos, Jörn-Henrik Jacobsen, Claudio Michaelis, Richard Zemel,
  Wieland Brendel, Matthias Bethge, and Felix~A. Wichmann.
\newblock Shortcut learning in deep neural networks.
\newblock \emph{Nature Machine Intelligence}, 2\penalty0 (11):\penalty0
  665–673, Nov 2020.

\bibitem[Gondal et~al.(2019)Gondal, Wuthrich, Miladinovic, Locatello, Breidt,
  Volchkov, Akpo, Bachem, Sch{\"o}lkopf, and Bauer]{gondal2019transfer}
Muhammad~Waleed Gondal, Manuel Wuthrich, Djordje Miladinovic, Francesco
  Locatello, Martin Breidt, Valentin Volchkov, Joel Akpo, Olivier Bachem,
  Bernhard Sch{\"o}lkopf, and Stefan Bauer.
\newblock On the transfer of inductive bias from simulation to the real world:
  a new disentanglement dataset.
\newblock \emph{Advances in Neural Information Processing Systems},
  32:\penalty0 15740--15751, 2019.

\bibitem[Gresele et~al.(2019)Gresele, Rubenstein, Mehrjou, Locatello, and
  Sch{\"{o}}lkopf]{gresele2019incomplete}
Luigi Gresele, Paul~K. Rubenstein, Arash Mehrjou, Francesco Locatello, and
  Bernhard Sch{\"{o}}lkopf.
\newblock The {I}ncomplete {R}osetta {S}tone problem: Identifiability results
  for multi-view nonlinear {ICA}.
\newblock In \emph{Proceedings of the Thirty-Fifth Conference on Uncertainty in
  Artificial Intelligence, {UAI}}, 2019.

\bibitem[Gresele et~al.(2021)Gresele, von K{\"u}gelgen, Stimper, Sch{\"o}lkopf,
  and Besserve]{gresele2021independent}
Luigi Gresele, Julius von K{\"u}gelgen, Vincent Stimper, Bernhard
  Sch{\"o}lkopf, and Michel Besserve.
\newblock Independent mechanism analysis, a new concept?
\newblock In \emph{Advances in Neural Information Processing Systems}, 2021.

\bibitem[Grill et~al.(2020)Grill, Strub, Altch{\'{e}}, Tallec, Richemond,
  Buchatskaya, Doersch, Pires, Guo, Azar, Piot, Kavukcuoglu, Munos, and
  Valko]{grill2020byol}
Jean{-}Bastien Grill, Florian Strub, Florent Altch{\'{e}}, Corentin Tallec,
  Pierre~H. Richemond, Elena Buchatskaya, Carl Doersch, Bernardo~{\'{A}}vila
  Pires, Zhaohan Guo, Mohammad~Gheshlaghi Azar, Bilal Piot, Koray Kavukcuoglu,
  R{\'{e}}mi Munos, and Michal Valko.
\newblock Bootstrap your own latent - {A} new approach to self-supervised
  learning.
\newblock In \emph{Advances in Neural Information Processing Systems 33}, 2020.

\bibitem[Gutmann and Hyv{\"a}rinen(2010)]{gutmann2010noise}
Michael Gutmann and Aapo Hyv{\"a}rinen.
\newblock Noise-contrastive estimation: A new estimation principle for
  unnormalized statistical models.
\newblock In \emph{Proceedings of the Thirteenth International Conference on
  Artificial Intelligence and Statistics}, pages 297--304, 2010.

\bibitem[Gutmann and Hyv\"{a}rinen(2012)]{Gutmann12JMLR}
Michael Gutmann and Aapo Hyv\"{a}rinen.
\newblock Noise-contrastive estimation of unnormalized statistical models, with
  applications to natural image statistics.
\newblock \emph{The Journal of Machine Learning Research}, 13:\penalty0
  307--361, 2012.

\bibitem[Hadsell et~al.(2006)Hadsell, Chopra, and
  LeCun]{hadsell2006dimensionality}
R.~Hadsell, S.~Chopra, and Y.~LeCun.
\newblock Dimensionality reduction by learning an invariant mapping.
\newblock In \emph{IEEE Conference on Computer Vision and Pattern Recognition},
  volume~2, pages 1735--1742, 2006.

\bibitem[H{\"a}lv{\"a} and Hyvarinen(2020)]{halva2020hidden}
Hermanni H{\"a}lv{\"a} and Aapo Hyvarinen.
\newblock Hidden markov nonlinear ica: Unsupervised learning from nonstationary
  time series.
\newblock In \emph{Conference on Uncertainty in Artificial Intelligence}, pages
  939--948. PMLR, 2020.

\bibitem[He et~al.(2016)He, Zhang, Ren, and Sun]{he2015deep}
Kaiming He, Xiangyu Zhang, Shaoqing Ren, and Jian Sun.
\newblock Deep residual learning for image recognition.
\newblock In \emph{IEEE Conference on Computer Vision and Pattern Recognition},
  2016.

\bibitem[He et~al.(2020)He, Fan, Wu, Xie, and Girshick]{he2019momentum}
Kaiming He, Haoqi Fan, Yuxin Wu, Saining Xie, and Ross~B. Girshick.
\newblock Momentum contrast for unsupervised visual representation learning.
\newblock In \emph{IEEE Conference on Computer Vision and Pattern Recognition},
  pages 9726--9735, 2020.

\bibitem[H{\'{e}}naff(2020)]{henaff2020data}
Olivier~J. H{\'{e}}naff.
\newblock Data-efficient image recognition with contrastive predictive coding.
\newblock In \emph{Proceedings of the 37th International Conference on Machine
  Learning}, volume 119 of \emph{Proceedings of Machine Learning Research},
  pages 4182--4192. {PMLR}, 2020.

\bibitem[Higgins et~al.(2017)Higgins, Matthey, Pal, Burgess, Glorot, Botvinick,
  Mohamed, and Lerchner]{higgins2017beta}
Irina Higgins, Lo{\"{\i}}c Matthey, Arka Pal, Christopher Burgess, Xavier
  Glorot, Matthew Botvinick, Shakir Mohamed, and Alexander Lerchner.
\newblock beta-vae: Learning basic visual concepts with a constrained
  variational framework.
\newblock In \emph{5th International Conference on Learning Representations},
  2017.

\bibitem[Hjelm et~al.(2019)Hjelm, Fedorov, Lavoie{-}Marchildon, Grewal,
  Bachman, Trischler, and Bengio]{hjelm2018learning}
R.~Devon Hjelm, Alex Fedorov, Samuel Lavoie{-}Marchildon, Karan Grewal, Philip
  Bachman, Adam Trischler, and Yoshua Bengio.
\newblock Learning deep representations by mutual information estimation and
  maximization.
\newblock In \emph{7th International Conference on Learning Representations},
  2019.

\bibitem[Hosoya(2019)]{hosoya2019group}
Haruo Hosoya.
\newblock Group-based learning of disentangled representations with
  generalizability for novel contents.
\newblock In \emph{IJCAI}, pages 2506--2513, 2019.

\bibitem[Howard(2013)]{howard2013improvements}
Andrew~G. Howard.
\newblock Some improvements on deep convolutional neural network based image
  classification, 2013.

\bibitem[Hyv{\"a}rinen and Hoyer(2000)]{hyvarinen2000emergence}
Aapo Hyv{\"a}rinen and Patrik Hoyer.
\newblock Emergence of phase-and shift-invariant features by decomposition of
  natural images into independent feature subspaces.
\newblock \emph{Neural computation}, 12\penalty0 (7):\penalty0 1705--1720,
  2000.

\bibitem[Hyvarinen and Morioka(2016)]{hyvarinen2016unsupervised}
Aapo Hyvarinen and Hiroshi Morioka.
\newblock Unsupervised feature extraction by time-contrastive learning and
  nonlinear {ICA}.
\newblock In \emph{Advances in Neural Information Processing Systems}, pages
  3765--3773, 2016.

\bibitem[Hyvarinen and Morioka(2017)]{hyvarinen2017nonlinear}
Aapo Hyvarinen and Hiroshi Morioka.
\newblock Nonlinear ica of temporally dependent stationary sources.
\newblock In \emph{Artificial Intelligence and Statistics}, pages 460--469.
  PMLR, 2017.

\bibitem[Hyv{\"a}rinen and Oja(2000)]{hyvarinen2000independent}
Aapo Hyv{\"a}rinen and Erkki Oja.
\newblock Independent component analysis: algorithms and applications.
\newblock \emph{Neural networks}, 13\penalty0 (4-5):\penalty0 411--430, 2000.

\bibitem[Hyv{\"a}rinen and Pajunen(1999)]{hyvarinen1999nonlinear}
Aapo Hyv{\"a}rinen and Petteri Pajunen.
\newblock Nonlinear independent component analysis: Existence and uniqueness
  results.
\newblock \emph{Neural Networks}, 12\penalty0 (3):\penalty0 429--439, 1999.

\bibitem[Hyvarinen et~al.(2019)Hyvarinen, Sasaki, and
  Turner]{hyvarinen2019nonlinear}
Aapo Hyvarinen, Hiroaki Sasaki, and Richard Turner.
\newblock Nonlinear ica using auxiliary variables and generalized contrastive
  learning.
\newblock In \emph{The 22nd International Conference on Artificial Intelligence
  and Statistics}, pages 859--868. PMLR, 2019.

\bibitem[Ilse et~al.(2021)Ilse, Tomczak, and Forr{\'e}]{ilse2021selecting}
Maximilian Ilse, Jakub~M Tomczak, and Patrick Forr{\'e}.
\newblock Selecting data augmentation for simulating interventions.
\newblock In \emph{International Conference on Machine Learning}, pages
  4555--4562. PMLR, 2021.

\bibitem[Jacobsen et~al.(2018)Jacobsen, Smeulders, and
  Oyallon]{jacobsen2018revnet}
J{\"o}rn-Henrik Jacobsen, Arnold Smeulders, and Edouard Oyallon.
\newblock i-{R}ev{N}et: Deep invertible networks.
\newblock In \emph{International Conference on Learning Representations}, 2018.

\bibitem[Jaynes(1982)]{jaynes1982rationale}
Edwin~T Jaynes.
\newblock On the rationale of maximum-entropy methods.
\newblock \emph{Proceedings of the IEEE}, 70\penalty0 (9):\penalty0 939--952,
  1982.

\bibitem[Keenan's 3D Model Repository()]{KeenanScanRep}
Keenan's 3D Model Repository.
\newblock {Keenan's 3D Model Repository }.
\newblock \text{https://www.cs.cmu.edu/~kmcrane/Projects/ModelRepository/},
  2021.

\bibitem[Khemakhem et~al.(2020{\natexlab{a}})Khemakhem, Kingma, Monti, and
  Hyv{\"{a}}rinen]{khemakhem2020variational}
Ilyes Khemakhem, Diederik~P. Kingma, Ricardo~Pio Monti, and Aapo
  Hyv{\"{a}}rinen.
\newblock Variational autoencoders and nonlinear {ICA:} {A} unifying framework.
\newblock In \emph{The 23rd International Conference on Artificial Intelligence
  and Statistics, {AISTATS}}, volume 108, pages 2207--2217, 2020{\natexlab{a}}.

\bibitem[Khemakhem et~al.(2020{\natexlab{b}})Khemakhem, Monti, Kingma, and
  Hyv{\"{a}}rinen]{khemakhem2020ice}
Ilyes Khemakhem, Ricardo~Pio Monti, Diederik~P. Kingma, and Aapo
  Hyv{\"{a}}rinen.
\newblock Ice-beem: Identifiable conditional energy-based deep models based on
  nonlinear {ICA}.
\newblock In \emph{Advances in Neural Information Processing Systems 33},
  2020{\natexlab{b}}.

\bibitem[Kim and Mnih(2018)]{kim2018disentangling}
Hyunjik Kim and Andriy Mnih.
\newblock Disentangling by factorising.
\newblock In \emph{International Conference on Machine Learning}, pages
  2649--2658. PMLR, 2018.

\bibitem[Kingma and Ba(2015)]{kingma2014adam}
Diederik~P. Kingma and Jimmy Ba.
\newblock Adam: {A} method for stochastic optimization.
\newblock In \emph{3rd International Conference on Learning Representations},
  2015.

\bibitem[Kingma and Dhariwal(2018)]{kingma2018glow}
Diederik~P Kingma and Prafulla Dhariwal.
\newblock Glow: generative flow with invertible 1$\times$ 1 convolutions.
\newblock In \emph{Advances in Neural Information Processing Systems}, pages
  10236--10245, 2018.

\bibitem[Kingma and Welling(2014)]{kingma2013auto}
Diederik~P. Kingma and Max Welling.
\newblock Auto-encoding variational bayes.
\newblock In \emph{2nd International Conference on Learning Representations
  ({ICLR})}, 2014.

\bibitem[Klindt et~al.(2021)Klindt, Schott, Sharma, Ustyuzhaninov, Brendel,
  Bethge, and Paiton]{klindt2020slowvae}
David Klindt, Lukas Schott, Yash Sharma, Ivan Ustyuzhaninov, Wieland Brendel,
  Matthias Bethge, and Dylan Paiton.
\newblock Towards nonlinear disentanglement in natural data with temporal
  sparse coding.
\newblock \emph{International Conference on Learning Representations (ICLR)},
  2021.

\bibitem[Krishnamurthy and Levoy(1996)]{Krishnamurthy1996armadillo}
Venkat Krishnamurthy and Marc Levoy.
\newblock Fitting smooth surfaces to dense polygon meshes.
\newblock In John Fujii, editor, \emph{Proceedings of the 23rd Annual
  Conference on Computer Graphics and Interactive Techniques, {SIGGRAPH} 1996,
  New Orleans, LA, USA, August 4-9, 1996}, pages 313--324. {ACM}, 1996.

\bibitem[Kumar et~al.(2018)Kumar, Sattigeri, and
  Balakrishnan]{kumar2018variational}
Abhishek Kumar, Prasanna Sattigeri, and Avinash Balakrishnan.
\newblock Variational inference of disentangled latent concepts from unlabeled
  observations.
\newblock In \emph{International Conference on Learning Representations}, 2018.

\bibitem[Lake et~al.(2017)Lake, Ullman, Tenenbaum, and
  Gershman]{lake2017building}
Brenden~M Lake, Tomer~D Ullman, Joshua~B Tenenbaum, and Samuel~J Gershman.
\newblock Building machines that learn and think like people.
\newblock \emph{Behavioral and brain sciences}, 40, 2017.

\bibitem[Le et~al.(2011)Le, Zou, Yeung, and Ng]{le2011learning}
Quoc~V Le, Will~Y Zou, Serena~Y Yeung, and Andrew~Y Ng.
\newblock Learning hierarchical invariant spatio-temporal features for action
  recognition with independent subspace analysis.
\newblock In \emph{CVPR 2011}, pages 3361--3368. IEEE, 2011.

\bibitem[Lee et~al.(2020)Lee, Lei, Saunshi, and Zhuo]{lee2020predicting}
Jason~D Lee, Qi~Lei, Nikunj Saunshi, and Jiacheng Zhuo.
\newblock Predicting what you already know helps: Provable self-supervised
  learning.
\newblock \emph{arXiv preprint arXiv:2008.01064}, 2020.

\bibitem[Lee et~al.(1999)Lee, Girolami, and Sejnowski]{lee1999independent}
Te-Won Lee, Mark Girolami, and Terrence~J Sejnowski.
\newblock Independent component analysis using an extended infomax algorithm
  for mixed subgaussian and supergaussian sources.
\newblock \emph{Neural computation}, 11\penalty0 (2):\penalty0 417--441, 1999.

\bibitem[Leeb et~al.(2020)Leeb, Annadani, Bauer, and
  Sch{\"o}lkopf]{leeb2020structural}
Felix Leeb, Yashas Annadani, Stefan Bauer, and Bernhard Sch{\"o}lkopf.
\newblock Structural autoencoders improve representations for generation and
  transfer.
\newblock \emph{arXiv preprint arXiv:2006.07796}, 2020.

\bibitem[Lehmann and Casella(2006)]{lehmann2006theory}
Erich~L Lehmann and George Casella.
\newblock \emph{Theory of point estimation}.
\newblock Springer Science \& Business Media, 2006.

\bibitem[Lepus Othus()]{Lepusothus}
Lepus Othus.
\newblock {Animal Diversity Web }.
\newblock \text{https://animaldiversity.org/accounts/Lepus\_othus/}, 2021.

\bibitem[Linsker(1988)]{linsker1988self}
Ralph Linsker.
\newblock Self-organization in a perceptual network.
\newblock \emph{Computer}, 21\penalty0 (3):\penalty0 105--117, 1988.

\bibitem[Linsker(1989)]{linsker1989application}
Ralph Linsker.
\newblock An application of the principle of maximum information preservation
  to linear systems.
\newblock In \emph{Advances in Neural Information Processing Systems}, pages
  186--194, 1989.

\bibitem[Liu et~al.(2019)Liu, Ott, Goyal, Du, Joshi, Chen, Levy, Lewis,
  Zettlemoyer, and Stoyanov]{liu2019roberta}
Yinhan Liu, Myle Ott, Naman Goyal, Jingfei Du, Mandar Joshi, Danqi Chen, Omer
  Levy, Mike Lewis, Luke Zettlemoyer, and Veselin Stoyanov.
\newblock Roberta: A robustly optimized bert pretraining approach.
\newblock \emph{arXiv preprint arXiv:1907.11692}, 2019.

\bibitem[Locatello et~al.(2019)Locatello, Bauer, Lucic, Raetsch, Gelly,
  Sch{\"o}lkopf, and Bachem]{locatello2019challenging}
Francesco Locatello, Stefan Bauer, Mario Lucic, Gunnar Raetsch, Sylvain Gelly,
  Bernhard Sch{\"o}lkopf, and Olivier Bachem.
\newblock Challenging common assumptions in the unsupervised learning of
  disentangled representations.
\newblock In \emph{International Conference on Machine Learning}, pages
  4114--4124, 2019.

\bibitem[Locatello et~al.(2020)Locatello, Poole, R{\"{a}}tsch, Sch{\"{o}}lkopf,
  Bachem, and Tschannen]{locatello2020weakly}
Francesco Locatello, Ben Poole, Gunnar R{\"{a}}tsch, Bernhard Sch{\"{o}}lkopf,
  Olivier Bachem, and Michael Tschannen.
\newblock Weakly-supervised disentanglement without compromises.
\newblock In \emph{Proceedings of the 37th International Conference on Machine
  Learning}, volume 119, pages 6348--6359. {PMLR}, 2020.

\bibitem[Logeswaran and Lee(2018)]{logeswaran2018efficient}
Lajanugen Logeswaran and Honglak Lee.
\newblock An efficient framework for learning sentence representations.
\newblock In \emph{6th International Conference on Learning Representations},
  2018.

\bibitem[Lu et~al.(2021)Lu, Wu, Hern{\'a}ndez-Lobato, and
  Sch{\"o}lkopf]{lu2021nonlinear}
Chaochao Lu, Yuhuai Wu, Jo{\'s}e~Miguel Hern{\'a}ndez-Lobato, and Bernhard
  Sch{\"o}lkopf.
\newblock Nonlinear invariant risk minimization: A causal approach.
\newblock \emph{arXiv preprint arXiv:2102.12353}, 2021.

\bibitem[Mikolov et~al.(2013)Mikolov, Sutskever, Chen, Corrado, and
  Dean]{mikolov2013distributed}
Tomas Mikolov, Ilya Sutskever, Kai Chen, Greg Corrado, and Jeffrey Dean.
\newblock Distributed representations of words and phrases and their
  compositionality.
\newblock In \emph{Advances in Neural Information Processing Systems}, pages
  3111--3119, 2013.

\bibitem[Mitrovic et~al.(2021)Mitrovic, McWilliams, Walker, Buesing, and
  Blundell]{mitrovic2020representation}
Jovana Mitrovic, Brian McWilliams, Jacob~C. Walker, Lars~Holger Buesing, and
  Charles Blundell.
\newblock Representation learning via invariant causal mechanisms.
\newblock In \emph{9th International Conference on Learning Representations},
  2021.

\bibitem[Murphy(2012)]{murphybook}
Kevin~P Murphy.
\newblock \emph{Machine learning: a probabilistic perspective}.
\newblock MIT Press, Cambridge, MA, 2012.

\bibitem[Newell(1975)]{newell1975utah}
Martin~Edward Newell.
\newblock \emph{The Utilization of Procedure Models in Digital Image
  Synthesis.}
\newblock PhD thesis, The University of Utah, 1975.
\newblock AAI7529894.

\bibitem[Noroozi and Favaro(2016)]{noroozi2016unsupervised}
Mehdi Noroozi and Paolo Favaro.
\newblock Unsupervised learning of visual representations by solving jigsaw
  puzzles.
\newblock In \emph{European conference on computer vision}, pages 69--84.
  Springer, 2016.

\bibitem[Oord et~al.(2018)Oord, Li, and Vinyals]{oord2018representation}
Aaron van~den Oord, Yazhe Li, and Oriol Vinyals.
\newblock Representation learning with contrastive predictive coding.
\newblock \emph{arXiv preprint arXiv:1807.03748}, 2018.

\bibitem[Papamakarios et~al.(2017)Papamakarios, Pavlakou, and
  Murray]{papamakarios2017masked}
George Papamakarios, Theo Pavlakou, and Iain Murray.
\newblock Masked autoregressive flow for density estimation.
\newblock In \emph{Advances in Neural Information Processing Systems}, pages
  2338--2347, 2017.

\bibitem[Papamakarios et~al.(2021)Papamakarios, Nalisnick, Rezende, Mohamed,
  and Lakshminarayanan]{papamakarios2021normalizing}
George Papamakarios, Eric Nalisnick, Danilo~Jimenez Rezende, Shakir Mohamed,
  and Balaji Lakshminarayanan.
\newblock Normalizing flows for probabilistic modeling and inference.
\newblock \emph{Journal of Machine Learning Research}, 22\penalty0
  (57):\penalty0 1--64, 2021.

\bibitem[Pearl(2009)]{pearl2009causality}
Judea Pearl.
\newblock \emph{Causality}.
\newblock Cambridge university press, 2009.

\bibitem[Pennington et~al.(2014)Pennington, Socher, and
  Manning]{pennington2014glove}
Jeffrey Pennington, Richard Socher, and Christopher~D Manning.
\newblock Glove: Global vectors for word representation.
\newblock In \emph{Proceedings of the 2014 conference on empirical methods in
  natural language processing (EMNLP)}, pages 1532--1543, 2014.

\bibitem[Pezeshki et~al.(2020)Pezeshki, Kaba, Bengio, Courville, Precup, and
  Lajoie]{pezeshki2020gradient}
Mohammad Pezeshki, S{\'e}kou-Oumar Kaba, Yoshua Bengio, Aaron Courville, Doina
  Precup, and Guillaume Lajoie.
\newblock Gradient starvation: A learning proclivity in neural networks.
\newblock \emph{arXiv preprint arXiv:2011.09468}, 2020.

\bibitem[Poole et~al.(2019)Poole, Ozair, Van Den~Oord, Alemi, and
  Tucker]{poole2019variational}
Ben Poole, Sherjil Ozair, Aaron Van Den~Oord, Alex Alemi, and George Tucker.
\newblock On variational bounds of mutual information.
\newblock In \emph{International Conference on Machine Learning}, pages
  5171--5180. PMLR, 2019.

\bibitem[Praun et~al.(2000)Praun, Finkelstein, and Hoppe]{Praun2000horse}
Emil Praun, Adam Finkelstein, and Hugues Hoppe.
\newblock Lapped textures.
\newblock In Judith~R. Brown and Kurt Akeley, editors, \emph{Proceedings of the
  27th Annual Conference on Computer Graphics and Interactive Techniques,
  {SIGGRAPH} 2000, New Orleans, LA, USA, July 23-28, 2000}, pages 465--470.
  {ACM}, 2000.

\bibitem[Radford et~al.(2018)Radford, Narasimhan, Salimans, and
  Sutskever]{radford2018improving}
Alec Radford, Karthik Narasimhan, Tim Salimans, and Ilya Sutskever.
\newblock Improving language understanding by generative pre-training.
\newblock \emph{Technical report, OpenAI}, 2018.

\bibitem[Ravanelli et~al.(2020)Ravanelli, Zhong, Pascual, Swietojanski,
  Monteiro, Trmal, and Bengio]{ravanelli2020multi}
Mirco Ravanelli, Jianyuan Zhong, Santiago Pascual, Pawel Swietojanski, Joao
  Monteiro, Jan Trmal, and Yoshua Bengio.
\newblock Multi-task self-supervised learning for robust speech recognition.
\newblock In \emph{IEEE International Conference on Acoustics, Speech and
  Signal Processing (ICASSP)}, pages 6989--6993, 2020.

\bibitem[Richard et~al.(2020)Richard, Gresele, Hyvarinen, Thirion, Gramfort,
  and Ablin]{richard2020modeling}
H.~Richard, L.~Gresele, A.~Hyvarinen, B.~Thirion, A.~Gramfort, and P.~Ablin.
\newblock Modeling shared responses in neuroimaging studies through multiview
  ica.
\newblock In \emph{Advances in Neural Information Processing Systems 33}, pages
  19149--19162. Curran Associates, Inc., December 2020.

\bibitem[Roeder et~al.(2021)Roeder, Metz, and Kingma]{roeder2020linear}
Geoffrey Roeder, Luke Metz, and Durk Kingma.
\newblock On linear identifiability of learned representations.
\newblock In \emph{International Conference on Machine Learning}, pages
  9030--9039. PMLR, 2021.

\bibitem[Rojas-Carulla et~al.(2018)Rojas-Carulla, Sch{\"o}lkopf, Turner, and
  Peters]{rojas2018invariant}
Mateo Rojas-Carulla, Bernhard Sch{\"o}lkopf, Richard Turner, and Jonas Peters.
\newblock Invariant models for causal transfer learning.
\newblock \emph{The Journal of Machine Learning Research}, 19\penalty0
  (1):\penalty0 1309--1342, 2018.

\bibitem[Schneider et~al.(2019)Schneider, Baevski, Collobert, and
  Auli]{schneider2019wav2vec}
Steffen Schneider, Alexei Baevski, Ronan Collobert, and Michael Auli.
\newblock wav2vec: Unsupervised pre-training for speech recognition.
\newblock In \emph{Interspeech 2019, 20th Annual Conference of the
  International Speech Communication Association}, pages 3465--3469, 2019.

\bibitem[Sch{\"o}lkopf(2019)]{scholkopf2019causality}
Bernhard Sch{\"o}lkopf.
\newblock Causality for machine learning.
\newblock \emph{arXiv preprint arXiv:1911.10500}, 2019.

\bibitem[Sch{\"o}lkopf et~al.(2021)Sch{\"o}lkopf, Locatello, Bauer, Ke,
  Kalchbrenner, Goyal, and Bengio]{scholkopf2021toward}
Bernhard Sch{\"o}lkopf, Francesco Locatello, Stefan Bauer, Nan~Rosemary Ke, Nal
  Kalchbrenner, Anirudh Goyal, and Yoshua Bengio.
\newblock Toward causal representation learning.
\newblock \emph{Proceedings of the IEEE}, 2021.

\bibitem[Shen et~al.(2020)Shen, Liu, Dong, Lian, Chen, and
  Zhang]{shen2020disentangled}
Xinwei Shen, Furui Liu, Hanze Dong, Qing Lian, Zhitang Chen, and Tong Zhang.
\newblock Disentangled generative causal representation learning.
\newblock \emph{arXiv preprint arXiv:2010.02637}, 2020.

\bibitem[Shu et~al.(2020)Shu, Chen, Kumar, Ermon, and Poole]{shu2019weakly}
Rui Shu, Yining Chen, Abhishek Kumar, Stefano Ermon, and Ben Poole.
\newblock Weakly supervised disentanglement with guarantees.
\newblock In \emph{8th International Conference on Learning Representations},
  2020.

\bibitem[Sorrenson et~al.(2020)Sorrenson, Rother, and
  Köthe]{Sorrenson2020Disentanglement}
Peter Sorrenson, Carsten Rother, and Ullrich Köthe.
\newblock Disentanglement by nonlinear ica with general incompressible-flow
  networks (gin).
\newblock In \emph{International Conference on Learning Representations}, 2020.

\bibitem[Stanford Scanning Repository()]{StanfordScanRep}
Stanford Scanning Repository.
\newblock {The Stanford 3D Scanning Repository}.
\newblock \text{http://graphics.stanford.edu/data/3Dscanrep/}, 2021.

\bibitem[Suggestive Contour Gallery()]{SuggestContourGallery}
Suggestive Contour Gallery.
\newblock {Suggestive Contour Gallery }.
\newblock \text{https://gfx.cs.princeton.edu/proj/sugcon/models/}, 2021.

\bibitem[Suter et~al.(2019)Suter, Miladinovic, Sch{\"o}lkopf, and
  Bauer]{suter2019robustly}
Raphael Suter, Djordje Miladinovic, Bernhard Sch{\"o}lkopf, and Stefan Bauer.
\newblock Robustly disentangled causal mechanisms: Validating deep
  representations for interventional robustness.
\newblock In \emph{International Conference on Machine Learning}, pages
  6056--6065. PMLR, 2019.

\bibitem[Szegedy et~al.(2015)Szegedy, Liu, Jia, Sermanet, Reed, Anguelov,
  Erhan, Vanhoucke, and Rabinovich]{szegedy2015going}
Christian Szegedy, Wei Liu, Yangqing Jia, Pierre Sermanet, Scott Reed, Dragomir
  Anguelov, Dumitru Erhan, Vincent Vanhoucke, and Andrew Rabinovich.
\newblock Going deeper with convolutions.
\newblock In \emph{Computer Vision and Pattern Recognition (CVPR)}, 2015.

\bibitem[Theis(2006)]{theis2006towards}
Fabian Theis.
\newblock Towards a general independent subspace analysis.
\newblock \emph{Advances in Neural Information Processing Systems},
  19:\penalty0 1361--1368, 2006.

\bibitem[Tian et~al.(2020)Tian, Krishnan, and Isola]{tian2019contrastive}
Yonglong Tian, Dilip Krishnan, and Phillip Isola.
\newblock Contrastive multiview coding.
\newblock In \emph{Computer Vision - {ECCV} 2020}, volume 12356, pages
  776--794. Springer, 2020.

\bibitem[Tian et~al.(2021)Tian, Chen, and Ganguli]{tian2021understanding}
Yuandong Tian, Xinlei Chen, and Surya Ganguli.
\newblock Understanding self-supervised learning dynamics without contrastive
  pairs.
\newblock In Marina Meila and Tong Zhang, editors, \emph{Proceedings of the
  38th International Conference on Machine Learning,}, volume 139, pages
  10268--10278, 2021.

\bibitem[Tosh et~al.(2020)Tosh, Krishnamurthy, and Hsu]{tosh2020contrastive}
Christopher Tosh, Akshay Krishnamurthy, and Daniel Hsu.
\newblock Contrastive estimation reveals topic posterior information to linear
  models.
\newblock \emph{arXiv preprint arXiv:2003.02234}, 2020.

\bibitem[Tosh et~al.(2021)Tosh, Krishnamurthy, and Hsu]{tosh2021contrastive}
Christopher Tosh, Akshay Krishnamurthy, and Daniel Hsu.
\newblock Contrastive learning, multi-view redundancy, and linear models.
\newblock In \emph{Algorithmic Learning Theory}, pages 1179--1206. PMLR, 2021.

\bibitem[Tsai et~al.(2020)Tsai, Wu, Salakhutdinov, and Morency]{tsai2020self}
Yao-Hung~Hubert Tsai, Yue Wu, Ruslan Salakhutdinov, and Louis-Philippe Morency.
\newblock Self-supervised learning from a multi-view perspective.
\newblock In \emph{International Conference on Learning Representations}, 2020.

\bibitem[Tschannen et~al.(2020)Tschannen, Djolonga, Rubenstein, Gelly, and
  Lucic]{tschannen2019mutual}
Michael Tschannen, Josip Djolonga, Paul~K. Rubenstein, Sylvain Gelly, and Mario
  Lucic.
\newblock On mutual information maximization for representation learning.
\newblock In \emph{8th International Conference on Learning Representations,
  {ICLR} 2020, Addis Ababa, Ethiopia, April 26-30, 2020}. OpenReview.net, 2020.

\bibitem[Turk and Levoy(1994)]{turk1994bunny}
Greg Turk and Marc Levoy.
\newblock Zippered polygon meshes from range images.
\newblock In Dino Schweitzer, Andrew~S. Glassner, and Mike Keeler, editors,
  \emph{Proceedings of the 21th Annual Conference on Computer Graphics and
  Interactive Techniques, {SIGGRAPH} 1994, Orlando, FL, USA, July 24-29, 1994},
  pages 311--318. {ACM}, 1994.

\bibitem[Vincent et~al.(2008)Vincent, Larochelle, Bengio, and
  Manzagol]{vincent2008extracting}
Pascal Vincent, Hugo Larochelle, Yoshua Bengio, and Pierre-Antoine Manzagol.
\newblock Extracting and composing robust features with denoising autoencoders.
\newblock In \emph{Proceedings of the 25th International Conference on Machine
  Learning}, pages 1096--1103, 2008.

\bibitem[von K{\"u}gelgen et~al.(2020)von K{\"u}gelgen, Ustyuzhaninov, Gehler,
  Bethge, and Sch{\"o}lkopf]{von2020towards}
Julius von K{\"u}gelgen, Ivan Ustyuzhaninov, Peter Gehler, Matthias Bethge, and
  Bernhard Sch{\"o}lkopf.
\newblock Towards causal generative scene models via competition of experts.
\newblock In \emph{ICLR Workshop on ``Causal Learning for Decision Making''},
  2020.

\bibitem[Wang and Isola(2020)]{wang2020understanding}
Tongzhou Wang and Phillip Isola.
\newblock Understanding contrastive representation learning through alignment
  and uniformity on the hypersphere.
\newblock In \emph{Proceedings of the 37th International Conference on Machine
  Learning, {ICML} 2020, 13-18 July 2020, Virtual Event}, volume 119 of
  \emph{Proceedings of Machine Learning Research}, pages 9929--9939. {PMLR},
  2020.

\bibitem[Wang and Gupta(2015)]{wang2015unsupervised}
Xiaolong Wang and Abhinav Gupta.
\newblock Unsupervised learning of visual representations using videos.
\newblock In \emph{Proceedings of the IEEE International Conference on Computer
  Vision}, pages 2794--2802, 2015.

\bibitem[Wu et~al.(2018)Wu, Xiong, Yu, and Lin]{wu2018unsupervised}
Zhirong Wu, Yuanjun Xiong, Stella~X. Yu, and Dahua Lin.
\newblock Unsupervised feature learning via non-parametric instance
  discrimination.
\newblock In \emph{Conference on Computer Vision and Pattern Recognition,
  {CVPR}}, pages 3733--3742. {IEEE} Computer Society, 2018.

\bibitem[Yang et~al.(2020)Yang, Liu, Chen, Shen, Hao, and
  Wang]{yang2020causalvae}
Mengyue Yang, Furui Liu, Zhitang Chen, Xinwei Shen, Jianye Hao, and Jun Wang.
\newblock Causalvae: Structured causal disentanglement in variational
  autoencoder.
\newblock \emph{arXiv preprint arXiv:2004.08697}, 2020.

\bibitem[Zbontar et~al.(2021)Zbontar, Jing, Misra, LeCun, and
  Deny]{zbontar2021barlow}
Jure Zbontar, Li~Jing, Ishan Misra, Yann LeCun, and St{\'{e}}phane Deny.
\newblock Barlow twins: Self-supervised learning via redundancy reduction.
\newblock In Marina Meila and Tong Zhang, editors, \emph{Proceedings of the
  38th International Conference on Machine Learning}, volume 139, pages
  12310--12320, 2021.

\bibitem[Zimmermann et~al.(2021)Zimmermann, Sharma, Schneider, Bethge, and
  Brendel]{zimmermann2021contrastive}
Roland~S. Zimmermann, Yash Sharma, Steffen Schneider, Matthias Bethge, and
  Wieland Brendel.
\newblock Contrastive learning inverts the data generating process.
\newblock In \emph{Proceedings of the 38th International Conference on Machine
  Learning}, volume 139, pages 12979--12990, 2021.

\end{thebibliography}
}

\clearpage
\appendix

\begin{center}
{\centering \LARGE APPENDIX}
\vspace{0.8cm}
\end{center}

\section*{Overview:}

\begin{itemize}
    \item \Cref{app:proofs} contains the full proofs for all theoretical results from the main paper.
    \item \Cref{app:data_set_details} contains additional details and plots for the \textit{Causal3DIdent} dataset.
    \item \Cref{app:additional_results} contains additional experimental results and analysis.
    \item \Cref{app:experimental_details} contains additional implementation details for our experiments.
\end{itemize}

\section{Proofs}
\label{app:proofs}
We now present the full detailed proofs of our three theorems
which were briefly sketched in the main paper. We remark that these proofs build on each other, in the sense that the (main) step 2 of the proof of~\cref{thm:main} is also used  in the proofs of~\cref{thm:CL,thm:CL_MaxEnt}.

\subsection{Proof of\texorpdfstring{~\Cref{thm:main}}{thmmain}}
\label{app:proof_generative}

\generative*

\begin{proof}
The proof consists of two main steps.

In the first step, we use assumption \textit{(i)} and the matching likelihoods to show that the representation $\zbh=\gb(\xb)$ extracted by $\gb=\hat{\fb}^{-1}$ is related to the true latent $\zb$ by a smooth invertible mapping $\hb$, and that $\zbh$ must satisfy invariance across $(\xb,\xbt)$ in the first $n_c$ (content) components almost surely (a.s.) with respect to (w.r.t.) the true generative process.

In the second step, we then use assumptions \textit{(ii)} and \textit{(iii)}
to  prove (by contradiction) that $\cbh:=\zbh_{1:n_c}=\hb(\zb)_{1:n_c}$ can, in fact, only depend on the true content $\cb$ and not on the true style $\sb$, for otherwise the invariance established in the first step would have be violated with probability greater than zero.

To provide some further intuition for the second step, the assumed generative process implies that $(\cb,\sb,\sbt)|A$ is constrained to take values (a.s.) in a subspace $\Rcal$ of $\Ccal\times\Scal\times\Scal$ of dimension $n_c+n_s+|A|$ (as opposed to dimension $n_c+2n_s$ for $\Ccal\times\Scal\times\Scal$). In this context, assumption \textit{(iii)} implies that  $(\cb,\sb,\sbt)|A$ has a density with respect to a measure on this subspace equivalent to the Lebesgue measure on $\RR^{n_c+n_s+|A|}$. This equivalence implies, in particular, that this ''subspace measure'' is strictly positive: it takes strictly positive values on open sets of $\Rcal$ seen as a topological subspace of $\Ccal\times\Scal\times\Scal$. These open sets are defined by the induced topology: they are the intersection of the open sets of $\Ccal\times\Scal\times\Scal$ with $\Rcal$. An open set $B$ of $V$ on which $p(\cb,\sb,\sbt|A) >0$ then satisfies $P(B|A)>0$. We look for such an open set to prove our result.

\paragraph{Step 1.}

From the assumed data generating process described in~\cref{sec:problem_formulation}---in particular, from the form of the model conditional $\hat{p}_{\zbt|\zb}$ described in~\cref{ass:content_invariance,ass:style_changes}---it follows that
\begin{align}
    \label{eq:invariance_constraint}
    \gb(\xb)_{1:n_c} &= \gb(\xbt)_{1:n_c}
\end{align}
a.s., i.e., with probability one, w.r.t.\ the model distribution $\hat{p}_{\xb,\xbt}$.

Due to the assumption of matching likelihoods, the invariance in~\eqref{eq:invariance_constraint} must also hold (a.s.) w.r.t.\ the true data distribution $p_{\xb,\xbt}$.

Next, since $\fb,\hat{\fb}:\Zcal\rightarrow\Xcal$ are smooth and invertible functions by assumption \textit{(i)}, there exists
a  smooth and invertible function $\hb=\gb\circ \fb:\Zcal\rightarrow\Zcal$ such that
\begin{equation}
\label{eq:relation_between_f_and_g}
\gb=\hb\circ\fbinv.
\end{equation}

Substituting~\eqref{eq:relation_between_f_and_g} into~\eqref{eq:invariance_constraint}, we obtain (a.s.\ w.r.t.\ $p$):
\begin{equation}
\label{eq:cbh}
\cbh
:=\zbh_{1:n_c}=\gb(\xb)_{1:n_c}
=\hb(\fbinv(\xb))_{1:n_c}
=\hb(\fbinv(\xbt))_{1:n_c}
\end{equation}

Substituting $\zb=\fbinv(\xb)$ and $\zbt=\fbinv(\xbt)$ into~\eqref{eq:cbh}, we obtain (a.s.\ w.r.t.\ $p$)
\begin{equation}
\label{eq:contradicted_expression}
    \cbh
    =\hb(\zb)_{1:n_c}
    =\hb(\zbt)_{1:n_c}.
\end{equation}

It remains to show that $\hb(\cdot)_{1:n_c}$ can only be a function of $\cb$, i.e., does not depend on any other (style) dimension of $\zb=(\cb,\sb)$.

\paragraph{Step 2.}
Suppose \textit{for a contradiction} that $\hb_c(\cb,\sb):=\hb(\cb,\sb)_{1:n_c}=\hb(\zb)_{1:n_c}$ depends on some component of the style variable $\sb$:
\begin{equation}
    \label{eq:contradiction_assumption}
    \exists l\in\{1, ..., n_s\}, (\cb^*,\sb^*)\in \Ccal\times \Scal,
    \quad \quad
    \text{s.t.}
    \quad \quad
    \frac{\partial \hb_c}{\partial s_{l}}(\cb^*,\sb^*)\neq 0,
\end{equation}
that is, we assume that the partial derivative of $\hb_c$ w.r.t.\ some style variable $s_l$ is non-zero at some point $\zb^*=(\cb^*,\sb^*)\in\Zcal=\Ccal\times\Scal$.

Since $\hb$ is smooth, so is $\hb_c$. Therefore, $\hb_c$ has continuous (first) partial derivatives.

By continuity of the partial derivative, $\frac{\partial \hb_c}{\partial s_{l}}$ must be non-zero in a neighbourhood of $(\cb^*,\sb^*)$, i.e.,
\begin{equation}
\label{eq:monotonicity}
\exists \eta>0 
\quad \quad \text{s.t.} \quad \quad 
s_{l} \mapsto \hb_c\big(\cb^*,(\sb^*_{-l},s_l)\big)
\quad  \text{is strictly monotonic on} \quad
(s^*_l-\eta,s^*_l+\eta),
\end{equation}
where $\sb_{-l}\in\Scal_{-l}$ denotes the vector of remaining style variables except $s_l$. 

Next, define the auxiliary function $\psi:\Ccal\times \Scal\times \Scal\rightarrow \RR_{\geq 0}$ as follows:
\begin{equation}
\label{eq:psi_def}
    \psi(\cb,\sb,\sbt):=|\hb_c(\cb,\sb)-\hb_c(\cb,\sbt)|\geq0\,.
\end{equation}

To obtain a contradiction to the invariance condition~\eqref{eq:contradicted_expression} from Step 1 under  assumption~\eqref{eq:contradiction_assumption},
it remains to show that $\psi$ from~\eqref{eq:psi_def} is \textit{strictly positive} with probability greater than zero (w.r.t.\ $p$).

First, the strict monotonicity from~\eqref{eq:monotonicity}
implies that 
\begin{equation}
\label{eq:positivity}
\psi\big(\cb^*,(\sb_{-l}^*,s_l),(\sb_{-l}^*,\tilde{s}_l)\big)
>0\,,\quad \forall (s_l,\tilde{s}_l)\in (s^*_l,s^*_l+\eta)\times(s^*_l-\eta,s^*_l) \,.
\end{equation}
Note that in order to obtain the strict inequality in~\eqref{eq:positivity}, it is important that $s_l$ and $\tilde{s}_l$ take values in \textit{disjoint} open subsets of the interval $(s^*_l-\eta,s^*_l+\eta)$ from~\eqref{eq:monotonicity}.

Since $\psi$ is a composition of continuous functions (absolute value of the difference of two continuous functions), $\psi$ is continuous.

Consider the open set $\RR_{>0}$, and recall that, under a continuous function, pre-images (or inverse images) of open sets are always \textit{open}.

Applied to the continuous function $\psi$, this pre-image corresponds to an \textit{open} set 
\begin{equation}
    \Ucal\subseteq \Ccal\times \Scal\times \Scal
\end{equation}
in the domain of $\psi$ on which $\psi$ is strictly positive. 

Moreover, due to~\eqref{eq:positivity}:
\begin{equation}
\label{eq:U_nonempty}
    \{\cb^*\}\times\left(\{\sb^*_{-l}\}\times(s^*_l,s^*_l+\eta)\right) \times \left(\{\sb^*_{-l}\}\times(s^*_l-\eta,s^*_l)\right)\subset \Ucal,
\end{equation}
so $\Ucal$ is \textit{non-empty}.

Next, by assumption \textit{(iii)}, there exists at least one subset $A\subseteq\{1, ..., n_s\}$ of changing style variables such that $l\in A$ and $p_A(A)>0$;
pick one such subset and call it $A$.

Then, also by assumption \textit{(iii)}, for any $\sb_A\in\Scal_A$, there is an open subset $\Ocal(\sb_A)\subseteq \Scal_A$ containing $\sb_A$, such that $p_{\sbt_\A|\sb_\A}(\cdot|\sb_A)>0$ within $\Ocal(\sb_A)$. 

Define the following space
\begin{equation}
    \Rcal_A:=\{(\sb_A,\sbt_A):\sb_A\in\Scal_A,\sbt_A\in\Ocal(\sb_A)\}
\end{equation}
and, recalling that $\Ac=\{1, ..., n_s\} \setminus A$ denotes the complement of $A$, define
\begin{equation}
    \Rcal := \Ccal\times \Scal_{\Ac}\times \Rcal_A
\end{equation}
which is a topological subspace of
$\Ccal\times\Scal\times\Scal$.

By assumptions \textit{(ii)} and \textit{(iii)},
$p_\zb$ is smooth and fully supported, and $p_{\sbt_A|\sb_A}(\cdot|\sb_A)$ is smooth and fully supported on $\Ocal(\sb_A)$ for any $\sb_A\in\Scal_A$.
Therefore,
the measure
$\mu_{(\cb,\sb_{\Ac},\sb_{A},\sbt_{A})|A}$ has fully supported, strictly-positive density on $\Rcal$
w.r.t.\ a strictly positive measure on $\Rcal$.
In other words, $p_\zb \times p_{\sbt_A|\sb_A}$ is fully supported (i.e., strictly positive) on $\Rcal$.

Now consider the intersection $\Ucal\cap \Rcal$ of 
the open set $\Ucal$ 
with the topological subspace $\Rcal$.

Since $\Ucal$ is open, by the definition of topological subspaces, the intersection $\Ucal\cap \Rcal\subseteq \Rcal$ is \textit{open} in $\Rcal$, (and thus has the same dimension as $\Rcal$ if non-empty). 

Moreover, since $\Ocal(\sb_A^*)$ is open containing $\sb_A^*$, there exists $\eta'>0$ such that $\{\sb^*_{-l}\}\times(s^*_l-\eta',s^*_l)\subset \Ocal(\sb_A^*)$. Thus, for $\eta''=\min(\eta,\eta')>0$,
\begin{equation}
     \{\cb^*\}\times\{\sb^*_{\Ac}\}
     \times\left(\{\sb^*_{A\setminus\{l\}}\}\times (s^*_l,s^*_l+\eta)\right)
     \times \left(\{\sb^*_{A\setminus\{l\}}\}\times(s^*_l-\eta'',s^*_l)\right)\subset \Rcal.
\end{equation}
In particular, this implies that 
\begin{equation}
\label{eq:R_nonempty} 
    \{\cb^*\}\times\left(\{\sb^*_{-l}\}\times(s^*_l,s^*_l+\eta)\right) \times \left(\{\sb^*_{-l}\}\times(s^*_l-\eta'',s^*_l)\right)\subset \Rcal,
\end{equation}
Now, since $\eta''\leq\eta$, the LHS of~\eqref{eq:R_nonempty} is also in $\Ucal$ according to~\eqref{eq:U_nonempty}, so the intersection $\Ucal\cap\Rcal$ is \textit{non-empty}.

In summary, the intersection $\Ucal\cap\Rcal\subseteq \Rcal$:
\begin{itemize}
    \item is non-empty (since both $\Ucal$ and $\Rcal$ contain the LHS of~\eqref{eq:U_nonempty});
    \item is an open subset of the topological subspace $\Rcal$ of $\Ccal\times\Scal\times\Scal$ (since it is the intersection of an open set, $\Ucal$, with $\Rcal$);
    \item satisfies $\psi>0$ (since this holds for all of $\Ucal$);
    \item is fully supported w.r.t. the generative process (since this holds for all of $\Rcal$).
\end{itemize}

As a consequence,
\begin{equation}
    \PP\left(\psi(\cb,\sb,\sbt)>0
    |A\right)\geq\PP(\Ucal\cap \Rcal)>0,
\end{equation}
where $\PP$ denotes probability w.r.t.\ the true generative process $p$.

Since $p_A(A)>0$, this is a \textbf{contradiction} to the invariance~\eqref{eq:contradicted_expression} from Step 1.

Hence, assumption~\eqref{eq:contradiction_assumption} cannot hold, i.e., $\hb_c(\cb,\sb)$ does not depend on any style variable $s_l$. It is thus only a function of $\cb$, i.e., $\cbh=\hb_c(\cb)$.

Finally, smoothness and invertibility of $\hb_c:\Ccal\rightarrow\Ccal$ follow from smoothness and invertibility of $\hb$, as established in Step 1.

This concludes the proof that $\cbh$ is related to the true content $\cb$ via a smooth invertible mapping.
\end{proof}

\subsection{Proof of\texorpdfstring{~\Cref{thm:CL}}{thmcl}}
\label{app:proof_discriminative}
\discriminative*

\begin{proof}
As in the proof of~\cref{thm:main}, the proof again consists of two main steps.

In the first step, we show that the representation $\zbh=\gb(\xb)$ extracted by any $\gb$ that minimises $\Lcal_\mathrm{Align}$ is related to the true latent $\zb$ through a smooth invertible mapping $\hb$, and that $\zbh$ must satisfy invariance across $(\xb,\xbt)$ in the first $n_c$ (content) components almost surely (a.s.) with respect to (w.r.t.) the true generative process.

In the second step, we use the same argument by contradiction as in Step 2 of the proof of~\cref{thm:main}, to show that $\cbh=\hb(\zb)_{1:n_c}$ can only depend on the true content $\cb$ and not on style $\sb$.

\paragraph{Step 1.}
From the form of the objective~\eqref{eq:CL_MSE_objective}, it is clear that $\Lcal_\mathrm{Align}\geq 0$ with equality if and only if $\gb(\xbt)_{1:n_c}=\gb(\xb)_{1:n_c}$ for all $(\xb, \xbt)$ s.t. $p_{\xb, \xbt}(\xb, \xbt)>0$.

Moreover, it follows from the assumed generative process that the global minimum of zero is attained by the true unmixing $\fbinv$ since
\begin{equation}
\fbinv(\xb)_{1:n_c}=\cb=\cbt=\fbinv(\xbt)_{1:n_c}
\end{equation}
holds a.s.\ (i.e., with probability one) w.r.t.\ the true generative process $p$. 

Hence, there exists at least one smooth invertible function ($\fbinv$) which attains the global minimum.

Let $\gb$ be \textit{any} function attaining the global minimum of $\Lcal_\mathrm{Align}$ of zero.

As argued above, this implies that (a.s.\ w.r.t.\ $p$):
\begin{equation}
    \gb(\xbt)_{1:n_c}=\gb(\xb)_{1:n_c}.
\end{equation}

Writing $\gb=\hb \circ \fbinv$, where $\hb$ is the smooth, invertible function $\hb=\gb \circ \fb$
we obtain (a.s.\ w.r.t.\ $p$):
\begin{equation}
    \cbh=\hb(\zbt)_{1:n_c}=\hb(\zb)_{1:n_c}.
\end{equation}

Note that this is the same invariance condition as~\eqref{eq:contradicted_expression} derived in Step 1 of the proof of~\cref{thm:main}.

\paragraph{Step 2.}
It remains to show that $\hb(\zb)_{1:n_c}$ can only depend on the true content $\cb$ and not on any of the style variables $\sb$.
To show this, we use the same Step 2 as in the proof of~\cref{thm:main}.
\end{proof}

\subsection{Proof of\texorpdfstring{~\Cref{thm:CL_MaxEnt}}{thmclmaxent}}
\label{app:proof_CL_MaxEnt}
\discriminativeMaxEnt*

\begin{proof}
The proof consists of three main steps.

In the first step, we show that the representation $\cbh=\gb(\xb)$ extracted by any smooth function $\gb$ that minimises~\eqref{eq:CL_MSE_MaxEnt_objective} 
is related to the true latent $\zb$ through a smooth mapping $\hb$; that $\cbh$ must satisfy invariance across $(\xb,\xbt)$ almost surely (a.s.) with respect to (w.r.t.) the true generative process $p$; and that $\cbh$ must follow a uniform distribution on $(0,1)^{n_c}$.

In the second step, we use the same argument by contradiction as in Step 2 of the proof of~\cref{thm:main}, to show that $\cbh=\hb(\zb)$ can only depend on the true content $\cb$ and not on style $\sb$.

Finally, in the third step, we show that $\hb$ must be a bijection, i.e., invertible, using a result from~\cite{zimmermann2021contrastive}.

\paragraph{Step 1.}

The global minimum of $\Lcal_{\mathrm{AlignMaxEnt}}$ is reached when the first term (alignment) is minimised (i.e., equal to zero) and the second term (entropy) is maximised.

Without additional moment constraints, the \textit{unique} maximum entropy distribution on $(0,1)^{n_c}$ is the uniform distribution~\cite{jaynes1982rationale,cover2012elements}.

First, we show that there exists a smooth function $\gb^*:\Xcal\rightarrow(0,1)^{n_c}$ which attains the global minimum of $\Lcal_{\mathrm{AlignMaxEnt}}$.

To see this, consider the function $\fbinv_{1:n_c}:\Xcal\rightarrow\Ccal$, i.e., the inverse of the true mixing~$\fb$, restricted to its first $n_c$ dimensions.
This exists and is smooth since $\fb$ is smooth and invertible by assumption \textit{(i)}.
Further, we have $\fbinv(\xb)_{1:n_c}=\cb$ by definition.

We now build a function $\db:\Ccal\rightarrow(0,1)^{n_c}$ which maps $\cb$ to a uniform random variable on $(0,1)^{n_c}$ using a recursive construction known as the \textit{Darmois construction}~\cite{darmois1951construction,hyvarinen1999nonlinear}. 

Specifically, we define
\begin{equation}
    d_i(\cb) := F_i(c_i|\cb_{1:i-1})=\PP(C_i\leq c_i|\cb_{1:i-1}), 
    \quad \quad \quad i=1, ..., n_c,
\end{equation}
where $F_i$ denotes the conditional cumulative distribution function (CDF) of $c_i$ given $\cb_{1:i-1}$.

By construction, $\db(\cb)$ is uniformly distributed on $(0,1)^{n_c}$~\cite{darmois1951construction,hyvarinen1999nonlinear}. 

Further, $\db$ is smooth by the assumption that $p_\zb$ (and thus $p_\cb$) is a smooth density.

Finally, we define
\begin{equation}
\label{eq:global_min_max_ent}
    \gb^*:=\db\circ\fbinv_{1:n_c}:\Xcal\rightarrow(0,1)^{n_c},
\end{equation}
which is a smooth function since it is a composition of two smooth functions. 

\begin{claim}
\label{claim:global_min_max_ent}
$\gb^*$ as defined in~\eqref{eq:global_min_max_ent} attains the global minimum of $\Lcal_{\mathrm{AlignMaxEnt}}$.
\end{claim}

\textbf{Proof of Claim~\ref{claim:global_min_max_ent}.}
Using $\fbinv(\xb)_{1:n_c}=\cb$ and $\fbinv(\xbt)_{1:n_c}=\cbt$, we have
\begin{align}
    \Lcal_{\mathrm{AlignMaxEnt}}(\gb^*)
    &=
    \EE_{(\xb,\xbt)\sim p_{(\xb, \xbt)}}
\left[
\bignorm{
\gb^*(\xb)-\gb^*(\xbt)
}_2^2
\right] - H\left(\gb^*(\xb)\right)
\\
&=
\EE_{(\xb,\xbt)\sim p_{(\xb, \xbt)}}
\left[
\bignorm{
\db(\cb)-\db(\cbt)
}^2_2
\right] - H\left(\db(\cb)\right)
\\
&= 0 
\end{align}
where in the last step we have used the fact that $\cb=\cbt$ almost surely w.r.t.\ to the ground truth generative process $p$ described in~\cref{sec:problem_formulation}, so the first term is zero; and the fact that $\db(\cb)$ is uniformly distributed on $(0,1)^{n_c}$ and the uniform distribution on the unit hypercube has zero entropy, so the second term is also zero.

Next, let $\gb:\Xcal\rightarrow(0,1)^{n_c}$ be \textit{any} smooth function which attains the global minimum of~\eqref{eq:CL_MSE_MaxEnt_objective}, i.e.,
\begin{equation}
\label{eq:g_align_maxent}
\Lcal_{\mathrm{AlignMaxEnt}}(\gb)=\EE_{(\xb,\xbt)\sim p_{(\xb, \xbt)}}
\left[
\bignorm{
\gb(\xb)-\gb(\xbt)
}^2_2
\right] - H\left(\gb(\xb)\right)
=0.
\end{equation}

Define
    $\hb:=\gb\circ\fb:\Zcal\rightarrow (0,1)^{n_c}$
which is smooth because both $\gb$ and $\fb$ are smooth.

Writing $\xb=\fb(\zb)$, ~\eqref{eq:g_align_maxent} then implies in terms of $\hb$:
\begin{align}
\label{eq:MaxEntinvariance}
\EE_{(\xb,\xbt)\sim p_{(\xb, \xbt)}}
\left[
\bignorm{
\hb(\zb)-\hb(\zbt)
}^2_2
\right]
&= 0\, ,
\\
\label{eq:MaxEntuniformity}
H\left(\hb(\zb)\right)
&=
0\, .
\end{align}

Equation \eqref{eq:MaxEntinvariance} implies that the same invariance condition~\eqref{eq:contradicted_expression} used in the proofs of~\cref{thm:main,thm:CL} must hold (a.s. w.r.t. $p$), and~\eqref{eq:MaxEntuniformity} implies that $\cbh=\hb(\zb)$ must be uniformly distributed on $(0,1)^{n_c}$.

\paragraph{Step 2.}
Next, we show that $\hb(\zb)=\hb(\cb,\sb)$ can only depend on the true content $\cb$ and not on any of the style variables $\sb$. For this we use the same Step 2 as in the proofs of~\cref{thm:main,thm:CL}.

\paragraph{Step 3.}
Finally, we show that the mapping $\cbh=\hb(\cb)$ is invertible.

To this end, we make use of the following result from~\cite{zimmermann2021contrastive}.

\begin{proposition}[Proposition 5 of~\cite{zimmermann2021contrastive}]
\label{prop:bijectivity}
Let $\Mcal, \Ncal$ be simply connected and oriented $\Ccal^1$ manifolds without boundaries and $h : \Mcal \rightarrow \Ncal$ be a differentiable map.
Further, let the random variable $\zb \in \Mcal$ be distributed according to $\zb \sim p(\zb)$ for a regular density function $p$, i.e., $0 < p < \infty$.
If the pushforward $p_{\#h}(\zb)$ of $p$ through $h$ is also a regular density, i.e., $0 < p_{\#h} < \infty$,
then $h$ is a bijection.
\end{proposition}

We apply this result to the simply connected and oriented $\Ccal^1$ manifolds without boundaries $\Mcal=\Ccal$ and  $\Ncal=(0,1)^{n_c}$, and the smooth (hence, differentiable) map $\hb:\Ccal\rightarrow (0,1)^{n_c}$ which maps the random variable $\cb$ to a uniform random variable $\cbh$ (as established in Step 1). 

Since both $p_\cb$ (by assumption) and the uniform distribution (the pushforward of $p_\cb$ through $\hb$) are regular densities in the sense of~\cref{prop:bijectivity}, we conclude that $\hb$ is a bijection, i.e., invertible.

We have shown that for any smooth $\gb:\Xcal\rightarrow(0,1)^{n_c}$ which minimises $\Lcal_{\mathrm{AlignMaxEnt}}$, we have that $\cbh=\gb(\xb)=\hb(\cb)$ for a smooth and invertible $\hb:\Ccal\rightarrow(0,1)^{n_c}$, i.e., $\cb$ is block-identified by $\gb$.
\end{proof}

\section{Additional details on the Causal3DIdent data set}
\label{app:data_set_details}
Using the Blender rendering engine~\citep{blender}, 3DIdent~\citep{zimmermann2021contrastive} is a recently proposed benchmark which  contains hallmarks of natural environments (e.g. shadows, different lighting conditions, a 3D object), but allows for identifiability evaluation by exposing the underlying generative factors. 

Each $224\times224\times3$ image in the dataset shows a coloured 3D object which is located and rotated above a coloured ground in a 3D space. Furthermore, each scene contains a coloured spotlight which is focused on the object and located on a half-circle around the scene. The images are rendered based on a $10$-dimensional latent, where: (i) three dimensions describe the XYZ position of the object, (ii) three dimensions describe the rotation of the object in Euler angles, (iii) two dimensions describe the colour (hue) of the object and the ground of the scene, respectively, and (iv) two dimensions describe the position and colour (hue) of the spotlight. For influence of the latent factors on the renderings, see Fig.~2 of~\citep{zimmermann2021contrastive}.

\subsection{Details on introduced object classes}
3DIdent contained a single object class, Teapot~\citep{newell1975utah}. We add \textbf{six} additional object classes: Hare~\citep{turk1994bunny}, Dragon~\citep{StanfordScanRep}, Cow~\citep{KeenanScanRep}, Armadillo~\citep{Krishnamurthy1996armadillo}, Horse~\citep{Praun2000horse}, Head~\citep{SuggestContourGallery}.

\subsection{Details on latent causal graph}

In 3DIdent, the latents are uniformly sampled independently. We instead impose a causal graph over the variables (see~\cref{fig:3DIdent_causal_graph}). While object class and all environment variables (spotlight position, spotlight hue, background hue) are sampled independently, all object variables are dependent. Specifically, for spotlight position, spotlight hue, and background hue, we sample from $U(-1,1)$. We impose the dependence by varying the mean ($\mu$) of a truncated normal distribution with standard deviation $\sigma=0.5$, truncated to the range $[-1,1]$.

Object rotation is dependent solely on object class, see~\Cref{tab:rotation_dependence} for details. Object position is dependent on both object class \& spotlight position, see~\Cref{tab:position_dependence}. Object hue is dependent on object class, background hue, \& object hue, see~\Cref{tab:hue_dependence}. Hares blending into their environment as a form of active camouflage has been observed in Alaskan~\citep{Lepusothus}, Arctic~\citep{arctichare}, \& Snowshoe hares.

\begin{table}[!ht]
    \centering
    \caption{Given a certain object class, the center of the truncated normal distribution from which we sample \textit{rotation} latents varies.}
    \label{tab:rotation_dependence}
    \vspace{0.5em}
    \begin{tabular}{c|ccc}
    \toprule
        object class & $\mu(\phi)$ & $\mu(\theta)$ & $\mu(\psi)$ \\
        \midrule
        Teapot & -0.35 & 0.35 & 0.35\\
        Hare & 0.35 & -0.35 & 0.35\\
        Dragon & 0.35 & 0.35 & -0.35\\
        Cow & 0.35 & -0.35 & -0.35\\
        Armadillo & -0.35 & 0.35 & -0.35\\
        Horse & -0.35 & -0.35 & 0.35\\
        Head & -0.35 & -0.35 & -0.35\\
        \bottomrule
    \end{tabular}
\end{table}

\begin{table}[!ht]
    \centering
    \caption{Given a certain object class \& spotlight position, the center of the truncated normal distribution from which we sample \textit{$xy$-position} latents varies. Note the spotlight position $\text{pos}_\text{spl}$ is rescaled from $[-1,1]$ to $[-\pi/2,\pi/2]$.}
    \label{tab:position_dependence}
    \vspace{0.5em}
    \begin{tabular}{c|ccc}
    \toprule
        object class & $\mu(x)$ & $\mu(y)$ & $\mu(z)$ \\
        \midrule
        Teapot & 0 & 0 & 0\\
        Hare & $-\sin(\text{pos}_\text{spl})$ & $-\cos(\text{pos}_\text{spl})$ & 0\\
        Dragon & $-\sin(\text{pos}_\text{spl})$ & $-\cos(\text{pos}_\text{spl})$ & 0\\
        Cow & $\sin(\text{pos}_\text{spl})$ & $\cos(\text{pos}_\text{spl})$ & 0\\
        Armadillo & $\sin(\text{pos}_\text{spl})$ & $\cos(\text{pos}_\text{spl})$ & 0\\
        Horse & $-\sin(\text{pos}_\text{spl})$ & $-\cos(\text{pos}_\text{spl})$ & 0\\
        Head & $\sin(\text{pos}_\text{spl})$ & $\cos(\text{pos}_\text{spl})$ & 0\\
        \bottomrule
    \end{tabular}
\end{table}

\begin{table}[!ht]
    \centering
    \caption{Given a certain object class, background hue, and spotlight hue, the center of the truncated normal distribution from which we sample the \textit{object hue} latent varies. Note that for the Hare and Dragon classes, in particular, the object either blends in or stands out from the environment.}
    \label{tab:hue_dependence}
    \vspace{0.5em}
    \begin{tabular}{c|c}
    \toprule
        object class & $\mu(\text{hue})$ \\
        \midrule
        Teapot & $0$\\
        Hare & $\frac{\text{hue}_\text{bg} + \text{hue}_\text{spl}}{2}$\\
        Dragon & $-\frac{\text{hue}_\text{bg} + \text{hue}_\text{spl}}{2}$\\\\
        Cow & $-0.35$\\
        Armadillo & $0.7$\\
        Horse & $-0.7$\\
        Head & $0.35$\\
        \bottomrule
    \end{tabular}
\end{table}

\subsection{Dataset Visuals}
We show $40$ random samples from the marginal of each object class in Causal3DIdent in~\Cref{fig:teapot_visuals,fig:hare_visuals,fig:dragon_visuals,fig:cow_visuals,fig:arm_visuals,fig:hor_visuals,fig:head_visuals}.

\begin{figure}[H]
    \centering
    \includegraphics[width=\textwidth]{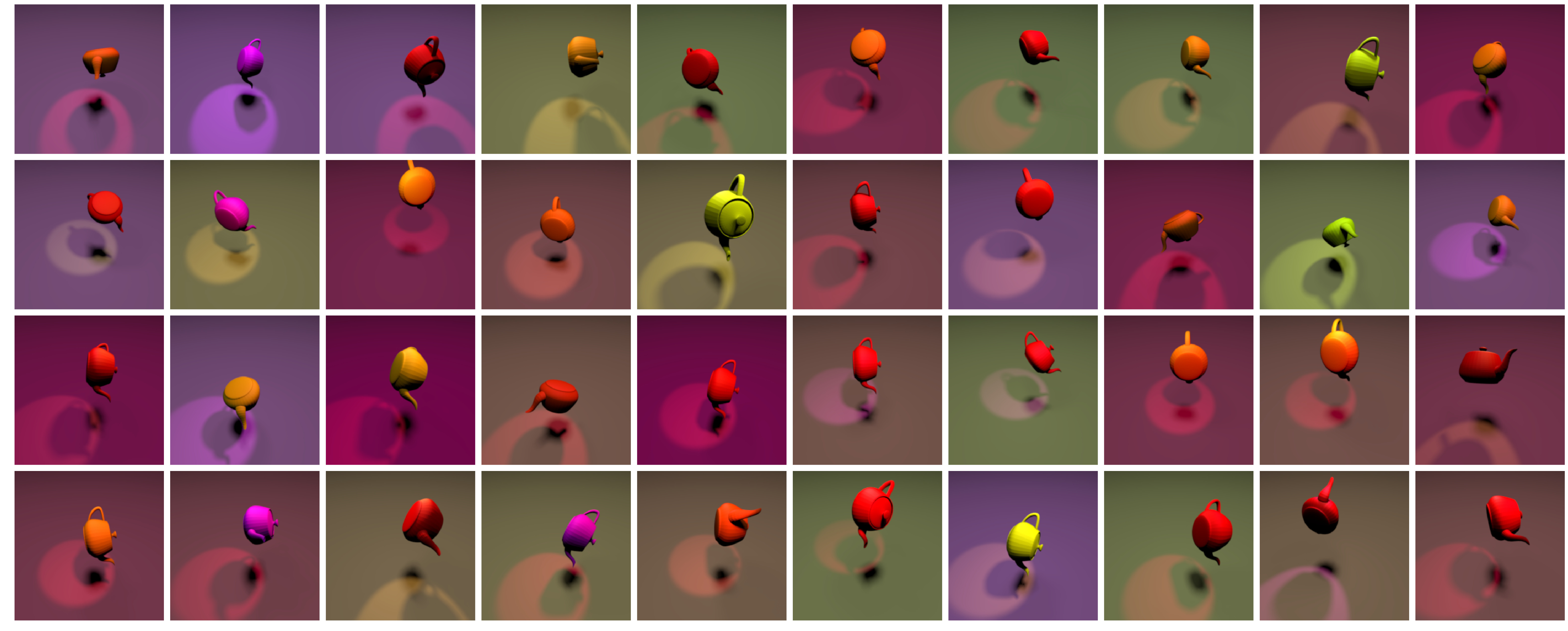}
    \caption{40 random samples from the marginal distribution of the \textit{Teapot} object class.}
    \label{fig:teapot_visuals}
\end{figure}

\begin{figure}[H]
    \centering
    \includegraphics[width=\textwidth]{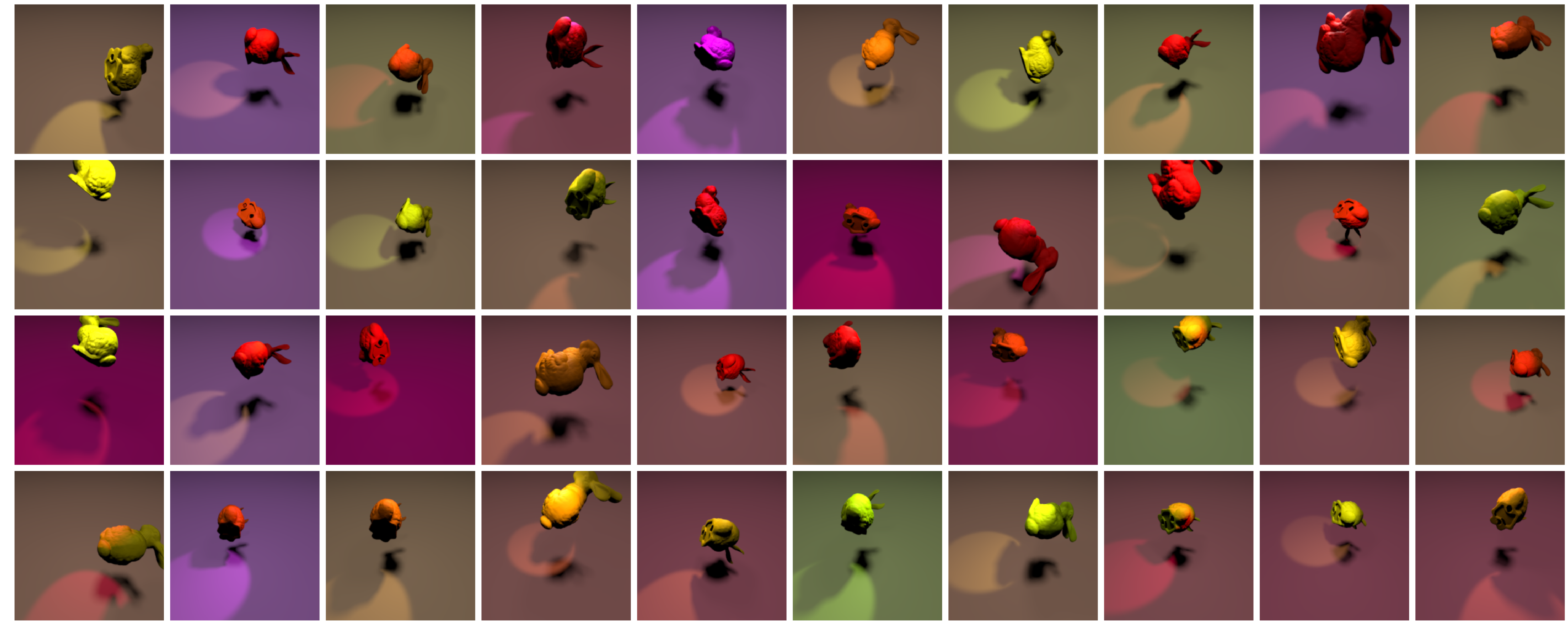}
    \caption{40 random samples from the marginal distribution of the \textit{Hare} object class.}
    \label{fig:hare_visuals}
\end{figure}

\begin{figure}[H]
    \centering
    \includegraphics[width=\textwidth]{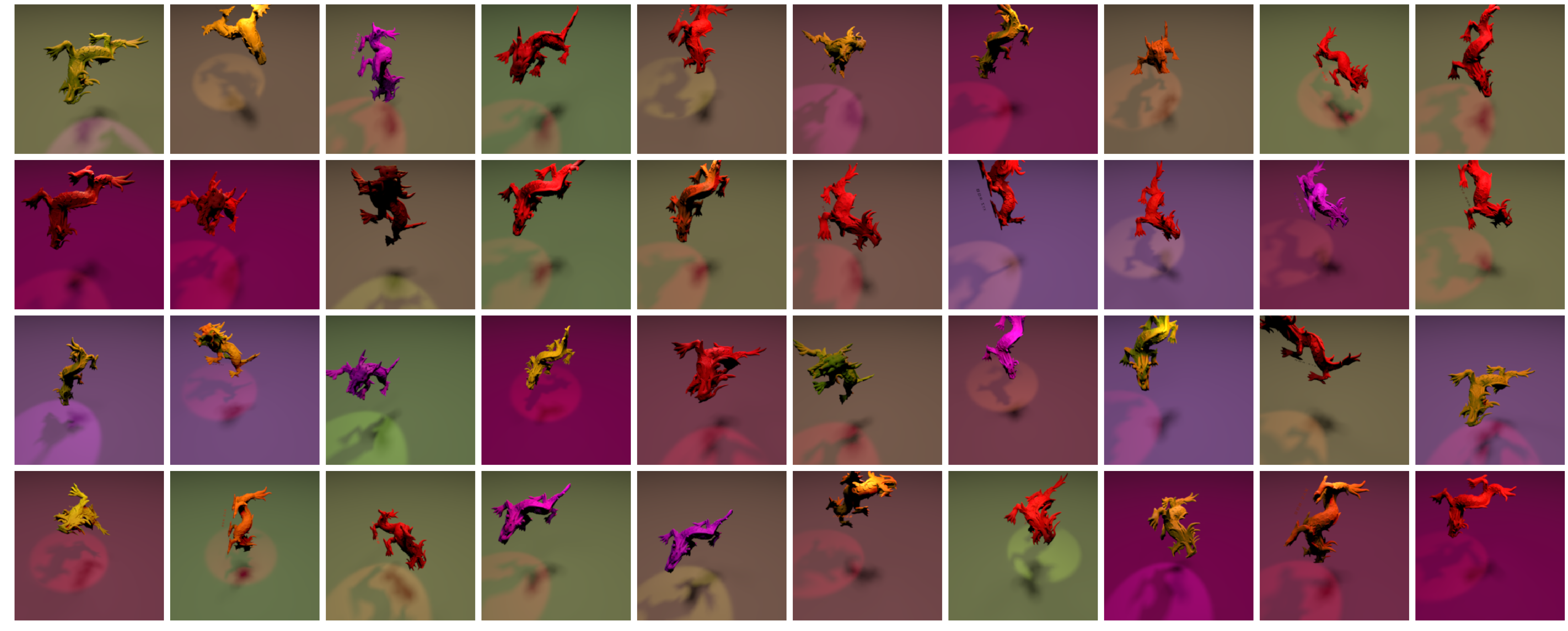}
    \caption{40 random samples from the marginal distribution of the \textit{Dragon} object class.}
    \label{fig:dragon_visuals}
\end{figure}

\begin{figure}[H]
    \centering
    \includegraphics[width=\textwidth]{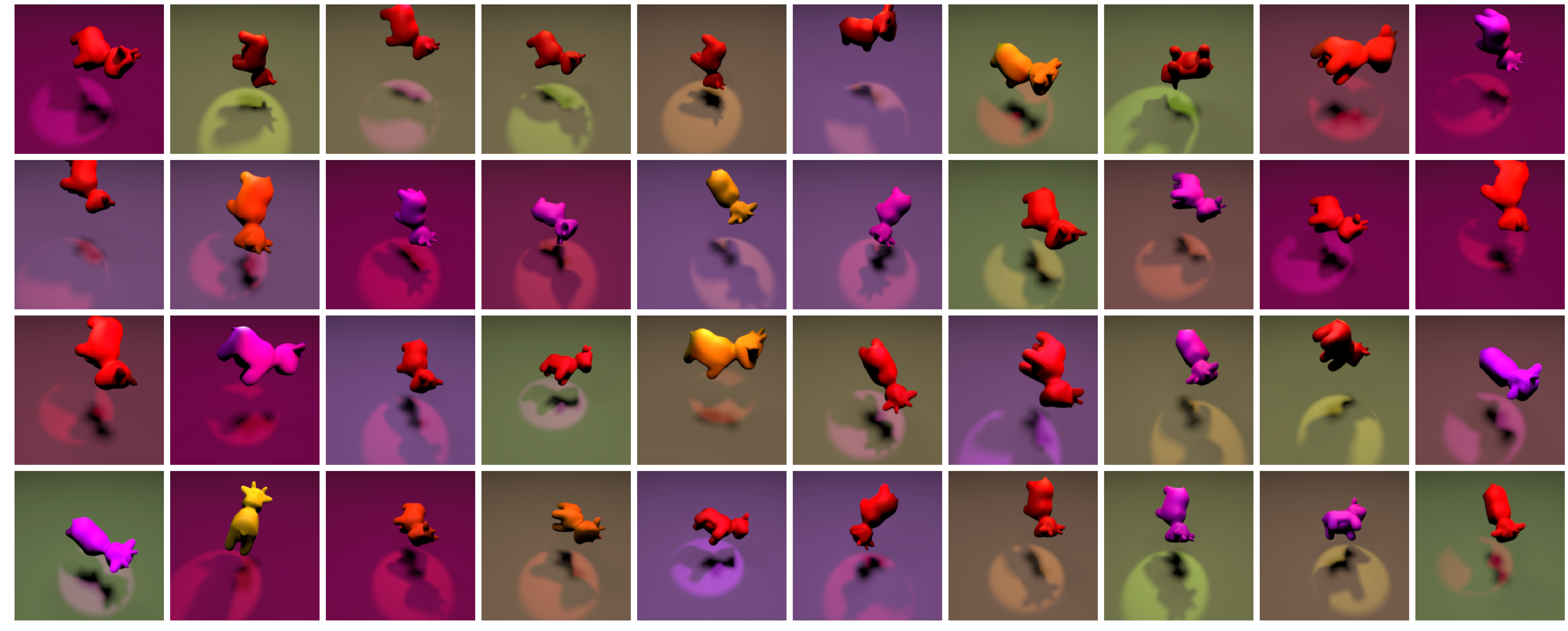}
    \caption{40 random samples from the marginal distribution of the \textit{Cow} object class.}
    \label{fig:cow_visuals}
\end{figure}

\begin{figure}[H]
    \centering
    \includegraphics[width=\textwidth]{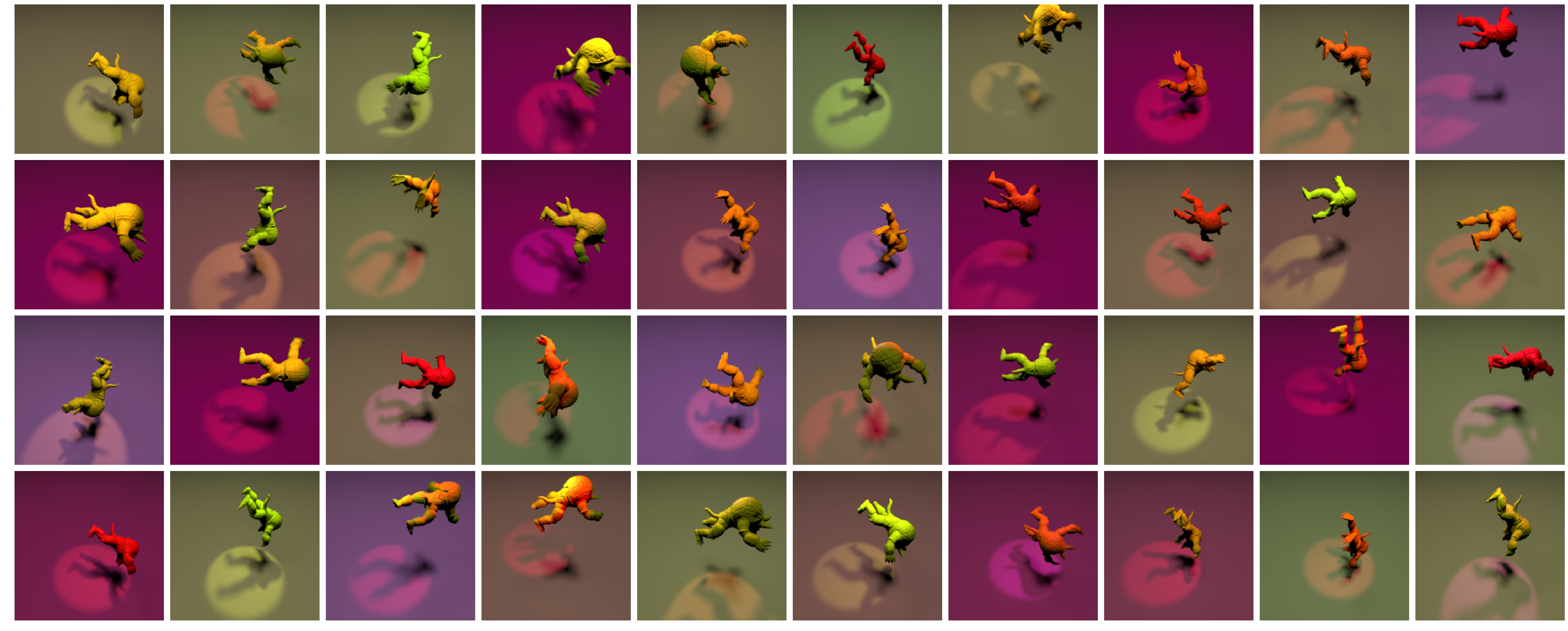}
    \caption{40 random samples from the marginal distribution of the \textit{Armadillo} object class.}
    \label{fig:arm_visuals}
\end{figure}

\begin{figure}[H]
    \centering
    \includegraphics[width=\textwidth]{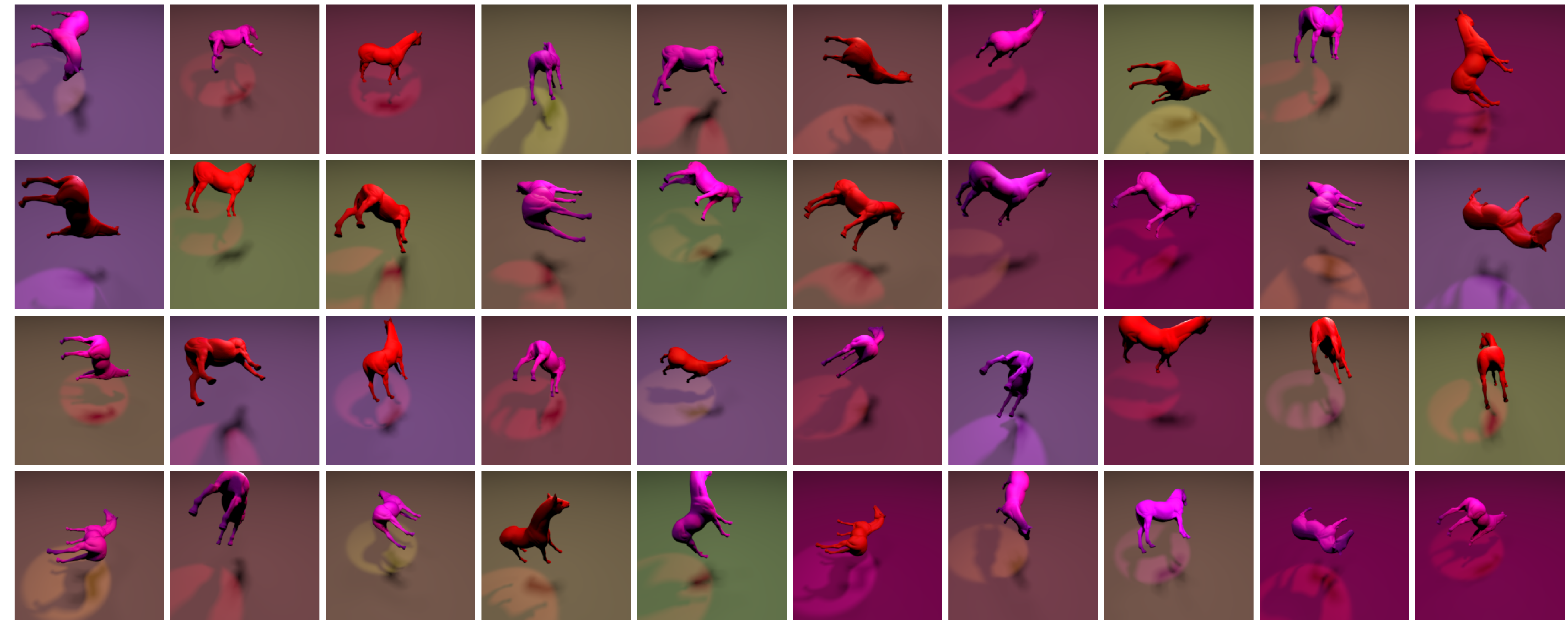}
    \caption{40 random samples from the marginal distribution of the \textit{Horse} object class.}
    \label{fig:hor_visuals}
\end{figure}

\begin{figure}[H]
    \centering
    \includegraphics[width=\textwidth]{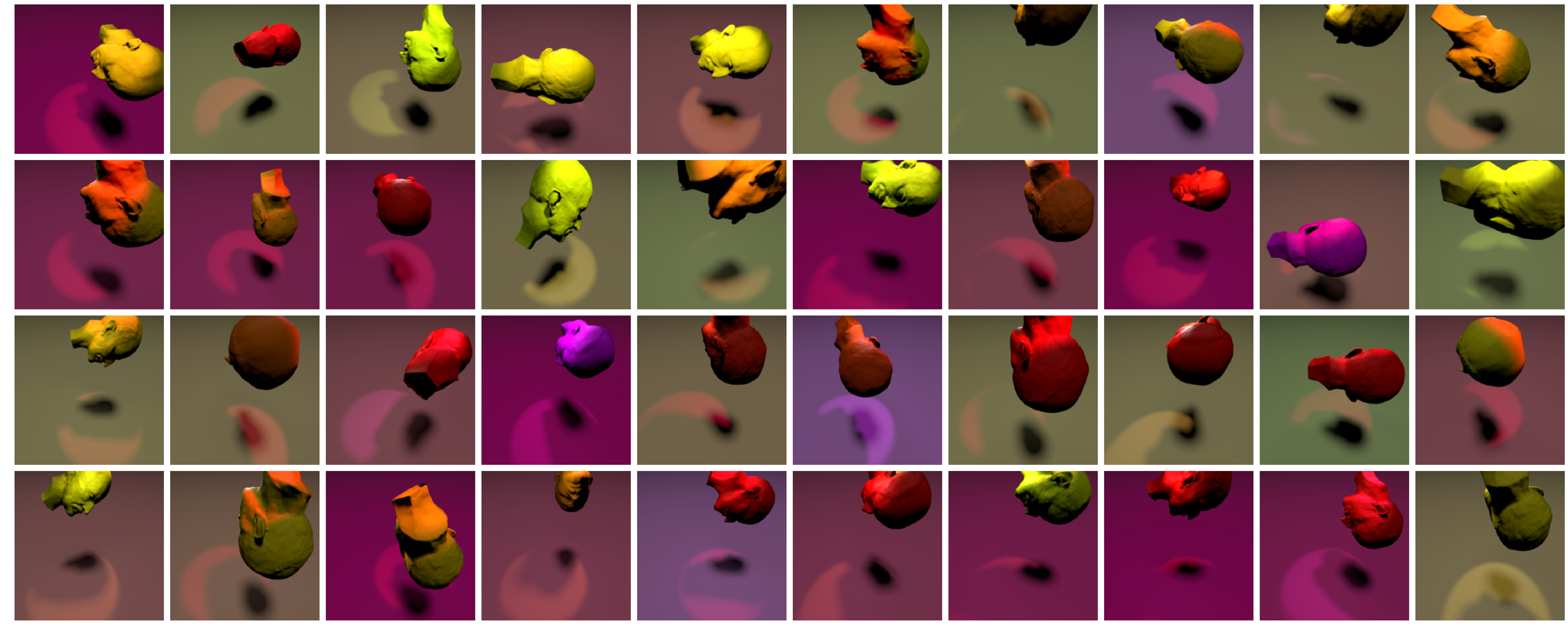}\
    \caption{40 random samples from the marginal distribution of the \textit{Head} object class.}
    \label{fig:head_visuals}
\end{figure}

\clearpage
\section{Additional results}
\label{app:additional_results}
\begin{itemize}
    \item \Cref{app:numerical} contains numerical experiments, namely linear evaluation \& an ablation on $\text{dim}(\cbh)$. 
    \item \Cref{app:causal} contains experiments on \textit{Causal3DIdent}, namely (i) nonlinear \& linear evaluation results of the output \& intermediate feature representation of \texttt{SimCLR} with results for the individual axes of object position \& rotation, and (ii) evaluation of \texttt{BarlowTwins}. 
    \item \Cref{app:mpi} contains experiments on the \textit{MPI3D-real} dataset~\cite{gondal2019transfer}, namely SimCLR \& a supervised sanity check.
\end{itemize}
\subsection{Numerical Data}
\label{app:numerical}
In~\Cref{tab:app_numsim_linear}, we report mean $\pm$ std.\ dev.\ $R^2$ over $3$ random seeds across four generative processes of increasing complexity using \textit{linear} (instead of nonlinear) regression to predict $\cb$ from $\cbh$. The block-identification %
of content can clearly still be seen even if we consider a linear fit. 

In~\Cref{fig:model_latent_ablation}, we perform an ablation 
on $\text{dim}(\cbh)$, visualising how varying the dimensionality of the learnt representation affects identifiability of the ground-truth content \& style partition. Generally, if $\text{dim}(\cbh)<n_c$,
there is insufficient capacity to encode all content, so a lower-dimensional mixture of content is learnt. 
Conversely, if $\text{dim}(\cbh)>n_c$, the excess capacity is used to encode some style information, as that increases entropy. 

\begin{table}[h!]
    \centering
    \caption{Results using linear regression for the experiment on numerical data presented in~\Cref{sec:experiment_1_numerical_simulation}}
    \label{tab:app_numsim_linear}
    \small
    \vspace{0.25em}
    \begin{tabular}{ccccc}
    \toprule
    \multicolumn{3}{c}{\textbf{Generative process}} & \multicolumn{2}{c}{$\bm{R^2}$ \textbf{(linear)}}  \\
    \cmidrule(r){1-3}\cmidrule(r){4-5}
    \textbf{p(chg.)} & \textbf{Stat.} & \textbf{Cau.} & \textbf{Content $\cb$} & \textbf{Style $\sb$} \\
    \midrule
    1.0 & \xmark & \xmark & $\textbf{1.00} \pm 0.00$ & $\textcolor{red}{0.00} \pm 0.00$ \\
    0.75 & \xmark & \xmark & $\textbf{0.99} \pm 0.00$ & $\textcolor{red}{0.00} \pm 0.00$ \\
    0.75 & \cmark & \xmark & $\textbf{0.97} \pm 0.03$ & $0.37 \pm 0.05$ \\
    0.75 & \cmark & \cmark & $\textbf{0.98} \pm 0.01$ & $\textbf{0.78} \pm 0.07$ \\
    \bottomrule
    \end{tabular}
\end{table}

\begin{figure}[!h]
\centering
\includegraphics[width=\textwidth]{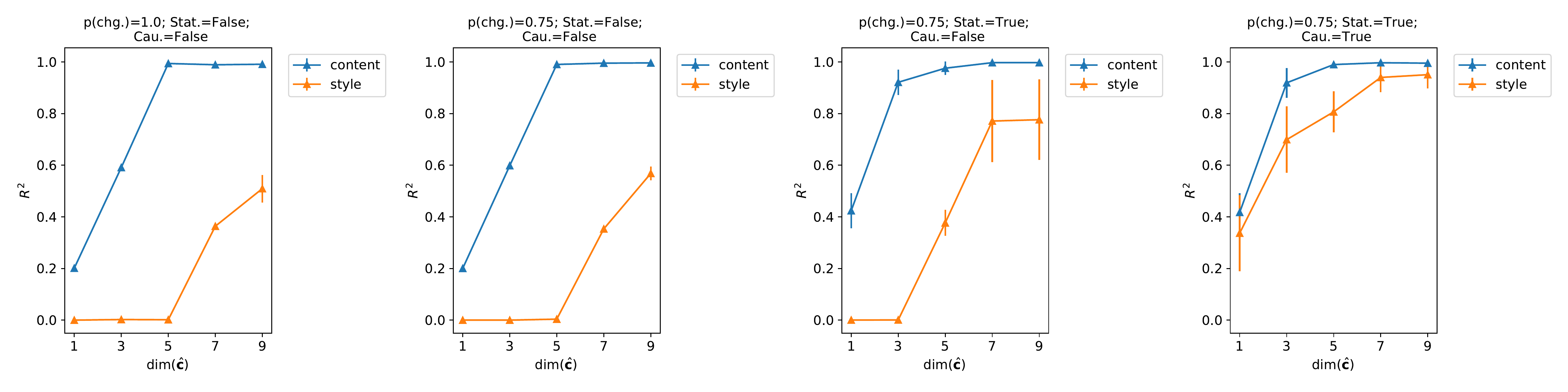}
\caption{Identifiability of the content \& style partition in the numerical experiment as a function of the model latent dimensionality}
\label{fig:model_latent_ablation}
\end{figure}

\paragraph{On Dependence.}
As can be seen from~\cref{tab:app_numsim_linear}, the corresponding inset table in~\cref{sec:experiment_1_numerical_simulation}, and~\cref{fig:model_latent_ablation}, scores for identifying style increase substantially when statistical dependence within blocks and causal dependence between blocks are included. 
This finding can be explained as follows. 

If we compare the performance for small latent dimensionalities ($\text{dim}(\cbh)<n_c$) between the first two (without) and the third plot (with statistical dependence) of~\Cref{fig:model_latent_ablation},
we observe a significantly higher score in identifying content for the latter (e.g., $R^2$ of ca.\ 0.4 vs 0.2 at $\text{dim}(\cbh)=1$). 
This suggest that the introduction of statistical dependence between content variables  (as well as between style variables, and in how style variables change) in the third plot/row, reduces the effective dimensionality of the ground-truth latents and thus leads to higher content identifiability for the same~$\text{dim}(\cbh)<n_c$.
Since the $R^2$ for content is already close to 1 for $\text{dim}(\cbh)=3$ in the third plot of~\cref{fig:model_latent_ablation} (due to the smaller effective dimensionality induced by statistical dependence between~$\cb$), when $\text{dim}(\cbh)=n_c=5$ is used (as reported in~\cref{tab:app_numsim_linear}), excess capacity is used to encode style, leading to a positive $R^2$.

Regarding causal dependence (i.e., the fourth plot in~\cref{fig:model_latent_ablation} and fourth row in~\cref{tab:app_numsim_linear}), we note that the ground truth dependence between $\cb$ and $\sb$ is linear, i.e., $p(\sb|\cb)$ is centred at a linear transformation $\ab+B\cb$ of $\cb$, see the data generating process in~\Cref{app:experimental_details} for details.
Given that our evaluation consists of predicting the ground truth $\cb$ and $\sb$ from the learnt representation $\cbh=\gb(\xb)$, if we were to block-identify $\cb$ according to~\cref{def:block-identifiability}, we should be able to also predict some aspects of $\sb$ from $\cbh$, due to the linear dependence between $\cb$ and $\sb$. This manifests in a relatively large $R^2$ for $\sb$ in the last row of~\cref{tab:app_numsim_linear} and the corresponding table in~\cref{sec:experiment_1_numerical_simulation}.

To summarise, we highlight two main takeaways: (i) when latent dependence is present, this may reduce the effective dimensionality, so that some style is encoded in addition to content unless a smaller representation size is chosen; (ii) even though the learnt representation isolates content in the sense of~\cref{def:block-identifiability}, it may still be predictive of style when content and style are (causally) dependent.

\subsection{\textit{Causal3DIdent}}
\label{app:causal}
\paragraph{Full version of~\Cref{tbl:final_nonlinear_abbrv}:} In~\Cref{tbl:final_nonlinear_full}, we a) provide the results for the individual axes of object position \& rotation and b) present additional rows omitted from~\Cref{tbl:final_nonlinear_abbrv} for space considerations.

Interestingly, we find that the variance across the individual axes is significantly higher for object position than object rotation. If we compare the causal dependence imposed for object position (see~\Cref{tab:position_dependence}) to the causal dependence imposed for object rotation (see~\Cref{tab:rotation_dependence}), we can observe that the dependence imposed over individual axes is also significantly more variable for position than rotation, i.e., for $x$ the sine nonlinearity is used, for $y$ the cosine nonlinearity is used, while for $z$, no dependence is imposed. 

Regarding the additional rows, we can observe that the composition of image-level rotation \& crops yields results quite similar to solely using crops, a relationship which mirrors how transforming the rotation \& position latents yields results quite similar to solely transforming the position latents. This suggests that the rotation variables are difficult to disentangle from the position variables in Causal3DIdent, regardless of whether data augmentation or latent transforms are used.

Finally, we can observe that applying image-level rotation in conjunction with small crops \& colour distortion does lead to a difference in the encoding, background hue is preserved, while the scores for object position \& rotation appear to slightly decrease. When using three augmentations as opposed to two, the effects of the individual augmentations are lessened. While colour distortion discourages the encoding of background hue, both small crops \& image-level rotation encourages it, and thus it is preserved when all three augmentations are used. While colour distortion encourages the encoding of object position \& rotation, both small crops \& image-level rotation discourage it, but as a causal relationship exists between the class variable and said latents, the scores merely decrease, the latents are still for the most part preserved. In reality, where complex interactions between latent variables abound, the effect of data augmentations may be uninterpretable, however with Causal3DIdent, we are able to interpret their effects in the presence of rich visual complexity and causal dependencies, even when applying three distinct augmentations in tandem.

\definecolor{LTcolor}{rgb}{0.95,1,1}
\definecolor{DAcolor}{rgb}{1,1,0.95}
\begin{table}[h!]
\centering
\caption{Full version of~\Cref{tbl:final_nonlinear_abbrv}.
}
\label{tbl:final_nonlinear_full}
\vspace{0.25em}
\resizebox{\textwidth}{!}{
\small
\begin{tabular}{lc|cccc|ccc|ccc}
\toprule
\multirow{2}{*}{\textbf{Views generated by}} & \multirow{2}{*}{\textbf{Class}} & \multicolumn{4}{c}{\textbf{Positions}} & \multicolumn{3}{c}{\textbf{Hues}} & \multicolumn{3}{c}{\textbf{Rotations}} \\
\cmidrule(r){3-6}\cmidrule(r){7-9}\cmidrule(r){10-12}
& & $\text{object}(x)$ & $\text{object}(y)$ & $\text{object}(z)$ & $\text{spotlight}$ & $\text{object}$ & $\text{spotlight}$ & $\text{background}$ & $\text{object}(\phi)$ & $\text{object}(\theta)$ & $\text{object}(\psi)$
\\
\midrule
\rowcolor{DAcolor}
DA: colour distortion  & 
$0.42 \pm 0.01$ & 
$\textbf{0.58} \pm 0.01$ & $\textbf{0.75} \pm 0.00$ & $\textbf{0.52} \pm 0.01$ &
$\textcolor{red}{0.17} \pm 0.00$ & $\textcolor{red}{0.10} \pm 0.01$ & $\textcolor{red}{0.01} \pm 0.00$ & $\textcolor{red}{0.01} \pm 0.00$ & 
$0.36 \pm 0.01$ & $0.33 \pm 0.01$ & $0.32 \pm 0.00$\\
\rowcolor{LTcolor}
LT: change hues & 
$\textbf{1.00} \pm 0.00$ & 
$\textbf{0.81} \pm 0.02$ & $\textbf{0.81} \pm 0.02$ & $\textcolor{red}{0.15} \pm 0.02$ &
$\textbf{0.91} \pm 0.00$ & $0.30 \pm 0.00$ & $\textcolor{red}{0.00} \pm 0.00$ & $\textcolor{red}{0.00} \pm 0.00$ & 
$0.30 \pm 0.02$ & $0.30 \pm 0.01$ & $0.30 \pm 0.01$\\
\midrule
\rowcolor{DAcolor}
DA: crop (large) & 
$0.28 \pm 0.04$ & 
$\textcolor{red}{0.04} \pm 0.02$ & $\textcolor{red}{0.03} \pm 0.01$ & $\textcolor{red}{0.19} \pm 0.02$ &
$\textcolor{red}{0.21} \pm 0.13$ & $\textbf{0.87} \pm 0.00$ & $\textcolor{red}{0.09} \pm 0.02$ & $\textbf{1.00} \pm 0.00$ & %
$\textcolor{red}{0.00} \pm 0.00$ & $\textcolor{red}{0.05} \pm 0.00$ & $\textcolor{red}{0.02} \pm 0.00$\\
\rowcolor{DAcolor}
DA: crop (small) & 
$\textcolor{red}{0.14} \pm 0.00$ &
$\textcolor{red}{0.00} \pm 0.00$ & $\textcolor{red}{0.01} \pm 0.02$ & $\textcolor{red}{0.00} \pm 0.00$ &
$\textcolor{red}{0.00} \pm 0.01$ & $\textcolor{red}{0.00} \pm 0.00$ & $\textcolor{red}{0.00} \pm 0.00$ & $\textbf{1.00} \pm 0.00$ &
 $\textcolor{red}{0.00} \pm 0.00$ & $\textcolor{red}{0.00} \pm 0.00$ & $\textcolor{red}{0.00} \pm 0.00$\\
\rowcolor{LTcolor}
LT: change positions & 
$\textbf{1.00} \pm 0.00$ &
$\textcolor{red}{0.01} \pm 0.00$ & $0.47 \pm 0.01$ & $\textcolor{red}{0.01} \pm 0.00$ &
$\textcolor{red}{0.00} \pm 0.01$ & $0.46 \pm 0.02$ & $\textcolor{red}{0.00} \pm 0.00$ & $\textbf{0.97} \pm 0.00$ & 
$0.30 \pm 0.00$ & $0.29 \pm 0.00$ & $0.28 \pm 0.00$\\
\midrule
\rowcolor{DAcolor}
DA: crop (large) + colour distortion & 
$\textbf{0.97} \pm 0.00$ & 
$\textbf{0.59} \pm 0.03$ & $\textbf{0.52} \pm 0.01$ & $\textbf{0.68} \pm 0.01$ &
$\textbf{0.59} \pm 0.05$ & $0.28 \pm 0.00$ & $\textcolor{red}{0.01} \pm 0.01$ & $\textcolor{red}{0.01} \pm 0.00$ & %
$\textbf{0.74} \pm 0.01$ & $\textbf{0.78} \pm 0.00$ & $\textbf{0.72} \pm 0.00$\\
\rowcolor{DAcolor}
DA: crop (small) + colour distortion &
$\textbf{1.00} \pm 0.00$ & 
$\textbf{0.72} \pm 0.02$ & $\textbf{0.65} \pm 0.02$ & $\textbf{0.70} \pm 0.00$ &
$\textbf{0.93} \pm 0.00$ & $0.30 \pm 0.01$ & $\textcolor{red}{0.00} \pm 0.00$ & $\textcolor{red}{0.02} \pm 0.03$ &
$\textbf{0.53} \pm 0.00$ & $\textbf{0.57} \pm 0.01$ & $\textbf{0.58} \pm 0.01$ \\
\rowcolor{LTcolor}
LT: change positions + hues & 
$\textbf{1.00} \pm 0.00$ &
$\textcolor{red}{0.10} \pm 0.10$ & $0.49 \pm 0.02$ & $\textcolor{red}{0.06} \pm 0.05$ &
$\textcolor{red}{0.07} \pm 0.08$ & $0.32 \pm 0.02$ & $\textcolor{red}{0.00} \pm 0.01$ & $\textcolor{red}{0.02} \pm 0.03$ & %
$0.34 \pm 0.09$ & $0.34 \pm 0.04$ & $0.34 \pm 0.08$\\
\midrule
\rowcolor{DAcolor}
DA: rotation &
$0.33 \pm 0.06$ & 
$0.29 \pm 0.03$ & $\textcolor{red}{0.11} \pm 0.01$ & $\textcolor{red}{0.12} \pm 0.04$ &
$\textcolor{red}{0.23} \pm 0.12$ & $\textbf{0.83} \pm 0.01$ & $0.30 \pm 0.12$ & $\textbf{0.99} \pm 0.00$ &
$\textcolor{red}{0.02} \pm 0.01$ & $\textcolor{red}{0.06} \pm 0.03$ & $\textcolor{red}{0.07} \pm 0.01$\\
\rowcolor{LTcolor}
LT: change rotations & 
$\textbf{1.00} \pm 0.00$ & 
$\textbf{0.78} \pm 0.01$ & $\textbf{0.72} \pm 0.03$ & $\textcolor{red}{0.09} \pm 0.03$ &
$\textbf{0.90} \pm 0.00$ & $0.41 \pm 0.00$ & $\textcolor{red}{0.00} \pm 0.00$ & $\textbf{0.97} \pm 0.00$ & 
$0.28 \pm 0.00$ & $0.28 \pm 0.00$ & $0.28 \pm 0.00$\\
\midrule
\rowcolor{DAcolor}
DA: rotation + colour distortion & 
$\textbf{0.59} \pm 0.01$ & 
$\textbf{0.63} \pm 0.01$ & $\textbf{0.57} \pm 0.08$ & $\textbf{0.54} \pm 0.02$ &
$\textcolor{red}{0.21} \pm 0.01$ & $\textcolor{red}{0.12} \pm 0.02$ & $\textcolor{red}{0.01} \pm 0.00$ & $\textcolor{red}{0.01} \pm 0.00$ & 
$0.36 \pm 0.03$ & $0.34 \pm 0.04$ & $0.30 \pm 0.03$ \\
\rowcolor{LTcolor}
LT: change rotations + hues &
$\textbf{1.00} \pm 0.00$ & 
$\textbf{0.80} \pm 0.02$ & $\textbf{0.77} \pm 0.01$ & $\textcolor{red}{0.13} \pm 0.02$ &
$\textbf{0.91} \pm 0.00$ & $0.30 \pm 0.00$ & $\textcolor{red}{0.00} \pm 0.00$ & $\textcolor{red}{0.00} \pm 0.00$ & 
$0.28 \pm 0.00$ & $0.28 \pm 0.01$ & $0.28 \pm 0.00$\\
\midrule
\rowcolor{DAcolor}
DA: rot.\ + crop (lg) &
$0.26 \pm 0.01$ & $\textcolor{red}{0.03} \pm 0.02$ & $\textcolor{red}{0.03} \pm 0.01$ & $\textcolor{red}{0.15} \pm 0.04$ & $\textcolor{red}{0.04} \pm 0.03$ & $\textbf{0.84} \pm 0.06$ & $\textcolor{red}{0.10} \pm 0.01$ & $\textbf{1.00} \pm 0.00$ & $\textcolor{red}{0.00} \pm 0.00$ & $\textcolor{red}{0.04} \pm 0.02$ & $\textcolor{red}{0.02} \pm 0.00$
\\
\rowcolor{DAcolor}
DA: rot.\ + crop (sm) &
$\textcolor{red}{0.15} \pm 0.00$ & $\textcolor{red}{0.00} \pm 0.00$ & $\textcolor{red}{0.00} \pm 0.00$ & $\textcolor{red}{0.00} \pm 0.00$ & $\textcolor{red}{0.00} \pm 0.00$ & $\textcolor{red}{0.00} \pm 0.00$ & $\textcolor{red}{0.00} \pm 0.00$ & $\textbf{1.00} \pm 0.00$ & $\textcolor{red}{0.00} \pm 0.00$ & $\textcolor{red}{0.00} \pm 0.00$ & $\textcolor{red}{0.00} \pm 0.00$
\\
\rowcolor{LTcolor}
LT: change rot.\ + pos.\ &
$\textbf{1.00} \pm 0.00$ & $\textcolor{red}{0.02} \pm 0.03$ & $0.48 \pm 0.02$ & $\textcolor{red}{0.01} \pm 0.01$ & $\textcolor{red}{0.02} \pm 0.03$ & $0.49 \pm 0.03$ & $\textcolor{red}{0.03} \pm 0.02$ & $\textbf{0.98} \pm 0.00$ & $0.29 \pm 0.01$ & $0.28 \pm 0.01$ & $0.28 \pm 0.01$
\\
\midrule
\rowcolor{DAcolor}
DA: rot.\ + crop (lg) + col.\ dist.\ &
$\textbf{0.99} \pm 0.00$ & $\textbf{0.69} \pm 0.03$ & $\textbf{0.60} \pm 0.01$ & $\textbf{0.70} \pm 0.02$ & $\textbf{0.86} \pm 0.03$ & $0.28 \pm 0.00$ & $\textcolor{red}{0.01} \pm 0.00$ & $\textcolor{red}{0.01} \pm 0.00$ & $\textbf{0.60} \pm 0.01$ & $\textbf{0.64} \pm 0.02$ & $\textbf{0.61} \pm 0.01$
\\
\rowcolor{DAcolor}
DA: rot.\ + crop (sm) + col.\ dist.\ &
$\textbf{1.00} \pm 0.00$ & $\textbf{0.61} \pm 0.02$ & $\textbf{0.59} \pm 0.01$ & $\textbf{0.64} \pm 0.01$ & $\textbf{0.82} \pm 0.01$ & $0.38 \pm 0.00$ & $\textcolor{red}{0.01} \pm 0.01$ & $\textbf{0.78} \pm 0.03$ & $0.44 \pm 0.00$ & $0.48 \pm 0.02$ & $0.45 \pm 0.01$
\\
\rowcolor{LTcolor}
LT: change rot.\ + pos.\ + hues & 
$\textbf{1.00} \pm 0.00$ & $\textcolor{red}{0.20} \pm 0.12$ & $\textbf{0.50} \pm 0.04$ & $\textcolor{red}{0.14} \pm 0.11$ & $\textcolor{red}{0.15} \pm 0.12$ & $0.32 \pm 0.01$ & $\textcolor{red}{0.00} \pm 0.00$ & $\textcolor{red}{0.02} \pm 0.01$ & $0.33 \pm 0.04$ & $0.33 \pm 0.02$ & $0.32 \pm 0.03$
\\
\bottomrule
\end{tabular}
}
\end{table}

\paragraph{Linear identifiability:} In~\Cref{tbl:final_linear_full}, we present results evaluating all continuous variables with linear regression. While, as expected, $R^2$ scores are reduced across the board, we can observe that even with a linear fit, the patterns observed in~\Cref{tbl:final_nonlinear_full} persist.

\definecolor{LTcolor}{rgb}{0.95,1,1}
\definecolor{DAcolor}{rgb}{1,1,0.95}
\begin{table}[h!]
\centering
\caption{Evaluation results using a linear fit for not only class, but all continuous variables. 
}
\label{tbl:final_linear_full}
\vspace{0.25em}
\resizebox{\textwidth}{!}{
\small
\begin{tabular}{lc|cccc|ccc|ccc}
\toprule
\multirow{2}{*}{\textbf{Views generated by}} & \multirow{2}{*}{\textbf{Class}} & \multicolumn{4}{c}{\textbf{Positions}} & \multicolumn{3}{c}{\textbf{Hues}} & \multicolumn{3}{c}{\textbf{Rotations}} \\
\cmidrule(r){3-6}\cmidrule(r){7-9}\cmidrule(r){10-12}
& & $\text{object}(x)$ & $\text{object}(y)$ & $\text{object}(z)$ & $\text{spotlight}$ & $\text{object}$ & $\text{spotlight}$ & $\text{background}$ & $\text{object}(\phi)$ & $\text{object}(\theta)$ & $\text{object}(\psi)$
\\
\midrule
\rowcolor{DAcolor}
DA: colour distortion  & 
$0.42 \pm 0.01$ & $0.37 \pm 0.03$ & $0.20 \pm 0.16$ & $0.23 \pm 0.02$ & $0.01 \pm 0.01$ & $0.03 \pm 0.01$ & $-0.00 \pm 0.00$ & $-0.00 \pm 0.00$ & $0.13 \pm 0.01$ & $0.04 \pm 0.01$ & $0.09 \pm 0.02$ \\
\rowcolor{LTcolor}
LT: change hues & 
$1.00 \pm 0.00$ & $0.72 \pm 0.07$ & $0.56 \pm 0.04$ & $-0.00 \pm 0.00$ & $0.65 \pm 0.07$ & $0.29 \pm 0.01$ & $-0.00 \pm 0.00$ & $-0.00 \pm 0.00$ & $0.27 \pm 0.01$ & $0.26 \pm 0.03$ & $0.26 \pm 0.01$ \\
\midrule
\rowcolor{DAcolor}
DA: crop (large) & 
$0.28 \pm 0.04$ & $0.00 \pm 0.00$ & $0.02 \pm 0.00$ & $0.04 \pm 0.07$ & $0.08 \pm 0.13$ & $0.51 \pm 0.05$ & $0.03 \pm 0.02$ & $0.20 \pm 0.04$ & $0.00 \pm 0.00$ & $0.02 \pm 0.00$ & $0.01 \pm 0.00$ \\
\rowcolor{DAcolor}
DA: crop (small) & 
$0.14 \pm 0.00$ & $-0.00 \pm 0.00$ & $-0.00 \pm 0.00$ & $-0.00 \pm 0.00$ & $-0.00 \pm 0.00$ & $-0.00 \pm 0.00$ & $-0.00 \pm 0.00$ & $0.17 \pm 0.05$ & $-0.00 \pm 0.00$ & $-0.00 \pm 0.00$ & $-0.00 \pm 0.00$ \\
\rowcolor{LTcolor}
LT: change positions & 
$1.00 \pm 0.00$ & $-0.00 \pm 0.00$ & $0.44 \pm 0.02$ & $-0.00 \pm 0.00$ & $-0.00 \pm 0.00$ & $0.29 \pm 0.04$ & $0.00 \pm 0.00$ & $0.73 \pm 0.16$ & $0.26 \pm 0.01$ & $0.25 \pm 0.03$ & $0.25 \pm 0.04$ \\
\midrule
\rowcolor{DAcolor}
DA: crop (large) + colour distortion & 
$0.97 \pm 0.00$ & $0.12 \pm 0.02$ & $0.24 \pm 0.03$ & $0.21 \pm 0.00$ & $0.08 \pm 0.03$ & $0.13 \pm 0.01$ & $-0.00 \pm 0.00$ & $-0.00 \pm 0.00$ & $0.14 \pm 0.04$ & $0.18 \pm 0.05$ & $0.22 \pm 0.02$ \\
\rowcolor{DAcolor}
DA: crop (small) + colour distortion &
$1.00 \pm 0.00$ & $0.35 \pm 0.02$ & $0.50 \pm 0.01$ & $0.19 \pm 0.03$ & $0.80 \pm 0.01$ & $0.28 \pm 0.00$ & $-0.00 \pm 0.00$ & $-0.00 \pm 0.00$ & $0.29 \pm 0.00$ & $0.30 \pm 0.00$ & $0.29 \pm 0.01$ \\
\rowcolor{LTcolor}
LT: change positions + hues & 
$1.00 \pm 0.00$ & $0.00 \pm 0.00$ & $0.42 \pm 0.06$ & $0.00 \pm 0.00$ & $0.00 \pm 0.00$ & $0.27 \pm 0.02$ & $-0.00 \pm 0.00$ & $-0.00 \pm 0.00$ & $0.23 \pm 0.07$ & $0.26 \pm 0.03$ & $0.25 \pm 0.04$ \\
\midrule
\rowcolor{DAcolor}
DA: rotation &
$0.33 \pm 0.06$ & $0.04 \pm 0.04$ & $0.04 \pm 0.00$ & $0.02 \pm 0.03$ & $0.12 \pm 0.08$ & $0.46 \pm 0.06$ & $0.06 \pm 0.04$ & $0.30 \pm 0.13$ & $0.00 \pm 0.00$ & $0.04 \pm 0.02$ & $0.02 \pm 0.00$ \\
\rowcolor{LTcolor}
LT: change rotations & 
$1.00 \pm 0.00$ & $0.34 \pm 0.21$ & $0.48 \pm 0.03$ & $-0.00 \pm 0.00$ & $0.60 \pm 0.15$ & $0.28 \pm 0.00$ & $0.00 \pm 0.00$ & $0.59 \pm 0.26$ & $0.27 \pm 0.01$ & $0.27 \pm 0.00$ & $0.27 \pm 0.01$ \\
\midrule
\rowcolor{DAcolor}
DA: rotation + colour distortion & 
$0.59 \pm 0.01$ & $0.31 \pm 0.02$ & $0.26 \pm 0.06$ & $0.25 \pm 0.07$ & $0.02 \pm 0.00$ & $0.03 \pm 0.02$ & $-0.00 \pm 0.00$ & $-0.00 \pm 0.00$ & $0.07 \pm 0.01$ & $0.06 \pm 0.01$ & $0.10 \pm 0.01$ \\
\rowcolor{LTcolor}
LT: change rotations + hues &
$1.00 \pm 0.00$ & $0.68 \pm 0.02$ & $0.57 \pm 0.01$ & $-0.00 \pm 0.00$ & $0.72 \pm 0.10$ & $0.29 \pm 0.00$ & $-0.00 \pm 0.00$ & $-0.00 \pm 0.00$ & $0.28 \pm 0.00$ & $0.28 \pm 0.00$ & $0.28 \pm 0.00$ \\
\midrule
\rowcolor{DAcolor}
DA: rot.\ + crop (lg) &
$0.26 \pm 0.01$ & $-0.00 \pm 0.00$ & $0.02 \pm 0.00$ & $0.00 \pm 0.00$ & $0.00 \pm 0.00$ & $0.59 \pm 0.05$ & $0.02 \pm 0.01$ & $0.20 \pm 0.04$ & $0.00 \pm 0.00$ & $0.01 \pm 0.00$ & $0.01 \pm 0.00$
\\
\rowcolor{DAcolor}
DA: rot.\ + crop (sm) &
$0.15 \pm 0.00$ & $-0.00 \pm 0.00$ & $-0.00 \pm 0.00$ & $-0.00 \pm 0.00$ & $-0.00 \pm 0.00$ & $-0.00 \pm 0.00$ & $-0.00 \pm 0.00$ & $0.29 \pm 0.21$ & $-0.00 \pm 0.00$ & $-0.00 \pm 0.00$ & $-0.00 \pm 0.00$ 
\\
\rowcolor{LTcolor}
LT: change rot.\ + pos.\ &
$1.00 \pm 0.00$ & $-0.00 \pm 0.00$ & $0.45 \pm 0.01$ & $-0.00 \pm 0.00$ & $-0.00 \pm 0.00$ & $0.32 \pm 0.02$ & $0.00 \pm 0.00$ & $0.80 \pm 0.09$ & $0.27 \pm 0.00$ & $0.27 \pm 0.01$ & $0.27 \pm 0.01$ 
\\
\midrule
\rowcolor{DAcolor}
DA: rot.\ + crop (lg) + col.\ dist.\ &
$0.99 \pm 0.00$ & $0.23 \pm 0.04$ & $0.26 \pm 0.07$ & $0.26 \pm 0.01$ & $0.51 \pm 0.14$ & $0.21 \pm 0.01$ & $-0.00 \pm 0.00$ & $-0.00 \pm 0.00$ & $0.21 \pm 0.04$ & $0.28 \pm 0.02$ & $0.22 \pm 0.02$
\\
\rowcolor{DAcolor}
DA: rot.\ + crop (sm) + col.\ dist.\ &
$1.00 \pm 0.00$ & $0.26 \pm 0.02$ & $0.48 \pm 0.01$ & $0.21 \pm 0.02$ & $0.61 \pm 0.01$ & $0.31 \pm 0.00$ & $-0.00 \pm 0.00$ & $0.34 \pm 0.02$ & $0.30 \pm 0.00$ & $0.30 \pm 0.01$ & $0.29 \pm 0.01$
\\
\rowcolor{LTcolor}
LT: change rot.\ + pos.\ + hues & 
$1.00 \pm 0.00$ & $0.03 \pm 0.05$ & $0.46 \pm 0.01$ & $0.01 \pm 0.01$ & $0.01 \pm 0.02$ & $0.29 \pm 0.01$ & $-0.00 \pm 0.00$ & $-0.00 \pm 0.00$ & $0.27 \pm 0.00$ & $0.28 \pm 0.01$ & $0.28 \pm 0.01$
\\
\bottomrule
\end{tabular}
}
\end{table}

\paragraph{Intermediate feature evaluation:} In~\Cref{tbl:int_nonlinear_full} and~\Cref{tbl:int_linear_full}, we present evaluation based on the representation from an intermediate layer (i.e., prior to applying a projection layer~\citep{chen2020simple}) with nonlinear and linear regression for the continuous variables, respectively. Note the intermediate layer has an output dimensionality of $100$. While it is clear that all $R^2$ scores are increased across the board, we can notice that certain latents which were discarded in the final layer, were not in an intermediate layer. For example, with ``LT: change hues'', in the final layer the $z$-position was discarded ($R^2=0.15$ in~\Cref{tbl:final_nonlinear_full}), inexplicably we may add, as position is content regardless of axis with this latent transformation. But in the intermediate layer, $z$-position was not discarded ($R^2=0.88$ in~\Cref{tbl:int_nonlinear_full}). 

\definecolor{LTcolor}{rgb}{0.95,1,1}
\definecolor{DAcolor}{rgb}{1,1,0.95}
\begin{table}[h!]
\centering
\caption{Evaluation of an intermediate layer. Logistic regression used for class, kernel ridge regression used for all continuous variables. 
}
\label{tbl:int_nonlinear_full}
\vspace{0.25em}
\resizebox{\textwidth}{!}{
\small
\begin{tabular}{lc|cccc|ccc|ccc}
\toprule
\multirow{2}{*}{\textbf{Views generated by}} & \multirow{2}{*}{\textbf{Class}} & \multicolumn{4}{c}{\textbf{Positions}} & \multicolumn{3}{c}{\textbf{Hues}} & \multicolumn{3}{c}{\textbf{Rotations}} \\
\cmidrule(r){3-6}\cmidrule(r){7-9}\cmidrule(r){10-12}
& & $\text{object}(x)$ & $\text{object}(y)$ & $\text{object}(z)$ & $\text{spotlight}$ & $\text{object}$ & $\text{spotlight}$ & $\text{background}$ & $\text{object}(\phi)$ & $\text{object}(\theta)$ & $\text{object}(\psi)$
\\
\midrule
\rowcolor{DAcolor}
DA: colour distortion  & 
$0.71 \pm 0.02$ & $0.68 \pm 0.02$ & $0.80 \pm 0.01$ & $0.63 \pm 0.01$ & $0.25 \pm 0.01$ & $0.13 \pm 0.00$ & $0.02 \pm 0.01$ & $0.01 \pm 0.01$ & $0.44 \pm 0.01$ & $0.48 \pm 0.01$ & $0.39 \pm 0.00$ \\
\rowcolor{LTcolor}
LT: change hues & 
$1.00 \pm 0.00$ & $0.98 \pm 0.00$ & $0.97 \pm 0.00$ & $0.88 \pm 0.01$ & $0.98 \pm 0.00$ & $0.34 \pm 0.01$ & $-0.00 \pm 0.00$ & $0.20 \pm 0.10$ & $0.71 \pm 0.02$ & $0.68 \pm 0.03$ & $0.68 \pm 0.02$ \\
\midrule
\rowcolor{DAcolor}
DA: crop (large) & 
$0.43 \pm 0.03$ & $0.41 \pm 0.05$ & $0.35 \pm 0.05$ & $0.32 \pm 0.04$ & $0.41 \pm 0.13$ & $0.88 \pm 0.00$ & $0.14 \pm 0.03$ & $1.00 \pm 0.00$ & $0.03 \pm 0.02$ & $0.06 \pm 0.01$ & $0.08 \pm 0.00$ \\
\rowcolor{DAcolor}
DA: crop (small) & 
$0.20 \pm 0.01$ & $0.04 \pm 0.05$ & $0.20 \pm 0.02$ & $0.01 \pm 0.02$ & $0.20 \pm 0.03$ & $-0.00 \pm 0.00$ & $-0.00 \pm 0.00$ & $1.00 \pm 0.00$ & $-0.00 \pm 0.00$ & $-0.00 \pm 0.00$ & $-0.00 \pm 0.00$ \\
\rowcolor{LTcolor}
LT: change positions & 
$1.00 \pm 0.00$ & $0.78 \pm 0.02$ & $0.90 \pm 0.01$ & $0.75 \pm 0.01$ & $0.59 \pm 0.02$ & $0.82 \pm 0.01$ & $0.18 \pm 0.02$ & $0.99 \pm 0.00$ & $0.64 \pm 0.02$ & $0.55 \pm 0.02$ & $0.56 \pm 0.02$ \\
\midrule
\rowcolor{DAcolor}
DA: crop (large) + colour distortion & 
$1.00 \pm 0.00$ & $0.92 \pm 0.00$ & $0.83 \pm 0.00$ & $0.92 \pm 0.00$ & $0.90 \pm 0.01$ & $0.29 \pm 0.00$ & $0.01 \pm 0.01$ & $0.01 \pm 0.01$ & $0.87 \pm 0.00$ & $0.90 \pm 0.00$ & $0.85 \pm 0.00$ \\
\rowcolor{DAcolor}
DA: crop (small) + colour distortion &
$1.00 \pm 0.00$ & $0.92 \pm 0.00$ & $0.87 \pm 0.01$ & $0.90 \pm 0.00$ & $0.97 \pm 0.00$ & $0.46 \pm 0.04$ & $0.02 \pm 0.02$ & $0.58 \pm 0.12$ & $0.79 \pm 0.01$ & $0.83 \pm 0.00$ & $0.79 \pm 0.00$ \\
\rowcolor{LTcolor}
LT: change positions + hues & 
$1.00 \pm 0.00$ & $0.83 \pm 0.04$ & $0.90 \pm 0.01$ & $0.81 \pm 0.04$ & $0.75 \pm 0.08$ & $0.42 \pm 0.09$ & $0.04 \pm 0.02$ & $0.52 \pm 0.20$ & $0.72 \pm 0.05$ & $0.69 \pm 0.07$ & $0.67 \pm 0.06$ \\
\midrule
\rowcolor{DAcolor}
DA: rotation &
$0.46 \pm 0.04$ & $0.35 \pm 0.04$ & $0.19 \pm 0.02$ & $0.28 \pm 0.04$ & $0.34 \pm 0.08$ & $0.85 \pm 0.01$ & $0.35 \pm 0.12$ & $1.00 \pm 0.00$ & $0.03 \pm 0.01$ & $0.08 \pm 0.02$ & $0.10 \pm 0.01$ \\
\rowcolor{LTcolor}
LT: change rotations & 
$1.00 \pm 0.00$ & $0.97 \pm 0.00$ & $0.96 \pm 0.01$ & $0.84 \pm 0.01$ & $0.98 \pm 0.00$ & $0.82 \pm 0.01$ & $0.17 \pm 0.02$ & $0.99 \pm 0.00$ & $0.64 \pm 0.02$ & $0.59 \pm 0.01$ & $0.60 \pm 0.03$ \\
\midrule
\rowcolor{DAcolor}
DA: rotation + colour distortion & 
$0.87 \pm 0.02$ & $0.76 \pm 0.01$ & $0.81 \pm 0.01$ & $0.71 \pm 0.01$ & $0.39 \pm 0.08$ & $0.19 \pm 0.02$ & $-0.00 \pm 0.00$ & $0.02 \pm 0.02$ & $0.55 \pm 0.03$ & $0.55 \pm 0.03$ & $0.48 \pm 0.02$ \\
\rowcolor{LTcolor}
LT: change rotations + hues &

$1.00 \pm 0.00$ & $0.98 \pm 0.00$ & $0.97 \pm 0.00$ & $0.87 \pm 0.00$ & $0.99 \pm 0.00$ & $0.39 \pm 0.05$ & $0.04 \pm 0.02$ & $0.37 \pm 0.21$ & $0.69 \pm 0.01$ & $0.68 \pm 0.01$ & $0.68 \pm 0.00$ \\
\midrule
\rowcolor{DAcolor}
DA: rot.\ + crop (lg) &
$0.43 \pm 0.03$ & $0.38 \pm 0.04$ & $0.34 \pm 0.02$ & $0.28 \pm 0.03$ & $0.30 \pm 0.05$ & $0.86 \pm 0.04$ & $0.17 \pm 0.02$ & $1.00 \pm 0.00$ & $0.02 \pm 0.00$ & $0.05 \pm 0.01$ & $0.10 \pm 0.01$ 
\\
\rowcolor{DAcolor}
DA: rot.\ + crop (sm) &
$0.20 \pm 0.01$ & $0.07 \pm 0.03$ & $0.09 \pm 0.10$ & $0.01 \pm 0.01$ & $0.20 \pm 0.01$ & $-0.00 \pm 0.00$ & $-0.00 \pm 0.00$ & $1.00 \pm 0.00$ & $-0.00 \pm 0.00$ & $-0.00 \pm 0.00$ & $-0.00 \pm 0.00$
\\
\rowcolor{LTcolor}
LT: change rot.\ + pos.\ &
$1.00 \pm 0.00$ & $0.81 \pm 0.01$ & $0.90 \pm 0.01$ & $0.76 \pm 0.01$ & $0.67 \pm 0.04$ & $0.84 \pm 0.01$ & $0.28 \pm 0.04$ & $0.99 \pm 0.00$ & $0.62 \pm 0.02$ & $0.57 \pm 0.01$ & $0.55 \pm 0.01$
\\
\midrule
\rowcolor{DAcolor}
DA: rot.\ + crop (lg) + col.\ dist.\ &
$1.00 \pm 0.00$ & $0.92 \pm 0.01$ & $0.89 \pm 0.00$ & $0.92 \pm 0.00$ & $0.95 \pm 0.01$ & $0.30 \pm 0.00$ & $0.02 \pm 0.02$ & $0.18 \pm 0.16$ & $0.81 \pm 0.00$ & $0.84 \pm 0.00$ & $0.79 \pm 0.00$
\\
\rowcolor{DAcolor}
DA: rot.\ + crop (sm) + col.\ dist.\ &
$1.00 \pm 0.00$ & $0.87 \pm 0.00$ & $0.85 \pm 0.00$ & $0.87 \pm 0.00$ & $0.93 \pm 0.00$ & $0.71 \pm 0.02$ & $0.33 \pm 0.05$ & $0.96 \pm 0.00$ & $0.72 \pm 0.00$ & $0.75 \pm 0.00$ & $0.71 \pm 0.00$
\\
\rowcolor{LTcolor}
LT: change rot.\ + pos.\ + hues & 
$1.00 \pm 0.00$ & $0.84 \pm 0.02$ & $0.91 \pm 0.01$ & $0.82 \pm 0.02$ & $0.78 \pm 0.06$ & $0.40 \pm 0.01$ & $0.06 \pm 0.01$ & $0.50 \pm 0.05$ & $0.72 \pm 0.04$ & $0.70 \pm 0.05$ & $0.67 \pm 0.04$ 
\\
\bottomrule
\end{tabular}
}
\end{table}

\definecolor{LTcolor}{rgb}{0.95,1,1}
\definecolor{DAcolor}{rgb}{1,1,0.95}
\begin{table}[h!]
\centering
\caption{Evaluation of an intermediate layer. Logistic regression used for class, linear regression used for all continuous variables. 
}
\label{tbl:int_linear_full}
\vspace{0.25em}
\resizebox{\textwidth}{!}{
\small
\begin{tabular}{lc|cccc|ccc|ccc}
\toprule
\multirow{2}{*}{\textbf{Views generated by}} & \multirow{2}{*}{\textbf{Class}} & \multicolumn{4}{c}{\textbf{Positions}} & \multicolumn{3}{c}{\textbf{Hues}} & \multicolumn{3}{c}{\textbf{Rotations}} \\
\cmidrule(r){3-6}\cmidrule(r){7-9}\cmidrule(r){10-12}
& & $\text{object}(x)$ & $\text{object}(y)$ & $\text{object}(z)$ & $\text{spotlight}$ & $\text{object}$ & $\text{spotlight}$ & $\text{background}$ & $\text{object}(\phi)$ & $\text{object}(\theta)$ & $\text{object}(\psi)$
\\
\midrule
\rowcolor{DAcolor}
DA: colour distortion  & 
$0.71 \pm 0.02$ & $0.53 \pm 0.01$ & $0.70 \pm 0.01$ & $0.46 \pm 0.01$ & $0.13 \pm 0.01$ & $0.11 \pm 0.01$ & $-0.01 \pm 0.00$ & $0.00 \pm 0.00$ & $0.28 \pm 0.01$ & $0.19 \pm 0.01$ & $0.25 \pm 0.01$ \\
\rowcolor{LTcolor}
LT: change hues & 
$1.00 \pm 0.00$ & $0.93 \pm 0.00$ & $0.93 \pm 0.00$ & $0.60 \pm 0.04$ & $0.95 \pm 0.00$ & $0.31 \pm 0.00$ & $0.01 \pm 0.01$ & $0.06 \pm 0.04$ & $0.44 \pm 0.02$ & $0.41 \pm 0.02$ & $0.42 \pm 0.00$ \\
\midrule
\rowcolor{DAcolor}
DA: crop (large) & 
$0.43 \pm 0.03$ & $0.18 \pm 0.06$ & $0.06 \pm 0.01$ & $0.17 \pm 0.02$ & $0.19 \pm 0.14$ & $0.82 \pm 0.02$ & $0.08 \pm 0.04$ & $0.98 \pm 0.00$ & $0.01 \pm 0.00$ & $0.05 \pm 0.01$ & $0.05 \pm 0.01$ \\
\rowcolor{DAcolor}
DA: crop (small) & 
$0.20 \pm 0.01$ & $0.01 \pm 0.01$ & $0.03 \pm 0.02$ & $0.00 \pm 0.01$ & $0.02 \pm 0.01$ & $-0.00 \pm 0.00$ & $-0.01 \pm 0.00$ & $0.99 \pm 0.00$ & $-0.01 \pm 0.00$ & $-0.01 \pm 0.00$ & $-0.00 \pm 0.01$ \\
\rowcolor{LTcolor}
LT: change positions & 
$1.00 \pm 0.00$ & $0.49 \pm 0.04$ & $0.72 \pm 0.03$ & $0.43 \pm 0.03$ & $0.19 \pm 0.03$ & $0.71 \pm 0.02$ & $0.09 \pm 0.02$ & $0.98 \pm 0.00$ & $0.39 \pm 0.01$ & $0.36 \pm 0.01$ & $0.35 \pm 0.00$ \\
\midrule
\rowcolor{DAcolor}
DA: crop (large) + colour distortion & 
$1.00 \pm 0.00$ & $0.67 \pm 0.03$ & $0.56 \pm 0.01$ & $0.66 \pm 0.02$ & $0.67 \pm 0.03$ & $0.28 \pm 0.00$ & $-0.01 \pm 0.00$ & $0.01 \pm 0.01$ & $0.58 \pm 0.02$ & $0.61 \pm 0.02$ & $0.56 \pm 0.01$ \\
\rowcolor{DAcolor}
DA: crop (small) + colour distortion &
$1.00 \pm 0.00$ & $0.76 \pm 0.01$ & $0.70 \pm 0.02$ & $0.68 \pm 0.01$ & $0.90 \pm 0.00$ & $0.38 \pm 0.03$ & $0.00 \pm 0.01$ & $0.39 \pm 0.13$ & $0.50 \pm 0.02$ & $0.50 \pm 0.01$ & $0.49 \pm 0.01$ \\
\rowcolor{LTcolor}
LT: change positions + hues & 
$1.00 \pm 0.00$ & $0.61 \pm 0.09$ & $0.74 \pm 0.02$ & $0.51 \pm 0.08$ & $0.40 \pm 0.15$ & $0.34 \pm 0.04$ & $0.02 \pm 0.01$ & $0.25 \pm 0.22$ & $0.47 \pm 0.04$ & $0.40 \pm 0.02$ & $0.41 \pm 0.03$ \\
\midrule
\rowcolor{DAcolor}
DA: rotation &
$0.46 \pm 0.04$ & $0.21 \pm 0.02$ & $0.10 \pm 0.01$ & $0.10 \pm 0.02$ & $0.21 \pm 0.09$ & $0.77 \pm 0.01$ & $0.25 \pm 0.11$ & $0.97 \pm 0.01$ & $0.02 \pm 0.01$ & $0.06 \pm 0.02$ & $0.08 \pm 0.01$ \\
\rowcolor{LTcolor}
LT: change rotations & 
$1.00 \pm 0.00$ & $0.92 \pm 0.00$ & $0.88 \pm 0.01$ & $0.51 \pm 0.02$ & $0.95 \pm 0.00$ & $0.70 \pm 0.06$ & $0.07 \pm 0.02$ & $0.98 \pm 0.00$ & $0.36 \pm 0.01$ & $0.34 \pm 0.00$ & $0.34 \pm 0.01$ \\
\midrule
\rowcolor{DAcolor}
DA: rotation + colour distortion & 
$0.87 \pm 0.02$ & $0.60 \pm 0.01$ & $0.62 \pm 0.03$ & $0.52 \pm 0.02$ & $0.23 \pm 0.02$ & $0.18 \pm 0.02$ & $-0.01 \pm 0.00$ & $0.02 \pm 0.01$ & $0.33 \pm 0.04$ & $0.29 \pm 0.01$ & $0.28 \pm 0.01$ \\
\rowcolor{LTcolor}
LT: change rotations + hues &

$1.00 \pm 0.00$ & $0.94 \pm 0.00$ & $0.92 \pm 0.01$ & $0.58 \pm 0.01$ & $0.96 \pm 0.00$ & $0.33 \pm 0.02$ & $0.02 \pm 0.01$ & $0.15 \pm 0.10$ & $0.40 \pm 0.02$ & $0.38 \pm 0.01$ & $0.41 \pm 0.03$ \\
\midrule
\rowcolor{DAcolor}
DA: rot.\ + crop (lg) &
$0.43 \pm 0.03$ & $0.24 \pm 0.04$ & $0.08 \pm 0.02$ & $0.16 \pm 0.03$ & $0.07 \pm 0.01$ & $0.80 \pm 0.04$ & $0.10 \pm 0.01$ & $0.98 \pm 0.00$ & $0.01 \pm 0.00$ & $0.05 \pm 0.01$ & $0.06 \pm 0.01$
\\
\rowcolor{DAcolor}
DA: rot.\ + crop (sm) &
$0.20 \pm 0.01$ & $0.01 \pm 0.01$ & $0.03 \pm 0.01$ & $-0.00 \pm 0.01$ & $0.04 \pm 0.01$ & $-0.01 \pm 0.00$ & $-0.01 \pm 0.00$ & $0.99 \pm 0.00$ & $-0.01 \pm 0.00$ & $-0.01 \pm 0.00$ & $-0.00 \pm 0.01$
\\
\rowcolor{LTcolor}
LT: change rot.\ + pos.\ &
$1.00 \pm 0.00$ & $0.55 \pm 0.05$ & $0.72 \pm 0.02$ & $0.44 \pm 0.04$ & $0.31 \pm 0.08$ & $0.76 \pm 0.01$ & $0.14 \pm 0.01$ & $0.99 \pm 0.00$ & $0.38 \pm 0.01$ & $0.35 \pm 0.01$ & $0.36 \pm 0.02$
\\
\midrule
\rowcolor{DAcolor}
DA: rot.\ + crop (lg) + col.\ dist.\ &
$1.00 \pm 0.00$ & $0.71 \pm 0.01$ & $0.69 \pm 0.01$ & $0.69 \pm 0.00$ & $0.84 \pm 0.03$ & $0.28 \pm 0.00$ & $-0.00 \pm 0.00$ & $0.07 \pm 0.07$ & $0.51 \pm 0.01$ & $0.50 \pm 0.02$ & $0.51 \pm 0.01$
\\
\rowcolor{DAcolor}
DA: rot.\ + crop (sm) + col.\ dist.\ &
$1.00 \pm 0.00$ & $0.66 \pm 0.00$ & $0.69 \pm 0.01$ & $0.65 \pm 0.02$ & $0.83 \pm 0.00$ & $0.57 \pm 0.03$ & $0.18 \pm 0.02$ & $0.89 \pm 0.01$ & $0.46 \pm 0.01$ & $0.45 \pm 0.02$ & $0.44 \pm 0.01$
\\
\rowcolor{LTcolor}
LT: change rot.\ + pos.\ + hues & 
$1.00 \pm 0.00$ & $0.65 \pm 0.04$ & $0.75 \pm 0.05$ & $0.57 \pm 0.03$ & $0.49 \pm 0.12$ & $0.35 \pm 0.01$ & $0.02 \pm 0.01$ & $0.23 \pm 0.04$ & $0.48 \pm 0.04$ & $0.43 \pm 0.01$ & $0.43 \pm 0.01$
\\
\bottomrule
\end{tabular}
}
\end{table}

In~\citep{chen2020simple}, the value in evaluating an intermediate layer as opposed to a final layer is discussed, where the authors demonstrated that predicting the data augmentations applied during training is significantly more accurate from an intermediate layer as opposed to the final layer, implying that the intermediate layer contains much more information about the transformation applied. Our results suggest a distinct hypothesis, the value in using an intermediate layer as a representation for downstream tasks is not due to preservation of style information, as can be seen, $R^2$ scores on style variables are not significantly higher in~\Cref{tbl:int_nonlinear_full} relative to~\Cref{tbl:final_nonlinear_full}. The value is in preservation of all content variables, as we can observe certain content variables are discarded in the final layer, but are preserved in an intermediate layer. With that being said, our theoretical result applies to the final layer, which is why said results were highlighted in the main paper. The discarding of certain content variables is an empirical phenomenon, likely a consequence of a limited number of negative samples in practice, leading to certain content variables being redundant, or unnecessary, for solving the contrastive objective. 

The fact that we can recover certain content variables which appeared discarded in the output from the intermediate layer may suggest that we should be able to decode class. While scores are certainly increased, we do not see such drastic differences in $R^2$ scores, as was seen above. The drastic difference highlighted above was with regards to latent transformation, for which we always observed class encoded as a content variable. So, unfortunately, using an intermediate layer does not rectify the discrepancy between data augmentations and latent transformations. While latent transformations allow us to better interpret the effect of certain empirical techniques~\citep{chen2020simple}, as discussed in the main paper, we cannot make a one-to-one correspondence between data augmentations used in practice and latent transformations.

\paragraph{BarlowTwins:} We repeat our analysis from~\cref{sec:experiment_2_causal3dident} using \texttt{BarlowTwins}~\cite{zbontar2021barlow} (instead of \texttt{SimCLR}) which, as discussed at the end of~\cref{sec:theory_discriminative}, is also loosely related to~\cref{thm:CL_MaxEnt}. The \texttt{BarlowTwins} objective consists of an invariance term, which equates the diagonal elements of the cross-correlation
matrix to $1$, thereby making the embedding invariant to the distortions applied and a redundancy reduction term, which equates the off-diagonal elements of the cross-correlation matrix to $0$, thereby decorrelating the different vector components
of the embedding, reducing  the redundancy between output units. 

In~\Cref{tbl:bt_0.005} we train \texttt{BarlowTwins} with $\lambda=0.0051$, the default value for the hyperparameter which weights the redundancy reduction term relative to the invariance term. To confirm the insights are robust to the value of $\lambda$,in~\Cref{tbl:bt_0.05}, we report results with $\lambda$ increased by an order of magnitude, $\lambda=0.051$. We find that the results mirror~\cref{tbl:final_nonlinear_abbrv}, e.g. colour distortion yields a discarding of hue, crops isolate background hue where the larger the crop, the higher the identifiability of object hue, and crops \& colour distortion yield high accuracy in inferring the object class variable. 

\definecolor{LTcolor}{rgb}{0.95,1,1}
\definecolor{DAcolor}{rgb}{1,1,0.95}
\begin{table}[t]
\centering
\caption{\small \textit{BarlowTwins $\lambda=0.0051$} results: $R^2$ mean $\pm$ std.\ dev.\  over $3$ random seeds. DA: data augmentation, LT: latent transformation, bold: $R^2\geq 0.5$, red: $R^2<0.25$.
Results for individual axes of object position \& rotation are aggregated. %
}
\label{tbl:bt_0.005}
\resizebox{\textwidth}{!}{
\small
\begin{tabular}{lc|cc|ccc|c}
\toprule
\multirow{2}{*}{\textbf{Views generated by}} & \multirow{2}{*}{\textbf{Class}} & \multicolumn{2}{c}{\textbf{Positions}} & \multicolumn{3}{c}{\textbf{Hues}} & \multirow{2}{*}{\textbf{Rotations}} \\
\cmidrule(r){3-4}\cmidrule(r){5-7}
& & $\text{object}$ & $\text{spotlight}$ & $\text{object}$ & $\text{spotlight}$ & $\text{background}$ & 
\\
\midrule
\rowcolor{DAcolor}
DA: colour distortion  & 
$0.48 \pm 0.02$ & $\textbf{0.51} \pm 0.14$ & $\textcolor{red}{0.07} \pm 0.01$ & $\textcolor{red}{0.08} \pm 0.00$ & $\textcolor{red}{0.00} \pm 0.00$ & $\textcolor{red}{0.00} \pm 0.00$ & $\textcolor{red}{0.21} \pm 0.04$ \\
\rowcolor{LTcolor}
LT: change hues & 
$\textbf{1.00} \pm 0.00$ & $\textbf{0.56} \pm 0.20$ & $\textbf{0.76} \pm 0.07$ & $0.30 \pm 0.01$ & $\textcolor{red}{0.00} \pm 0.00$ & $\textcolor{red}{0.01} \pm 0.00$ & $0.35 \pm 0.01$ \\
\midrule
\rowcolor{DAcolor}
DA: crop (large) & 
$\textcolor{red}{0.17} \pm 0.02$ & $\textcolor{red}{0.10} \pm 0.03$ & $\textcolor{red}{0.06} \pm 0.02$ & $0.29 \pm 0.13$ & $\textcolor{red}{0.11} \pm 0.05$ & $\textbf{0.99} \pm 0.00$ & $\textcolor{red}{0.02} \pm 0.01$ \\
\rowcolor{DAcolor}
DA: crop (small) & 
$\textcolor{red}{0.15} \pm 0.00$ & $\textcolor{red}{0.04} \pm 0.02$ & $\textcolor{red}{0.05} \pm 0.02$ & $\textcolor{red}{0.02} \pm 0.01$ & $\textcolor{red}{0.00} \pm 0.01$ & $\textbf{1.00} \pm 0.00$ & $\textcolor{red}{0.00} \pm 0.01$ \\
\rowcolor{LTcolor}
LT: change positions & 
$\textbf{0.88} \pm 0.00$ & $\textcolor{red}{0.19} \pm 0.20$ & $\textcolor{red}{0.05} \pm 0.00$ & $\textbf{0.50} \pm 0.02$ & $\textcolor{red}{0.04} \pm 0.01$ & $\textbf{0.98} \pm 0.00$ & $0.27 \pm 0.03$ \\
\midrule
\rowcolor{DAcolor}
DA: crop (large) + colour distortion & %
$\textbf{0.87} \pm 0.02$ & $0.49 \pm 0.06$ & $0.32 \pm 0.03$ & $0.25 \pm 0.01$ & $\textcolor{red}{0.00} \pm 0.00$ & $\textcolor{red}{0.00} \pm 0.00$ & $\textbf{0.50} \pm 0.02$ \\
\rowcolor{DAcolor}
DA: crop (small) + colour distortion & %
$\textbf{0.81} \pm 0.01$ & $0.39 \pm 0.07$ & $0.42 \pm 0.06$ & $0.47 \pm 0.04$ & $\textcolor{red}{0.03} \pm 0.01$ & $\textbf{0.85} \pm 0.02$ & $0.30 \pm 0.02$ \\
\rowcolor{LTcolor}
LT: change positions + hues & 
$\textbf{1.00} \pm 0.00$ & $0.28 \pm 0.20$ & $\textcolor{red}{0.12} \pm 0.05$ & $0.31 \pm 0.00$ & $\textcolor{red}{0.00} \pm 0.00$ & $\textcolor{red}{0.01} \pm 0.01$ & $0.37 \pm 0.06$ \\
\bottomrule
\end{tabular}
}
\end{table}

\definecolor{LTcolor}{rgb}{0.95,1,1}
\definecolor{DAcolor}{rgb}{1,1,0.95}
\begin{table}[t]
\centering
\caption{\small \textit{BarlowTwins $\lambda=0.051$} results: $R^2$ mean $\pm$ std.\ dev.\  over $3$ random seeds. DA: data augmentation, LT: latent transformation, bold: $R^2\geq 0.5$, red: $R^2<0.25$.
Results for individual axes of object position \& rotation are aggregated. %
}
\label{tbl:bt_0.05}
\resizebox{\textwidth}{!}{
\small
\begin{tabular}{lc|cc|ccc|c}
\toprule
\multirow{2}{*}{\textbf{Views generated by}} & \multirow{2}{*}{\textbf{Class}} & \multicolumn{2}{c}{\textbf{Positions}} & \multicolumn{3}{c}{\textbf{Hues}} & \multirow{2}{*}{\textbf{Rotations}} \\
\cmidrule(r){3-4}\cmidrule(r){5-7}
& & $\text{object}$ & $\text{spotlight}$ & $\text{object}$ & $\text{spotlight}$ & $\text{background}$ & 
\\
\midrule
\rowcolor{DAcolor}
DA: colour distortion  & 
$\textbf{0.52} \pm 0.07$ & $0.43 \pm 0.18$ & $\textcolor{red}{0.07} \pm 0.02$ & $\textcolor{red}{0.10} \pm 0.03$ & $\textcolor{red}{0.00} \pm 0.00$ & $\textcolor{red}{0.00} \pm 0.00$ & $\textcolor{red}{0.21} \pm 0.05$ \\
\rowcolor{LTcolor}
LT: change hues & 
$\textbf{1.00} \pm 0.00$ & $\textbf{0.55} \pm 0.24$ & $\textbf{0.74} \pm 0.02$ & $0.30 \pm 0.00$ & $\textcolor{red}{0.00} \pm 0.00$ & $\textcolor{red}{0.01} \pm 0.01$ & $0.33 \pm 0.02$ \\
\midrule
\rowcolor{DAcolor}
DA: crop (large) & 
$\textcolor{red}{0.19} \pm 0.05$ & $\textcolor{red}{0.08} \pm 0.02$ & $\textcolor{red}{0.05} \pm 0.01$ & $0.39 \pm 0.36$ & $\textcolor{red}{0.08} \pm 0.05$ & $\textbf{0.96} \pm 0.05$ & $\textcolor{red}{0.01} \pm 0.02$ \\

\rowcolor{DAcolor}
DA: crop (small) & 
$\textcolor{red}{0.15} \pm 0.00$ & $\textcolor{red}{0.05} \pm 0.02$ & $\textcolor{red}{0.07} \pm 0.02$ & $\textcolor{red}{0.00} \pm 0.01$ & $\textcolor{red}{0.01} \pm 0.01$ & $\textbf{1.00} \pm 0.00$ & $\textcolor{red}{0.00} \pm 0.00$ \\
\rowcolor{LTcolor}
LT: change positions & 
$\textbf{0.89} \pm 0.01$ & $\textcolor{red}{0.19} \pm 0.20$ & $\textcolor{red}{0.05} \pm 0.01$ & $0.48 \pm 0.04$ & $\textcolor{red}{0.05} \pm 0.02$ & $\textbf{0.98} \pm 0.00$ & $0.25 \pm 0.03$ \\
\midrule
\rowcolor{DAcolor}
DA: crop (large) + colour distortion & %
$\textbf{0.86} \pm 0.03$ & $0.40 \pm 0.07$ & $\textcolor{red}{0.23} \pm 0.02$ & $\textcolor{red}{0.24} \pm 0.01$ & $\textcolor{red}{0.00} \pm 0.00$ & $\textcolor{red}{0.00} \pm 0.00$ & $0.47 \pm 0.04$ \\
\rowcolor{DAcolor}
DA: crop (small) + colour distortion & %
$\textbf{0.99} \pm 0.01$ & $\textbf{0.63} \pm 0.03$ & $\textbf{0.88} \pm 0.01$ & $0.32 \pm 0.02$ & $\textcolor{red}{0.00} \pm 0.00$ & $\textcolor{red}{0.16} \pm 0.13$ & $\textbf{0.52} \pm 0.03$ \\
\rowcolor{LTcolor}
LT: change positions + hues & 
$\textbf{1.00} \pm 0.00$ & $\textcolor{red}{0.21} \pm 0.22$ & $\textcolor{red}{0.07} \pm 0.01$ & $0.30 \pm 0.00$ & $\textcolor{red}{0.00} \pm 0.00$ & $\textcolor{red}{0.02} \pm 0.01$ & $0.46 \pm 0.06$ \\
\bottomrule
\end{tabular}
}
\end{table}

\subsection{\textit{MPI3D-real}}
\label{app:mpi}
We ran the same experimental setup as in~\cref{sec:experiment_2_causal3dident} also on the \textit{MPI3D-real} dataset~\cite{gondal2019transfer} containing $>1$ million \textit{real} images with ground-truth annotations of 3D objects being moved by a robotic arm. 

As \textit{MPI3D-real} contains much lower resolution images ($64\times64$) compared to ImageNet \& Causal3DIdent ($224\times224$), we used the standard convolutional encoder from the disentanglement literature~\citep{locatello2019challenging}, and ran a sanity check experiment to verify that by training the same backbone as in our unsupervised experiment with supervised learning, we can recover the ground-truth factors from the augmented views. In~\Cref{tbl:sup_mpi3d}, we observe that only five out of seven factors can be consistently inferred, object shape and size are somewhat ambiguous even when observing the original image. Note that while in the self-supervised case, we evaluate by training a nonlinear regression for each ground truth factor separately, in the supervised case, we train a network for all ground truth factors simultaneously from scratch for as many gradient steps as used for learning the self-supervised model.

In~\Cref{tbl:mpi3d}, we report the evaluation results in the self-supervised scenario. Subject to the aforementioned caveats, the results show a similar trend as those on \textit{Causal3DIdent}, i.e. with colour distortion, color factors of variation are decoded significantly worse than positional/rotational information. 

\definecolor{LTcolor}{rgb}{0.95,1,1}
\definecolor{DAcolor}{rgb}{1,1,0.95}
\begin{table}[t]
\centering
\caption{\small \textit{MPI3D-real} results: $R^2$ mean $\pm$ std.\ dev.\  over $3$ random seeds for dim($\hat{\cb}$)$=5$. DA: data augmentation, bold: $R^2\geq 0.5$, red: $R^2<0.25$.
}
\label{tbl:mpi3d}
\resizebox{\textwidth}{!}{
\small
\begin{tabular}{l|ccccccc}
\toprule
\textbf{Views generated by} 
& $\text{object color}$ & $\text{object shape}$ & $\text{object size}$ & $\text{camera height}$ & $\text{background color}$ & $\text{horizontal axis}$ & $\text{vertical axis}$ 
\\
\midrule
\rowcolor{DAcolor}
DA: colour distortion   & $0.39 \pm 0.01$ & $\textcolor{red}{0.00} \pm 0.00$ & $\textcolor{red}{0.16} \pm 0.01$ & $\textbf{1.00} \pm 0.00$ & $\textcolor{red}{0.09} \pm 0.15$ & $\textbf{0.60} \pm 0.06$ & $0.42 \pm 0.08$ \\
\midrule
\rowcolor{DAcolor}
DA: crop (large)  & $\textbf{0.65} \pm 0.17$ & $\textcolor{red}{0.01} \pm 0.02$ & $0.31 \pm 0.03$ & $\textbf{1.00} \pm 0.00$ & $\textbf{1.00} \pm 0.00$ & $0.37 \pm 0.06$ & $\textcolor{red}{0.08} \pm 0.03$ \\
\midrule
\rowcolor{DAcolor}
DA: crop (small)  & $\textcolor{red}{0.09} \pm 0.02$ & $\textcolor{red}{0.03} \pm 0.00$ & $\textcolor{red}{0.19} \pm 0.01$ & $\textbf{1.00} \pm 0.00$ & $\textbf{1.00} \pm 0.00$ & $\textcolor{red}{0.21} \pm 0.02$ & $\textcolor{red}{0.07} \pm 0.00$ \\
\midrule
\rowcolor{DAcolor}
DA: crop (large) + colour distortion & $0.34 \pm 0.00$ & $\textcolor{red}{0.00} \pm 0.00$ & $\textcolor{red}{0.22} \pm 0.03$ & $\textbf{1.00} \pm 0.00$ & $0.39 \pm 0.02$ & $\textbf{0.54} \pm 0.01$ & $0.29 \pm 0.01$\\
\midrule
\rowcolor{DAcolor}
DA: crop (small) + colour distortion  & $0.25 \pm 0.02$ & $\textcolor{red}{0.00} \pm 0.00$ & $\textcolor{red}{0.10} \pm 0.01$ & $\textbf{1.00} \pm 0.00$ & $\textbf{0.75} \pm 0.16$ & $\textbf{0.54} \pm 0.01$ & $0.29 \pm 0.03$ \\
\bottomrule
\end{tabular}
}
\end{table}

\definecolor{LTcolor}{rgb}{0.95,1,1}
\definecolor{DAcolor}{rgb}{1,1,0.95}
\begin{table}[t]
\centering
\caption{\small \textbf{Supervised} \textit{MPI3D-real} results: $R^2$ mean $\pm$ std.\ dev.\  over $3$ random seeds. DA: data augmentation. bold: $R^2\geq 0.5$, red: $R^2<0.25$.
}
\label{tbl:sup_mpi3d}
\resizebox{\textwidth}{!}{
\small
\begin{tabular}{l|ccccccc}
\toprule
\textbf{Views generated by} 
& $\text{object color}$ & $\text{object shape}$ & $\text{object size}$ & $\text{camera height}$ & $\text{background color}$ & $\text{horizontal axis}$ & $\text{vertical axis}$ 
\\
\midrule
Original   & $\textbf{0.90} \pm 0.01$ & $0.25 \pm 0.02$ & $\textbf{0.61} \pm 0.02$ & $\textbf{0.99} \pm 0.00$ & $\textbf{0.97} \pm 0.01$ & $\textbf{1.00} \pm 0.00$ & $\textbf{1.00} \pm 0.00$ \\
\midrule
\rowcolor{DAcolor}
DA: colour distortion   & $\textbf{0.61} \pm 0.01$ & $\textcolor{red}{0.11} \pm 0.00$ & $0.47 \pm 0.01$ & $\textbf{0.98} \pm 0.00$ & $\textbf{0.93} \pm 0.00$ & $\textbf{0.99} \pm 0.00$ & $\textbf{1.00} \pm 0.00$ \\
\midrule
\rowcolor{DAcolor}
DA: crop (large)  & $\textbf{0.82} \pm 0.01$ & $\textcolor{red}{0.05} \pm 0.01$ & $0.42 \pm 0.02$ & $\textbf{0.97} \pm 0.01$ & $\textbf{0.91} \pm 0.00$ & $\textbf{0.96} \pm 0.00$ & $\textbf{0.97} \pm 0.01$ \\
\midrule
\rowcolor{DAcolor}
DA: crop (small)  & $\textbf{0.71} \pm 0.04$ & $\textcolor{red}{0.01} \pm 0.00$ & $0.32 \pm 0.02$ & $\textbf{0.95} \pm 0.00$ & $\textbf{0.85} \pm 0.01$ & $\textbf{0.79} \pm 0.02$ & $\textbf{0.90} \pm 0.01$ \\
\midrule
\rowcolor{DAcolor}
DA: crop (large) + colour distortion & $0.45 \pm 0.02$ & $\textcolor{red}{0.02} \pm 0.00$ & $\textcolor{red}{0.22} \pm 0.00$ & $\textbf{0.95} \pm 0.01$ & $\textbf{0.67} \pm 0.01$ & $\textbf{0.91} \pm 0.00$ & $\textbf{0.94} \pm 0.00$\\
\midrule
\rowcolor{DAcolor}
DA: crop (small) + colour distortion  & $0.45 \pm 0.02$ & $\textcolor{red}{0.00} \pm 0.00$ & $\textcolor{red}{0.17} \pm 0.02$ & $\textbf{0.91} \pm 0.02$ & $\textbf{0.55} \pm 0.03$ & $\textbf{0.69} \pm 0.01$ & $\textbf{0.79} \pm 0.08$ \\
\bottomrule
\end{tabular}
}
\end{table}

\section{Experimental details}
\label{app:experimental_details}
\paragraph{Ground-truth generative model.}
The generative process used in our numerical simulations~(\cref{sec:experiment_1_numerical_simulation}) is summarised by the following:
\begin{align*}
    \cb &\sim p(\cb)=\Ncal(0,\Sigma_\cb), 
    \quad \text{with} \quad \Sigma_\cb\sim \mathrm{Wishart}_{n_c}(\Ib, n_c),
    \\
    \sb| \cb &\sim p(\sb|\cb) = \Ncal(\ab+B\cb,\Sigma_\sb), 
    \quad \text{with} \quad \Sigma_\sb\sim \mathrm{Wishart}_{n_s}(\Ib, n_s),
    \quad a_i, b_{ij}\overset{\mathrm{i.i.d.}}{\sim}\Ncal(0,1),
    \\
    \sbt_A| \sb_A, A &\sim p(\sbt_A|\sb_A) = N(\sb_A,\Sigma(A))
    \quad \text{with} \quad \Sigma\sim \mathrm{Wishart}_{n_s}(\Ib, n_s),
    \\
    (\xbt, \xb)&=(\fb_{\MLP}(\zbt), \fb_{\MLP}(\zb)),
\end{align*}
where the set of changing style vectors $A$ is obtained by flipping a (biased) coin with $\text{p(chg.)}=0.75$ for each style dimension independently, and where $\Sigma(A)$ denotes the submatrix of $\Sigma$ defined by selecting the rows and columns corresponding to subset $A$. 

When we do not allow for \emph{statistical dependence} (Stat.) within blocks of content and style variables, we set the covariance matrices $\Sigma_\cb$, $\Sigma_\sb$, and  $\Sigma$ to the identity. 
When we do not allow for \emph{causal dependence} (Cau.) of style on content, we set $a_i,b_{ij}=0, \forall i,j$.

For $\fb_{\MLP}$, we use a $3$-layer MLP with LeakyReLU ($\alpha=0.2$) activation functions, specified using the same process as used in previous work~\citep{zimmermann2021contrastive,hyvarinen2016unsupervised,hyvarinen2017nonlinear}. For the square weight matrices, we draw $(n_c+n_s)\times(n_c+n_s)$ samples from $U(-1,1)$, and perform $l_2$ column normalisation. In addition, to control for invertibility, we re-sample the weight matrices until their condition number is less than or equal to a threshold value. The threshold is pre-computed by sampling $24,975$ weight matrices, and recording the minimum condition number. 

\paragraph{Training encoder.}

Recall that the result of~\Cref{thm:CL_MaxEnt} corresponds to minimizing the following functional~\eqref{eq:CL_MSE_MaxEnt_objective}:

\begin{equation*}
\textstyle
\Lcal_\mathrm{AlignMaxEnt}(\gb)
:= 
\EE_{(\xb,\xbt)\sim p_{\xb, \xbt}}
\big[
\big(
\gb(\xb)-\gb(\xbt)
\big)^2
\big] - H\left(\gb(\xb)\right).
\end{equation*}

Note that InfoNCE~\cite{oord2018representation,chen2020simple}~\eqref{eq:InfoNCE_objective} can be rewritten as:

\begin{equation}
\resizebox{.9 \textwidth}{!}{$
\textstyle
    \Lcal_{\text{InfoNCE}}
    (\gb;\tau,K)
    =
    \EE_{\{\xb_i,\xbt_i\}_{i=1}^K \sim p_{\xb, \xbt}}
    \big[
    -
    \sum_{i=1}^K\text{sim}(\gb(\xb)_i,\gb(\xbt)_i)/\tau
    +
    \log{
    \sum_{j=1}^K
    \exp\{\text{sim}(\gb(\xb)_i,\gb(\xbt)_j)/\tau\}
    }
    \big]
$}.
\end{equation}

Thus, if we consider $\tau=1$, and $\text{sim}(u,v)=-(u-v)^2$, 

\begin{equation}
\label{eq:InfoNCE_objective_rewrite}
\resizebox{.9 \textwidth}{!}{$
\textstyle
    \Lcal_{\text{InfoNCE}}
    (\gb;K)
    =
    \EE_{\{\xb_i,\xbt_i\}_{i=1}^K \sim p_{\xb, \xbt}}
    \big[
    \sum_{i=1}^K\big(\gb(\xb)_i-\gb(\xbt)_i\big)^2
    +
    \log{
    \sum_{j=1}^K
    \exp\{-(\gb(\xb)_i-\gb(\xbt)_j)^2\}
    }
    \big]
$}
\end{equation}

we can approximately match the form of~\eqref{eq:CL_MSE_MaxEnt_objective}. 
In practice, we use $K=6,144$. 

For $\gb$, as in~\citep{zimmermann2021contrastive}, we use a $7$-layer MLP with (default) LeakyReLU ($\alpha=0.01$) activation functions. As the input dimensionality is $(n_c+n_s)$, we consider the following multipliers $[10,50,50,50,50,10]$ for the number of hidden units per layer. In correspondence with~\Cref{thm:CL_MaxEnt}, we set the output dimensionality to $n_c$.

We train our feature encoder for $300,000$ iterations, using Adam \citep{kingma2014adam} with a learning rate of $10^{-4}$.

\paragraph{Causal3DIdent.} We here elaborate on details specific to the experiments in~\Cref{sec:experiment_2_causal3dident}. We train the feature encoder for $200,000$ iterations using Adam with a learning rate of $10^{-4}$. For the encoder we use a ResNet18 \citep{he2015deep} architecture followed by a single hidden layer with dimensionality $100$ and LeakyReLU activation function using the default ($0.01$) negative slope. The scores are evaluated on a test set consisting of $25,000$ samples not included in the training set.

\paragraph{Data augmentations.}
We here specify the parameters for the data augmentations we considered:
\begin{itemize}
    \item colour distortion: see the paragraph labelled ``Color distortion'' in Appendix A of~\citep{chen2020simple} for details. We use $s=1.0$, the default value.
    \item crop: see the paragraph labelled ``Random crop and resize to $224\times224$'' in Appendix A of~\citep{chen2020simple} for details. For small crops, a crop of random size (uniform from $0.08$ to $1.0$ in area) of the original size is made, which corresponds to what was used in the experiments reported in~\citep{chen2020simple}. For large crops, a crop of random size (uniform from $0.8$ to $1.0$ in area) of the original size is made.
    \item rotation: as specified in the captions for Figure $4$ \& Table $3$ in~\citep{chen2020simple}, we sample one of $\{0\degree,90\degree,180\degree,270\degree\}$ uniformly. Note that for the pair, we sample two values without replacement. 
\end{itemize}
A visual overview of the effect of these image-level data augmentations is shown in~\cref{fig:data_augs}.

\begin{figure}
    \centering
    \includegraphics[width=\textwidth]{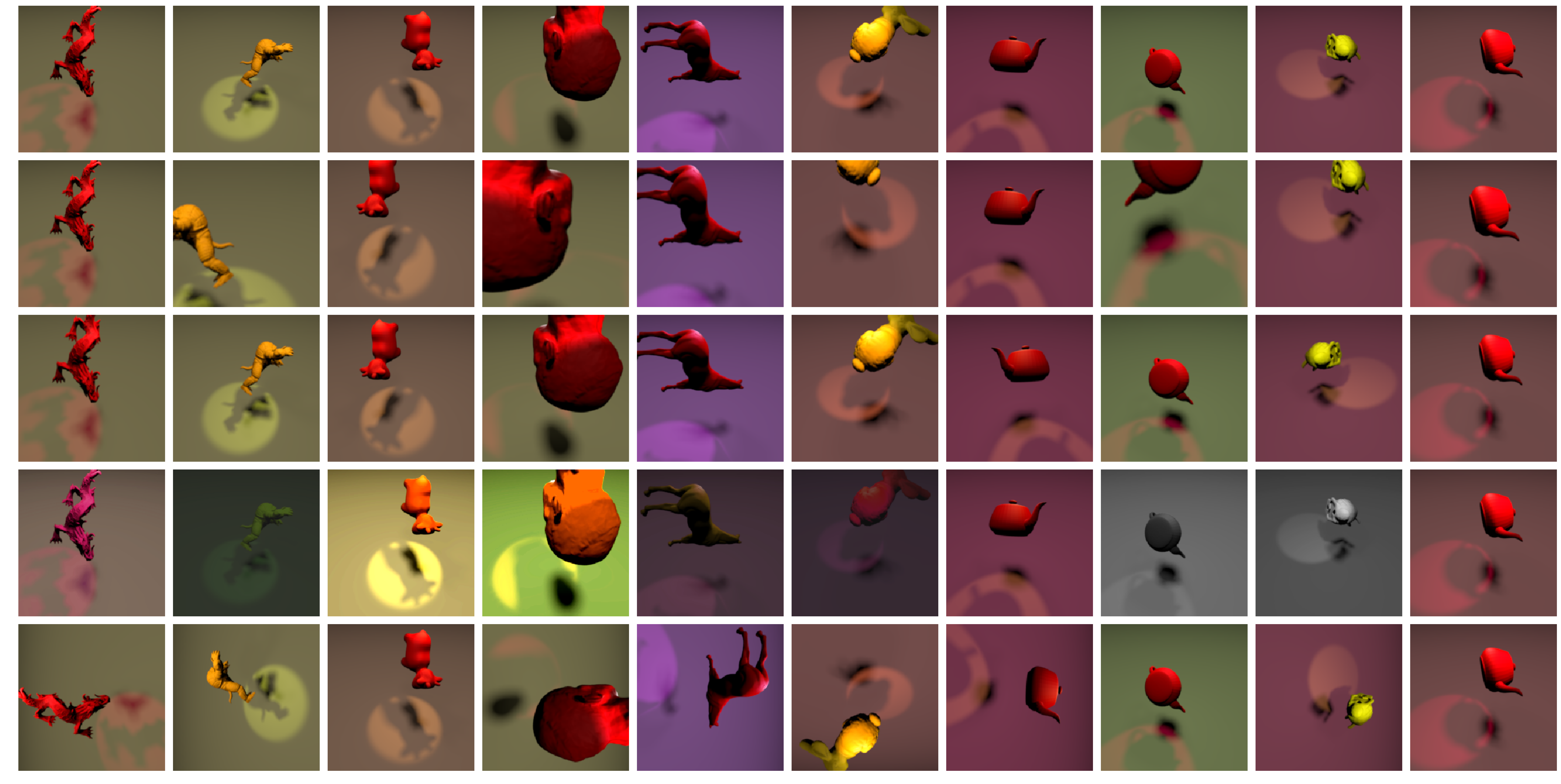}
    \caption{Visual overview of the effect of different  data augmentations (DA), applied to 10 representative samples. Rows correspond to \textit{(top to bottom)}: original images, small random crop (+ random flip), large random crop (+ random flip),  colour distortion  (jitter \& drop), and  random rotation.}
    \label{fig:data_augs}
\end{figure}

\paragraph{Latent transformations.}
To generate views via latent transformations (LT) in our experiments on Causal3DIdent~(\cref{sec:experiment_2_causal3dident}), we proceed as follows.

Let $\zb$ refer to the latent corresponding to the original image. For all latents specified to change, we sample $\zbh'$ from a truncated normal distribution constrained to $[-1,1]$, centered at $\zb$, with $\sigma=1.$. Then, we use nearest-neighbor matching to find the latent $\zbh$ closest to $\zbh'$ (in $L^2$ distance) for which there exists an image rendering.\footnote{see~\citep{zimmermann2021contrastive} for further details} 

\paragraph{Evaluation.}
Recall that \Cref{thm:CL_MaxEnt} states that $\gb$ block-identifies the true content variables in the sense of~\cref{def:block-identifiability}, i.e., there exists an \textit{invertible} function $\hb:\RR^{n_c}\rightarrow \RR^{n_c}$ s.t.\ $\cbh=\hb(\cb)$. 

Since this is different from typical evaluation in disentanglement or ICA in that we do not assume independence and do not aim to find a one-to-one correspondence between inferred and ground truth latents,  existing metrics, such as MCC~\citep{hyvarinen2016unsupervised,hyvarinen2017nonlinear} or MIG~\citep{chen2018isolating}, do not apply. 

We therefore treat identifying $\hb$ as a regression task, which we solve using kernel ridge regression with a Gaussian kernel~\citep{murphybook}.
Since the Gaussian kernel is universal, this constitutes a nonparametric regression technique with universal approximation capabilities, i.e., any nonlinear function can be approximated arbitrarily well given sufficient data.

We sample $4096\times10$ datapoints from the marginal for evaluation. For kernel ridge regression, we standardize the inputs and targets, and fit the regression model on $4096\times5$ (distinct) datapoints. We tune the regularization strength $\alpha$ and kernel variance $\gamma$ by $3$-fold cross-validated grid search over the following parameter grids: $\alpha\in[1,0.1,0.001,0.0001]$, $\gamma\in[0.01,0.22,4.64,100]$.

\paragraph{Compute.} The experiments in~\Cref{sec:experiment_1_numerical_simulation} took on the order of 5-10 hours on a single GeForce RTX 2080 Ti GPU. The experiments in~\Cref{sec:experiment_2_causal3dident} on 3DIdent took 28 hours on four GeForce RTX 2080 Ti GPUs. The creation of the Causal3DIdent dataset additionally required approximately 150 hours of compute time on a GeForce RTX 2080 Ti.

\end{document}